\documentclass[11pt]{article}

\usepackage{times}

\def\colorful{0}

\usepackage{amsthm,amsfonts,amsmath,amssymb,epsfig,color,float,graphicx,verbatim, enumitem}
\usepackage{multirow}

\newif\ifhyper\IfFileExists{hyperref.sty}{\hypertrue}{\hyperfalse}
\hypertrue
\ifhyper\usepackage{hyperref}\fi

\usepackage{enumitem}

\makeatletter
\renewcommand{\section}{\@startsection{section}{1}{0pt}{-12pt}{5pt}{\large\bf}}
\renewcommand{\subsection}{\@startsection{subsection}{2}{0pt}{-12pt}{-5pt}{\normalsize\bf}}
\renewcommand{\subsubsection}{\@startsection{subsubsection}{3}{0pt}{-12pt}{-5pt}{\normalsize\bf}}
\makeatother

\usepackage{framed}
\usepackage{nicefrac}

\def\nnewcolor{1}
\ifnum\nnewcolor=1

\fi
\ifnum\nnewcolor=0

\fi

\ifnum\colorful=1
\newcommand{\new}[1]{{\color{red} #1}}

\else
\newcommand{\new}[1]{{#1}}

\fi

\newtheorem{theorem}{Theorem}[section]
\newtheorem{question}{Question}[section]

\newtheorem{cond}[theorem]{Condition}
\newtheorem{lemma}[theorem]{Lemma}
\newtheorem{informal theorem}[theorem]{Theorem (informal statement)}

\newtheorem{proposition}[theorem]{Proposition}
\newtheorem{corollary}[theorem]{Corollary}
\newtheorem{claim}[theorem]{Claim}
\newtheorem{fact}[theorem]{Fact}

\newtheorem{remark}[theorem]{Remark}

\theoremstyle{definition}
\newtheorem{definition}[theorem]{Definition}
\newcommand{\eqdef}{\stackrel{{\mathrm {\footnotesize def}}}{=}}

\newcommand{\p}{\mathbf{P}}
\newcommand{\q}{\mathbf{Q}}
\newcommand{\h}{\mathbf{H}}
\newcommand{\bx}{\mathbf{x}}
\newcommand{\R}{\mathbb{R}}
\newcommand{\s}{\mathbb{S}}
\newcommand{\Z}{\mathbb{Z}}
\newcommand{\N}{\mathbb{N}}
\newcommand{\E}{\mathbf{E}}
\newcommand{\eps}{\epsilon}
\newcommand{\dtv}{d_{\mathrm TV}}
\newcommand{\pr}{\mathbf{Pr}}
\newcommand{\poly}{\mathrm{poly}}
\newcommand{\var}{\mathbf{Var}}

\newcommand{\littleint}{\mathop{\textstyle \int}}
\newcommand{\littlesum}{\mathop{\textstyle \sum}}
\newcommand{\littleprod}{\mathop{\textstyle \prod}}

\newcommand{\ba}{\mathbf{a}}
\newcommand{\bb}{\mathbf{b}}
\newcommand{\be}{\mathbf{e}}
\newcommand{\bi}{\mathbf{i}}

\newcommand{\wt}{\widetilde}
\newcommand{\wh}{\widehat}

\title{Statistical Query Lower Bounds for Robust Estimation
\\of High-Dimensional Gaussians and Gaussian Mixtures}

\author{
Ilias Diakonikolas\thanks{Supported by NSF Award CCF-1652862 (CAREER) and a Sloan Research Fellowship.}\\
University of Southern California\\
{\tt diakonik@usc.edu}\\
\and
Daniel M. Kane\thanks{Supported by NSF Award CCF-1553288 (CAREER) and a Sloan Research Fellowship.}\\
University of California, San Diego\\
{\tt dakane@cs.ucsd.edu}\\
\and
Alistair Stewart\\ University of Southern California\\
{\tt alistais@usc.edu}
}

\begin{document}

\maketitle

\thispagestyle{empty}

\vspace{-0.2cm}

\begin{abstract}
We describe a general technique that yields the first {\em Statistical Query lower bounds} for
a range of fundamental high-dimensional learning problems involving 
Gaussian distributions. Our main results are for the problems of 
(1) learning Gaussian mixture models (GMMs), and (2) robust (agnostic) learning of a single unknown Gaussian distribution. 
For each of these problems, we show a {\em super-polynomial gap} between the (information-theoretic)
sample complexity and the computational complexity of {\em any} Statistical Query algorithm for the problem. 
Statistical Query (SQ) algorithms are a class of algorithms 
that  are only allowed to query expectations of functions of the distribution rather than directly access samples.
This class of algorithms is quite broad: 
a wide range of known algorithmic techniques in machine learning are known to 
be implementable using SQs. Moreover, for the unsupervised learning problems studied in this paper, all known algorithms with non-trivial performance guarantees are SQ or are easily implementable using SQs.

Our SQ lower bound for Problem (1)
is qualitatively matched by known learning algorithms for GMMs. 
At a conceptual level, this result implies that -- as far as SQ algorithms are concerned -- the computational complexity 
of learning GMMs is inherently exponential 
{\em in the dimension of the latent space} -- even though there 
is no such information-theoretic barrier. 
Our lower bound for Problem (2) implies that the accuracy of the robust learning algorithm 
in~\cite{DiakonikolasKKLMS16} is essentially best possible among all polynomial-time SQ algorithms.
On the positive side, we also give a new (SQ) learning algorithm for Problem (2) achieving
the information-theoretically optimal accuracy, up to a constant factor, 
whose running time essentially matches our lower bound.
Our algorithm relies on a filtering technique generalizing~\cite{DiakonikolasKKLMS16} 
that removes outliers based on higher-order tensors.

Our SQ lower bounds are attained via a unified moment-matching technique that is useful in other contexts and may be of broader interest. Our technique yields nearly-tight lower bounds for a number of related unsupervised estimation problems.
Specifically, for the problems of (3) robust covariance estimation in spectral norm,  
and (4) robust sparse mean estimation, we establish a quadratic {\em statistical--computational tradeoff} for SQ algorithms,
matching known upper bounds. Finally, our technique can be used to obtain tight sample complexity
lower bounds for high-dimensional {\em testing} problems. Specifically, for the classical problem of robustly {\em testing} an unknown mean (known covariance) Gaussian, our technique implies 
an information-theoretic sample lower bound that scales {\em linearly} in the dimension.
Our sample lower bound matches the sample complexity of the corresponding robust {\em learning} problem and separates the sample complexity of robust testing from standard (non-robust) testing.
This separation is surprising because such a gap does not exist for the corresponding learning problem.
\end{abstract}

\thispagestyle{empty}
\setcounter{page}{0}

\newpage

\section{Introduction} \label{sec:intro}

\subsection{Background and Overview} \label{ssec:motiv}
For the unsupervised estimation problems considered here,
the input is a probability distribution which is accessed via a sampling oracle,
i.e., an oracle that provides i.i.d. samples from the underlying distribution.
Statistical Query (SQ) algorithms are a restricted class of algorithms
that  are only allowed to query expectations of functions of the distribution rather than directly access samples.
This class of algorithms is quite broad:  a wide range of known algorithmic techniques in machine learning are known to be implementable using SQs. These include spectral techniques, moment and tensor methods, local search (e.g., Expectation Maximization), and many others 
(see, e.g.,~\cite{Chu:2006, Feldman13} for a detailed discussion).
Moreover, for the unsupervised learning problems studied in this paper, all known
algorithms with non-trivial performance guarantees are SQ or are easily implementable using SQs.

A number of techniques have been developed in information theory and statistics
to characterize the sample complexity of inference tasks. These involve both techniques for proving sample complexity
upper bounds (e.g., VC dimension, metric/bracketing entropy) and information-theoretic lower bounds (e.g., Fano and Le Cam methods).
On the other hand, computational lower bounds have been much more scarce in the unsupervised setting.
Perhaps surprisingly, it is possible to prove {\em unconditional} lower bounds
on the computational complexity of {\em any} SQ algorithm that solves a given learning problem.
Given the ubiquity and generality of SQ algorithms, an SQ lower bound provides strong evidence 
of the problem's computational intractability.

In this paper, we describe a general technique that yields the first {\em Statistical Query lower bounds} for a range of fundamental high-dimensional learning problems involving 
Gaussian distributions. Such problems are ubiquitous in applications across the data sciences
and have been intensely investigated by different communities of researchers for several decades.
Our main results are for the problems of  (1) learning Gaussian mixture models (GMMs), 
and (2) robust (agnostic) learning of a single unknown Gaussian distribution. 
In particular, we show a {\em super-polynomial gap} between the (information-theoretic)
sample complexity and the computational complexity of {\em any} Statistical Query algorithm for these problems.
In more detail, our SQ lower bound for Problem (1)
is qualitatively matched by known learning algorithms for GMMs (all of which can be implemented as SQ algorithms).
For Problem (2), we give a new (SQ) algorithm in this paper whose running time nearly matches our SQ lower bound.

Our SQ lower bounds are attained via a unified moment-matching technique that is useful in other contexts and may be of broader interest. Our technique yields nearly-tight lower bounds for a number of related unsupervised estimation problems.
Specifically, for the problems of (3) robust covariance estimation in spectral norm,  
and (4) robust sparse mean estimation, we establish a quadratic {\em statistical--computational tradeoff} for SQ algorithms,
matching known upper bounds. 

Finally, we use our technique to obtain tight sample complexity
lower bounds for high-dimensional {\em testing} problems. Specifically, for the classical problem of robustly {\em testing} 
an unknown mean (known covariance) Gaussian, our technique implies 
an information-theoretic lower bound that scales {\em linearly} in the dimension.
This lower bound matches the sample complexity of the corresponding robust learning problem and 
separates the sample complexity of robust testing from standard (non-robust) testing.
This separation is surprising because such a gap does not exist for the corresponding learning problem.



Before we discuss our contributions in detail, we provide the necessary background for the Statistical Query model
and the unsupervised estimation problems that we study.

\paragraph{Statistical Query Algorithms.}
A Statistical Query (SQ) algorithm relies on an oracle that given any bounded function
on a single domain element provides an estimate of the expectation of the function on a random sample from the input distribution.
This computational model was introduced by Kearns~\cite{Kearns:98} in the context of supervised learning
as a natural restriction of the PAC model~\cite{Valiant:84}. Subsequently, the SQ model
has been extensively studied in a plethora of contexts (see, e.g.,~\cite{Feldman16b} and references therein).

A recent line of work~\cite{Feldman13, FeldmanPV15, FeldmanGV15, Feldman16}
developed a framework of SQ algorithms for search problems over distributions --
encompassing the distribution estimation problems we study in this work. It turns out that one can prove unconditional
lower bounds on the computational complexity of SQ algorithms via the notion of {\em Statistical Query dimension}.
This complexity measure was introduced in~\cite{BFJ+:94} for PAC learning of Boolean functions and was recently generalized to the unsupervised setting~\cite{Feldman13, Feldman16}. A lower bound on the SQ dimension of a learning problem provides an unconditional lower bound on the computational complexity of any SQ algorithm for the problem.

\medskip

\noindent {\bf Remark.}
We would like to emphasize here that the SQ lower bounds shown in this paper 
apply to the running time of an SQ algorithm and not on its sample complexity 
(when we simulate the SQ algorithm by drawing samples to answer its SQ queries).
Specifically, for all learning problems considered in this paper, there exist straightforward SQ algorithms (that can be simulated with sample access to the distribution)
with near-optimal sample complexity, albeit with exponential running time. 
Specifically, lower bounds on the SQ dimension of the corresponding problems
establish lower bounds on the running time of any SQ algorithm for the problem -- not on its sample complexity.

\paragraph{Learning Gaussian Mixture Models.}
A mixture model is a convex combination of distributions of known type. The most commonly studied case is
a Gaussian mixture model (GMM). An {\em $n$-dimensional $k$-GMM} is
a distribution in $\R^n$ that is composed
of $k$ unknown Gaussian components, i.e., $F = \sum_{i=1}^k w_i N(\mu_i, \Sigma_i)$, where
the weights $w_i$, mean vectors $\mu_i$, and covariance matrices $\Sigma_i$ are unknown.
The problem of learning a GMM from samples has received tremendous attention in statistics
and, more recently, in TCS.  A long line of work initiated by Dasgupta~\cite{Dasgupta:99, AroraKannan:01, VempalaWang:02, AchlioptasMcSherry:05, KSV08, BV:08}
provides computationally efficient algorithms for recovering the parameters of a GMM under separability
assumptions. 
Subsequently, efficient parameter learning algorithms have been obtained~\cite{MoitraValiant:10, BelkinSinha:10, HardtP15}
under minimal information-theoretic separation assumptions. The related problems of density estimation and proper learning
have also been extensively studied~\cite{FOS:06, SOAJ14, DK14, MoitraValiant:10, HardtP15, LiS15a}. In density estimation (resp. proper learning),
the goal is to output some hypothesis (resp. GMM) that is close to the unknown mixture in total variation distance.

The sample complexity of density estimation (and proper learning)
for $n$-dimensional $k$-GMMs, up to variation distance $\eps$, is easily seen to be $\poly(n, k, 1/\eps)$
-- without any assumptions.(In Appendix~\ref{sec:sample-gmm}, 
we describe a simple SQ algorithm for this learning problem
with sample complexity $\poly(n, k, 1/\eps)$, albeit exponential running time).
Given that there is no information-theoretic barrier for learnability in this setting,
the following question arises: {\em Is there a $\poly(n, k, 1/\eps)$ {\em time} 
algorithm for density estimation (or proper learning) of $n$-dimensional $k$-GMMs?}
This question has been raised as an open problem in a number of settings (see, e.g.,~\cite{Moitra14, Diak16} and references therein).

For parameter learning, the situation is somewhat subtle:
In full generality, the sample complexity is of the form
$\poly(n) \cdot (1/\gamma)^{\Omega(k)}$, where the parameter $\gamma >0$ quantifies
the ``separation'' between the components. Even in one-dimension, a sample complexity
lower bound of $(1/\gamma)^{\Omega(k)}$ is known~\cite{MoitraValiant:10, HardtP15}\footnote{To circumvent the information-theoretic bottleneck of parameter learning, 
a related line of work has studied parameter learning in a smoothed setting~\cite{HK, BhaskaraCMV14, AndersonBGRV14, GeHK15}.}. 
The corresponding ``hard'' instances~\cite{MoitraValiant:10, HardtP15} consist of GMMs whose components have large
overlap, so many samples are required to distinguish between them. {\em Is this the only obstacle towards a $\poly(n, k)$ time parameter learning algorithm?}
Specifically, suppose that we are given an instance of the problem with the additional promise that the components are ``nearly non-overlapping'' -- 
so that $\poly(n, k)$ samples suffice for the parameter learning problem as well. 
(In Appendix~\ref{sec:param-gmm}, we show that when the total variation distance between
any pair of components in the given mixture is close to $1$, parameter learning reduces 
to proper learning; hence, there is a $\poly(n, k)$-sample parameter learning (SQ) algorithm that runs in exponential time.)
Is there a $\poly(n, k)$ {\em time} parameter learning algorithm for such instances?

In summary, the sample complexity of both versions of the learning problem is 
$\poly(n) f(k)$. On the other hand, the running time of all 
known algorithms for either version scales as $n^{g(k)}$,
where $g(k) \geq k$. 
This runtime is super-polynomial in the sample complexity of the problem
for super-constant values of $k$ and is tight for these algorithms, even for 
GMMs with almost non-overlapping components. 
The preceding discussion is summarized in the following:

\begin{question} \label{q:gmm}
Is there a $\poly(n, k)$-time density estimation algorithm for $n$-dimensional $k$-GMMs?
Is there a $\poly(n, k)$-time parameter learning algorithm for nearly non-overlapping $n$-dimensional $k$-GMMs?
\end{question}

\vspace{-0.3cm}

\paragraph{Robust Learning of a Gaussian.}
In the preceding paragraphs, we were working under the assumption that the
unknown distribution generating the samples is {\em exactly} a mixture of Gaussians.
The more general and realistic setting of {\em robust} (or agnostic) learning
-- when our assumption about the model is {\em approximately} true -- 
turns out to be significantly more challenging. 
Specifically, until recently, even the most basic setting of robustly learning an unknown mean Gaussian
with identity covariance matrix was poorly understood. Without corruptions, this problem is straightforward:
The empirical mean gives a sample-optimal efficient estimator.
Unfortunately, the empirical estimate is very brittle and fails in the presence of corruptions.

The standard definition of agnostically learning a Gaussian
(see, e.g., Definition~2.1 in~\cite{DiakonikolasKKLMS16} and references therein) 
is the following: Instead of drawing samples from a perfect Gaussian, 
we have access to a distribution $D$ that is promised
to be {\em close} to an unknown Gaussian $G$ -- specifically $\eps$-close in total variation distance. 
This is the only assumption about the distribution $D$,
which may otherwise be arbitrary: the $\eps$-fraction of ``errors'' can be adversarially selected.
The goal of an agnostic learning algorithm is to output a hypothesis distribution $H$ that is as close as possible to $G$ (or, equivalently, $D$)
in variation distance. 
Note that the minimum variation distance, $\dtv(H, G)$, information-theoretically achievable 
under these assumptions is $\Theta(\eps)$, and we would like to obtain a polynomial-time algorithm 
with this guarantee.


Agnostically learning a single high-dimensional Gaussian 
is arguably {\em the} prototypical problem in robust statistics
~\cite{Huber64, HampelEtalBook86, Huber09}. 
Early work in this field~\cite{Tukey75, Donoho92} 
studied the sample complexity of robust estimation.
Specifically, for the case of an unknown mean and known covariance Gaussian, 
the Tukey median~\cite{Tukey75} achieves $O(\eps)$-error 
with $O(n/\eps^2)$ samples (see, e.g.,~\cite{CGR15b} for a simple proof).
Since $\Omega(n/\eps^2)$ samples are information-theoretically necessary -- 
even without noise -- the robustness requirement does 
not change the sample complexity of the problem.

The {\em computational} complexity of agnostically learning a Gaussian is less understood.
Until recently, all known polynomial time estimators 
could only guarantee error of $\Theta (\eps \sqrt{n})$.
Two recent works~\cite{DiakonikolasKKLMS16, LaiRV16} made a first step in designing
robust polynomial-time estimators for this problem.
The results of~\cite{DiakonikolasKKLMS16} apply in the standard agnostic model;~\cite{LaiRV16} works in a 
weaker model -- known as Huber's contamination model~\cite{Huber64} -- 
where the noisy distribution $D$ is of the form 
$(1-\eps) G + \eps N$, where $N$ is an unknown ``noise'' distribution.
For the problem of robustly estimating an unknown mean Gaussian $N(\mu, I)$,~\cite{LaiRV16} obtains
an error guarantee of $O(\eps \sqrt{\log n})$, while~\cite{DiakonikolasKKLMS16} obtains error
$O(\eps \sqrt{\log (1/\eps)})$, independent of the dimension\footnote{The algorithm of ~\cite{LaiRV16} can be 
extended to work in the standard agnostic model at the expense 
of an increased error guarantee of $O(\eps \sqrt{\log n \log(1/\eps)})$.}. 

A natural and important open problem, put forth by these works~\cite{DiakonikolasKKLMS16, LaiRV16},
is the following:

\begin{question}  \label{q:robust}
Is there a $\poly(n/\eps)$- time agnostic learning algorithm, 
with error $O(\eps)$, for an $n$-dimensional Gaussian?
\end{question}

\vspace{-0.3cm}

{
\paragraph{Statistical--Computational Tradeoffs.}
A statistical--computational tradeoff refers to the phenomenon that there
is an inherent gap between the information-theoretic sample complexity of a learning problem
and its computational sample complexity, i.e, the minimum sample complexity attainable 
by any polynomial time algorithm for the problem.
The prototypical example is the estimation of a covariance matrix 
under sparsity constraints (sparse PCA)~\cite{JL09, CMW13, CMW15}, 
where a nearly-quadratic gap between information-theoretic and computational sample complexity
has been established (see~\cite{BR13, WBS16}) -- 
assuming the computational hardness of the planted clique problem.

For a number of high-dimensional learning problems 
(including the problem of robustly learning a Gaussian under the total variation distance), 
it is known that the robustness requirement 
does not change the information-theoretic sample complexity of the problem.
On the other hand, it is an intriguing possibility that injecting noise into 
a high-dimensional learning problem may change its computational sample complexity. 

\begin{question}  \label{q:tradeoffs}
Does robustness create 
inherent statistical--computational tradeoffs 
for natural high-dimensional estimation problems?
\end{question}

In this work, we consider two natural instantiations of 
the above general question: 
(i) robust estimation of the covariance matrix in spectral norm, and (ii) robust sparse mean estimation.
We give basic background for these problems in the following paragraphs.

For (i), suppose we have sample access to a (zero-mean) $n$-dimensional unknown-covariance Gaussian,
and we want to estimate the covariance matrix {\em with respect to the spectral norm}. 
It is known (see, e.g., ~\cite{Vershynin2012}) that  $O(n/\eps^2)$ samples
suffice so that the empirical covariance is within spectral error at most $\eps$ 
from the true covariance; and this bound is information-theoretically optimal, to constant factors,
for any estimator. For simplicity, let us assume that the desired accuracy is a small positive constant, 
e.g., $\eps=1/10$. Now suppose that we observe samples from a corrupted Gaussian 
in Huber's contamination model (the weaker adversarial model) where the noise rate $\delta \ll 1/10$.
First, it is not hard to see that the injection of noise does not change the information-theoretic
sample complexity of the problem: there exist (computationally inefficient) robust estimators 
(see, e.g.,~\cite{CGR15b}) that use $O(n)$ samples. (There is a 
straightforward SQ algorithm for this problem as well that uses $O(n)$ samples, 
but again runs in exponential time.) On the other hand, if we are willing to use $\tilde O (n^2)$ samples,
a polynomial-time robust estimator with constant spectral error guarantee 
is known~\cite{DiakonikolasKKLMS16, DiakonikolasKKL16-icml}\footnote{We note that the robust covariance
estimators of ~\cite{DiakonikolasKKLMS16, DiakonikolasKKL16-icml} provide error guarantees 
under the Mahalanobis distance, which is stronger than the spectral norm. Under the stronger metric,
$\Omega(n^2)$ samples are information-theoretically required even without noise.}.
The immediate question that follows is this: 
\begin{quote}
{\em Is there a computationally efficient robust covariance estimator in spectral error
that uses a strongly sub-quadratic sample size, i.e., $O(n^{2-c})$ for a constant $0<c<1$?}
\end{quote}

For (ii), suppose we want to estimate the mean $\mu \in \R^n$
of an identity covariance Gaussian up to $\ell_2$-distance $\eps$, 
under the additional promise that $\mu$ is $k$-sparse, and suppose that $k \ll n^{1/2}$. 
It is well-known that the information-theoretic sample complexity of this problem is $O(k \log n / \eps^2)$,
and the truncated empirical mean achieves the optimal bound.
For simplicity, let us assume that $\eps=1/10$. Now suppose that we observe samples from a corrupted sparse mean 
Gaussian (in Huber's contamination model), where the noise rate $\delta \ll 1/10$.
As in the setting of the previous paragraph, the injection of noise does not change the information-theoretic
sample complexity of the problem: there exist a (computationally inefficient) robust SQ algorithm for this problem
(see~\cite{Li17}) that use $O(k \log n)$ samples. Two recent works~\cite{Li17, DBS17} gave polynomial time
robust algorithms for robust sparse mean estimation
with sample complexity $\tilde{O}(k^2 \log n)$. In summary, in the absence of robustness,
the information-theoretically optimal sample bound is known to be achievable by a computationally efficient algorithm.
In contrast, in the presence of robustness, there is a quadratic gap between the information-theoretic optimum
and the sample complexity of known polynomial-time algorithms. The immediate question is whether
this gap is inherent:

\begin{quote}
{\em Is there a computationally efficient robust $k$-sparse mean estimator
that uses a strongly sub-quadratic sample size , i.e., $O(k^{2-c})$ for a constant $0<c<1$?}
\end{quote}
It is conjectured in~\cite{Li17} that a quadratic gap is in fact inherent for efficient algorithms.
}

\vspace{-0.2cm}

\paragraph{High-Dimensional Hypothesis Testing.}
So far, we have discussed the problem of learning an unknown distribution that is promised to belong
(exactly or approximately) in a given family (Gaussians, mixtures of Gaussians). 
A related inference problem is that of {\em hypothesis testing}~\cite{NeymanP,  lehmann2005testing}: 
Given samples from a distribution in a given family, we want to distinguish 
between a null hypothesis and an alternative hypothesis.
Starting with~\cite{GRexp:00, BFR+:00}, this broad question has been extensively investigated in TCS 
with a focus on discrete probability distributions.
A natural way to solve a distribution testing problem is to 
learn the distribution in question to good accuracy and then check if the 
corresponding hypothesis is close to one satisfying the null hypothesis. 
This testing-via-learning approach is typically suboptimal and
the main goal in this area has been to obtain testers with sub-learning sample complexity.

In this paper, we study natural hypothesis testing analogues 
of the high-dimensional learning problems discussed in the previous paragraphs. 
Specifically, we study the sample complexity of (i) {\em robustly} testing an unknown mean Gaussian, 
and (ii) testing a GMM. 

To motivate (i), we consider arguably the most basic high-dimensional testing task: 
Given samples from a Gaussian $N(\mu, I)$, where $\mu \in \R^n$ is unknown,
distinguish between the case that $\mu = \mathbf{0}$ versus 
$\|\mu\|_2 \geq \eps$. (The latter condition is equivalent, up to constant factors, 
to $\dtv(N(\mu, I), N(0, I)) \geq \eps$.)
The classical test for this task is Hotelling's
T-squared statistic~\cite{Hotelling1931}, which is unfortunately not defined
when the sample size is smaller than the dimension~\cite{BaiS96}.
More recently, testers that succeed in the sub-linear regime have been developed~\cite{Sriv08}
(also see~\cite{BaiS96, chen2010}). In Appendix~\ref{sec:test-app}, we give
a simple and natural tester for this problem that uses $O(\sqrt{n}/\eps^2)$ samples, 
and show that this sample bound is information-theoretically optimal, up to constant factors.

Now suppose that our Gaussianity assumption about the unknown distribution 
is only {\em approximately} satisfied. Formally, we are given samples from a distribution $D$ on $\R^n$ 
which is promised to be either (a) a standard Gaussian $N(0, I)$, 
or  (b) a $\delta$-noisy version of $N(\mu, I)$, where $\mu \in \R^n$ satisfies 
$\|\mu\|_2 \geq \eps$, and the noise rate $\delta$ satisfies $\delta \ll \eps$.
The {\em robust} hypothesis testing problem is to distinguish, 
with high constant probability, between these two cases.
Note that condition (b) implies that $\dtv(D, N(0, I)) = \Omega(\eps)$, 
and therefore the two cases are distinguishable.
\footnote{Robust testing should not be confused with 
tolerant testing, where the completeness is relaxed. In our context, tolerant testing
corresponds to distinguishing between $\dtv(D, N(0, I)) \leq \eps/2$
versus $\dtv(D, N(0, I)) \geq \eps$, where $D = N(\mu, I)$, and is 
easily seen to be solvable with $O(\sqrt{n}/\eps^2)$ samples as well.}

Robust hypothesis testing is of fundamental importance and 
has been extensively studied in robust statistics~\cite{Huber09, HampelEtalBook86, Wilcox97}.
Perhaps surprisingly, it is poorly understood in the most basic settings, 
even information-theoretically. Specifically, the sample complexity 
of our aforementioned robust mean testing problem has remained open.
It is easy to see that the tester of Appendix~\ref{sec:test-app} fails 
in the robust setting. On the other hand, the testing-via-learning approach
implies a sample upper bound of $O(n/\eps^2)$ for our robust testing problem -- by using, e.g., the Tukey median. 
The following question arises:

\begin{question}\label{q:rt}
Is there an {\em information-theoretic} gap between robust testing and non-robust testing?
What is the sample complexity of {\em robustly} testing the mean of a high-dimensional Gaussian?
\end{question}

\noindent We conclude with our hypothesis testing problem regarding GMMs: 
Given samples from a distribution $D$ on $\R^n$, we want to distinguish between the case that
$D = N(0,I)$, or $D$ is a $2$-mixture of identity covariance Gaussians. This is a natural 
high-dimensional testing problem that we believe merits investigation in its own right.
The obvious open question here is whether there exists a tester for this problem 
with {\em sub-learning} sample complexity.

\subsection{Our Results} \label{sec:results}
The main contribution of this paper is a general technique to prove 
lower bounds for a range of high-dimensional estimation problems 
involving Gaussian distributions. We use analytic and probabilistic ideas to 
construct explicit families of hard instances for the estimation problems described in Section~\ref{ssec:motiv}.
Using our technique, we prove super-polynomial Statistical Query (SQ) lower bounds 
that answer Questions~\ref{q:gmm} and~\ref{q:robust} in the negative 
for the class of SQ algorithms. We also show that the observed quadratic
statistical--computational gap for robust sparse mean estimation and robust spectral covariance 
estimation is inherent for SQ algorithms.
As an additional important application of our technique, we obtain information-theoretic 
lower bounds on the sample complexity of the corresponding testing problems. 
(We note that our testing lower bounds apply to {\em all} algorithms.)
Specifically, we answer Question~\ref{q:rt} in the affirmative, 
by showing that the robustness requirement makes 
the Gaussian testing problem information-theoretically harder.
In the body of this section, we state our results 
and elaborate on their implications and the connections between them.

\vspace{-0.3cm}

\paragraph{SQ Lower Bound for Learning GMMs.}
Our first main result is a lower bound of $n^{\Omega(k)}$ on the complexity
of any SQ algorithm that learns an arbitrary $n$-dimensional $k$-GMM 
to constant accuracy (see Theorem~\ref{thm:mixtures-sq} for the formal statement):

\begin{theorem}[SQ Lower Bound for Learning GMMs] \label{thm:mixtures-sq-lb-informal}
Any SQ algorithm that learns an arbitrary $n$-dimensional $k$-GMM to constant accuracy,
for all $n \geq \poly(k)$, requires
$2^{n^{\Omega(1)}} \geq n^{\Omega(k)}$ queries to an SQ oracle of precision $n^{-O(k)}$.
\end{theorem}


\noindent Theorem~\ref{thm:mixtures-sq-lb-informal} establishes a {\em super-polynomial gap} 
between the information-theoretic sample complexity of learning GMMs
and the complexity of {\em any} SQ learning algorithm for this problem. 
It is worth noting that our hard instance is a family of high-dimensional GMMs
whose components are {\em almost non-overlapping}. Specifically, for each GMM 
$F = \sum_{i=1}^k w_i N(\mu_i, \Sigma_i)$ in the family, the total variation distance
between any pair of Gaussian components can be made as large as $1 - 1/\poly(n, k)$. 
More specifically,  for our family of hard instances, 
the sample complexity of both density and parameter learning is $\Theta(k \cdot \log n)$ 
(the standard cover-based algorithm that achieves this sample upper bound is SQ).
In contrast, any SQ learning algorithm for this family of instances requires runtime at least $n^{\Omega(k)}$.

At a conceptual level, Theorem~\ref{thm:mixtures-sq-lb-informal} 
implies that -- as far as SQ algorithms are concerned -- the computational complexity 
of learning high-dimensional GMMs is inherently exponential 
{\em in the dimension of the latent space} -- even though there 
is no such information-theoretic barrier in general. 
Our SQ lower bound identifies a common barrier 
of the strongest known algorithmic approaches for this learning problem,
and provides a rigorous explanation why a long line of algorithmic research 
on this front either relied on strong separation assumptions 
or resulted in runtimes of the form $n^{\Omega(k)}$.

\vspace{-0.3cm}

\paragraph{SQ Lower Bound for Robustly Learning a Gaussian.}
Our second main result concerns the agnostic learning of a single $n$-dimensional Gaussian.
We prove two SQ lower bounds with qualitatively similar guarantees for different versions of this problem.
Our first lower bound is for the problem of agnostically learning 
a Gaussian with unknown mean and identity covariance. 
Roughly speaking, we show that any SQ algorithm that solves this learning 
problem to accuracy $O(\eps)$ requires complexity $n^{\Omega(\log^{1/4}(1/\eps))}$.
We show (see Theorem~\ref{thm:robust-lb-sq} 
for a more detailed statement):

\begin{theorem}[SQ Lower Bound for Robust Learning of Unknown Mean Gaussian] \label{thm:robust-gaussian-sq-lb-informal}
Let $\eps > 0$, $0 < c \leq 1/2$, and $n \geq \poly(\log(1/\eps))$. 
Any SQ algorithm that robustly learns an $n$-dimensional 
Gaussian $N(\mu,I)$, within total variation distance $O(\eps \log(1/\eps)^{1/2-c})$,
requires $2^{n^{\Omega(1)}} \geq n^{\Omega(\log(1/\eps)^{c/2})}$ queries to an SQ oracle 
of precision $n^{-\Omega(\log(1/\eps)^{c/2})}$.
\end{theorem}

\noindent 
Some comments are in order. First, Theorem~\ref{thm:robust-gaussian-sq-lb-informal}
shows a {\em super-polynomial gap} between the sample complexity of agnostically learning
an unknown mean Gaussian and the complexity of SQ learning algorithms for this problem. 
As mentioned in the introduction, $O(n/\eps^2)$ samples information-theoretically suffice 
to agnostically learn an unknown mean Gaussian to within error $O(\eps)$. Second, the robust learning algorithm 
of~\cite{DiakonikolasKKLMS16} runs in $\poly(n, 1/\eps)$ time, can be implemented in the SQ model,
and achieves error $O(\eps \sqrt{\log(1/\eps)})$. 
As a corollary of Theorem~\ref{thm:robust-gaussian-sq-lb-informal}, we obtain 
that the $O(\eps \sqrt{\log(1/\eps)})$ error guarantee of the~\cite{DiakonikolasKKLMS16} algorithm
is best possible among all polynomial-time SQ algorithms.


Roughly speaking, Theorem~\ref{thm:robust-gaussian-sq-lb-informal} shows that any SQ algorithm 
that solves the (unknown mean Gaussian) robust learning problem to accuracy $O(\eps)$ 
needs to have running time at least $n^{\Omega(\log^{1/4}(1/\eps))}$, i.e., {\em quasi-polynomial} in $1/\eps$.
It is natural to ask whether this quasi-polynomial lower bound can be improved to, say, exponential, e.g., $n^{\Omega(1/\eps)}.$
We show that the lower bound of Theorem~\ref{thm:robust-gaussian-sq-lb-informal} 
is qualitatively tight.  We design an (SQ) algorithm that uses 
$O_{\eps}(n^{\sqrt{\log(1/\eps)}})$ SQ queries of inverse quasi-polynomial precision. 
Moreover, we can turn this SQ algorithm into an algorithm in the sampling oracle model
with similar complexity. Specifically, we show
(see Theorem~\ref{thm:upper-bound-learning} and Corollary~\ref{cor:learn-alg}):

\begin{theorem}[SQ Algorithm for Robust Learning of Unknown Mean Gaussian]  \label{thm:sq-algorithm}
Let $D$ be a distribution on $\R^n$ such that $\dtv(D, N(\mu,I)) \leq \eps$ 
for some $\mu \in \R^n$.
There is an SQ algorithm that uses 
$O_{\eps}(n^{O(\sqrt{\log(1/\eps)})})$ SQ's to $D$ of precision $\eps/n^{O(\sqrt{\log(1/\eps)})}$, and
outputs $\widetilde \mu \in \R^n$ such that $\dtv(N(\widetilde \mu,I), N(\mu,I)) \leq O(\eps)$. 
The SQ algorithm can be turned into an algorithm (in the sample model)
with the same error guarantee that has sample complexity and running time 
$O_{\eps} (n^{O(\sqrt{\log(1/\eps)})})$.
\end{theorem}

Theorems~\ref{thm:robust-gaussian-sq-lb-informal} and~\ref{thm:sq-algorithm} give a qualitatively
tight characterization of the complexity of robustly learning an unknown mean Gaussian
in the standard agnostic model, where the noisy distribution $D$ is such that $\dtv(D, N(\mu, I)) \leq \eps$.
Equivalently, $D$ satisfies 
$(1-\eps_1) D + \eps_1 N_1 = (1-\eps_2) N(\mu, I) + \eps_2 N_2$, 
where $N_1, N_2$ are unknown distributions and $\eps_1+\eps_2 \leq \eps$.
A weaker error model, known as {\em Huber's contamination model} in the statistics literature~\cite{Huber64, HampelEtalBook86, Huber09},
prescribes that the noisy distribution $D$ is of the form $D = (1-\eps) N(\mu, I) + \eps N$, 
where $N$ is an unknown distribution. Intuitively, the difference is that in the former model
the adversary is allowed to subtract good samples and add corrupted ones, 
while in the latter the adversary is only allowed to add corrupted ones.
We note that the lower bound of Theorem~\ref{thm:robust-gaussian-sq-lb-informal} 
does not apply in Huber's contamination model. This holds for a reason: 
Concurrent  work~\cite{DiakonikolasKKLMS17}
gives a $\poly(n/\eps)$ time algorithm with $O(\eps)$ error for robustly learning $N(\mu, I)$ in Huber's model.
Hence, as a corollary, we establish a computational
separation between these two models of corruptions. 
We provide an intuitive justification in Section~\ref{ssec:techniques}.

Our second super-polynomial SQ lower bound is for the problem of 
robustly learning a zero-mean unknown covariance Gaussian with respect
to the {\em spectral norm}. Specifically, we show (see Theorem~\ref{thm:robust-cov-lb-sq} for a detailed statement):

\begin{theorem}[SQ Lower Bound for Robust Learning of Unknown Covariance Gaussian] \label{thm:robust-cov-lb-sq-informal}
Let $\eps > 0$, $0 < c \leq 1$, and $n \geq \poly(\log(1/\eps))$. 
Any SQ algorithm that, given access to an $\eps$-corrupted $n$-dimensional  Gaussian $N(0,\Sigma)$, 
with $I/2 \preceq \Sigma  \preceq 2I$, returns $\widetilde \Sigma$ with $\|\widetilde \Sigma - \Sigma\|_2 \leq O(\eps \log(1/\eps)^{1-c})$,
requires at least  $2^{n^{\Omega(1)}} \geq n^{\Omega(\log(1/\eps)^{c/4})}$ queries to an SQ oracle 
of precision $n^{-\Omega(\log(1/\eps)^{c/4})}$.
\end{theorem}

Similarly, Theorem~\ref{thm:robust-cov-lb-sq-informal}
shows a {\em super-polynomial gap} between the information-theoretic sample complexity 
and the complexity of any SQ algorithm for this problem. 
As mentioned in the introduction, $O(n/\eps^2)$ samples information-theoretically suffice 
to agnostically learn the covariance to within spectral error $O(\eps)$. Second, the robust learning algorithm 
of~\cite{DiakonikolasKKLMS16} runs in $\poly(n, 1/\eps)$ time, can be implemented in the SQ model,
and achieves error $O(\eps \log(1/\eps))$ in Mahalanobis distance (hence, also in spectral norm). 
Again, the immediate corollary is that the $O(\eps \log(1/\eps))$ error guarantee of the~\cite{DiakonikolasKKLMS16} algorithm
is best possible among all polynomial-time SQ algorithms.
The lower bound of Theorem~\ref{thm:robust-cov-lb-sq-informal} 
does not apply in Huber's contamination model. This holds for a reason:~\cite{DiakonikolasKKLMS17}
gives a $\poly(n) \cdot 2^{\poly\log (1/\eps)}$ time algorithm with $O(\eps)$ error in Huber's model.

\medskip

\vspace{-0.3cm}

\paragraph{Statistical--Computational Tradeoffs for SQ algorithms.}
Our next SQ lower bounds establish nearly quadratic statistical--computational tradeoffs for robust spectral
covariance estimation and robust sparse mean estimation. We note that both these lower bounds
also hold in Huber's contamination model.
For the former problem, we show 
(see Theorem~\ref{thm:robust-covariance-tradeoff} for the formal statement):

\begin{theorem} \label{thm:robust-covariance-tradeoff-inf}
Let $0 < c < 1/6$, and $n$ sufficiently large. 
Any SQ algorithm that, given access to an $\eps$-corrupted $N(0,\Sigma)$, 
where $\eps \leq c/\ln(n)$  for $\|\Sigma \|_2 \leq \poly(n/\eps)$, 
and returns $\tilde \Sigma$ with $\tilde \Sigma/2 \preceq \Sigma \preceq 2 \tilde \Sigma$, 
requires at least $2^{\Omega(n^{c/3})}$ queries to an SQ oracle
of precision $\gamma = O(n)^{-(1-5c/2)}$.
\end{theorem}

We note that, in order to simulate a single query of the above precision, 
we need to draw $\Omega(1/\gamma^2) = \Omega(n^{2-5c})$ samples from our distribution.
Roughly speaking, Theorem~\ref{thm:robust-covariance-tradeoff-inf} shows that if an SQ algorithm
uses less than this many samples, then it needs to run in $2^{\Omega(n^{c/3})}$ time.
This suggests a nearly-quadratic statistical-computational tradeoff for this problem.

For robust sparse mean estimation  we show (see Theorem~\ref{thm:sq-lb-sparse-mean} for the detailed
statement):

\begin{theorem} \label{thm:sq-lb-sparse-mean-inf} 
Fix any $0< c <1$ and let $n \geq 8 k^2$.
Any SQ algorithm that, given access to an $\eps$-corrupted $N(\mu,I)$, where 
$\eps= k^{-c/4}$, and $\mu \in \R^n$ is promised to be $k$-sparse with $\|\mu\|_2  =  1$, 
and outputs a hypothesis vector $\wh{\mu}$ satisfying $\|\wh{\mu} - \mu\|_2 \leq 1/2$, 
requires at least $n^{\Omega(c k^c)}$ 
queries to an SQ oracle of precision $\gamma = O(k)^{3c/2-1}$.
\end{theorem}

Similarly, to simulate a single query of the above precision, 
we need to draw $\Omega(1/\gamma^2) = \Omega(k^{2-3c})$ samples from our distribution.
Hence, any SQ algorithm that uses this many samples requires runtime at least 
$n^{\Omega(c k^c)}$. This suggests a nearly-quadratic 
statistical-computational tradeoff for this problem.

\vspace{-0.3cm}

\paragraph{Sample Complexity Lower Bounds for High-Dimensional Testing.}
We now turn to our information-theoretic lower bounds on the sample complexity 
of the corresponding high-dimensional testing problems.
For the robust Gaussian mean testing problem {in Huber's contamination model}, we show 
(see Theorem~\ref{thm:robust-testing-lb} for a more detailed statement):

\begin{theorem}[Sample Complexity Lower Bound for Robust Testing of Unknown Mean Gaussian] 
\label{thm:robust-gaussian-testing-lb-informal}
{Fix $\eps>0$.}
Any algorithm with sample access to a distribution $D$ on $\R^n$ 
which satisfies either (a) $D = N(0 ,I)$ or (b) $D$ is a  {$\delta$-noisy $N(\mu, I)$}, 
and $\|\mu\|_2 \geq \eps$, and distinguishes between the two cases 
with probability $2/3$ requires  {(i) $\Omega(n)$ samples if $\delta = \eps/100$, 
(ii) $\Omega(n^{1-c})$ samples if $\delta = \eps/n^{c/4}$, for any constant $0<c<1$.}
\end{theorem}

As stated in the Introduction, without the robustness requirement, 
for any constant $\eps>0$, the Gaussian mean testing problem can be solved with $O_{\eps}(\sqrt{n})$ samples.
Hence, the conceptual message of Theorem~\ref{thm:robust-gaussian-testing-lb-informal} 
is that robustness makes the Gaussian mean testing problem {\em information-theoretically} harder.
In particular, the sample complexity of robust testing is 
essentially the same as that of the corresponding learning problem. 
Theorem~\ref{thm:robust-gaussian-testing-lb-informal}  
can be viewed as a surprising fact because it implies that {\em the effect of 
robustness can be very different for testing versus learning} of the same distribution family.
Indeed, recall that the sample complexity of robustly learning an $\eps$-corrupted 
unknown mean Gaussian, up to error $O(\eps)$, is $O(n/\eps^2)$ -- i.e., the same 
as in the noiseless case.

As a final application of our techniques, we show a sample complexity 
lower bound for the problem of testing whether a spherical GMM 
is close to a Gaussian (see Theorem~\ref{thm:test-mixture-lb} 
for the detailed statement):

\begin{theorem}[Sample Complexity Lower Bound for Testing a GMM] \label{thm:gmm-testing-lb-informal}
Any algorithm with sample access to a distribution $D$ on $\R^n$ 
which satisfies either (a) $D = N(0,I)$, or   {(b) $D = (1/2) N(\mu_1, I) + (1/2) N(\mu_2, I) $, 
such that $\dtv(D, N(0, I)) \geq \eps$,} 
and distinguishes between the two cases with probability at least $2/3$ requires
$\Omega( {n/\eps^2})$ samples.
\end{theorem}

 {Similarly, the sample lower bound of Theorem~\ref{thm:gmm-testing-lb-informal} is optimal, up to constant factors, 
and coincides with the sample complexity of learning the underlying distribution.}

\subsection{Our Approach and Techniques} \label{ssec:techniques}

In this section, we provide a detailed outline of our approach and techniques. 
The structure of this section is as follows:
We start by describing our Generic Lower Bound Construction, followed
by our  {main applications to} the problems of Learning GMMs and Robustly Learning an Unknown Gaussian.
 {We continue with our applications to statistical--computational tradeoffs.}
We then explain how our generic technique can be used to obtain 
our Sample Complexity Testing  Lower Bounds, which rely on essentially 
the same hard instances {as our SQ lower bounds}.
We conclude with a sketch of our new 
(SQ) Algorithm for Robustly Learning an Unknown Mean Gaussian {to optimal accuracy}.

\vspace{-0.3cm}

\paragraph{Generic Lower Bound Construction.}
The main idea of our lower bound construction is quite simple:
We construct a family of distributions $\mathcal{D}$ that 
are standard Gaussians in all but one direction, 
but are somewhat different in the remaining direction (Definition~\ref{def:pv-hidden}). 
Effectively, {\em we are hiding the interesting information about our distributions 
in this unknown choice of direction}. By exploiting the simple fact that 
it is possible to find exponentially many nearly-orthogonal directions (Lemma~\ref{lem:set-of-nearly-orthogonal}), 
we are able to show that any SQ algorithm with insufficient precision needs 
many queries in order to learn {an unknown distribution from $\mathcal{D}$.}

To prove our generic SQ lower bound, we need to bound from below the SQ-dimension 
of our hard family of distributions $\mathcal{D}$. Roughly speaking, the SQ-dimension
of a distribution family (Definition~\ref{def:sq-dim}) 
corresponds to the number of {\em nearly uncorrelated} distributions (with respect to some fixed distribution)
in the family (see Definitions~\ref{def:pc} and~\ref{def:uncor}). It is known that a lower bound on the SQ-dimension
implies a corresponding lower bound on the number and precision of queries of any SQ algorithm
(see Lemma~\ref{lem:sq-from-pairwise}). 

{More concretely, our hard families of distributions are constructed as follows:}
Given a distribution $A$ on the real-line, 
we define a family of high-dimensional distributions $\p_v(x)$, for $v \in \s_n$ a unit $n$-dimensional vector.
The distribution $\p_v$ gives a copy of $A$ in the $v$-direction, 
while being an independent standard Gaussian in the orthogonal directions {(Definition~\ref{def:pv-hidden})}.
Our hard family will be the set $\mathcal{D} = \{\p_v \mid v \in \s_n  \}$.

For the sake of the intuition, we make two observations: 
(1) If $A$ and $N(0, 1)$ have substantially different moments of degree at most $m$,
for some $m$, then $\p_v$ and $N(0, I)$ can be easily distinguished by comparing 
their $m^{th}$-order moment tensors. Since these tensors can be approximated 
in roughly $n^m$ queries (and time), the aforementioned {lower bound} construction would necessarily 
fail unless the low-order moments of $A$ match the corresponding 
low-order moments of $G$. We show that, aside 
from a few mild technical conditions (see~Condition~\ref{cond:moments}), 
this moment-matching condition is essentially sufficient for our purposes. 
If the degree at most $m$ tensors agree, we need to approximate tensors of degree $m+1$. 
Intuitively, in order to extract useful information from these higher degree tensors, 
one needs to approximate essentially all of the $n^{m+1}$ many such tensor entries.
(2) A natural approach to distinguish between $\p_v$ and $N(0, I)$ would be
via random projections. As a critical component of our proof, we show (see Lemma~\ref{lem:1d-proj})
that a random projection of $\p_v$ will be exponentially close to $N(0, 1)$ with high probability.
Therefore, a random projection-based algorithm would require 
exponentially many random directions until it found a good one.



We now proceed with a somewhat more technical description of our proof.
To bound from below the SQ-dimension of our hard family of distributions, we proceed as follows:
The definition of the pairwise correlation {(Definition~\ref{def:pc})} implies
we need to show that $\int \p_v \p_{v'} / G \approx 1$, 
where $G \sim N(0, I)$ is the Gaussian measure, 
for any pair of unit vectors $v, v'$ that are nearly orthogonal. 
{To prove this fact, we make essential use of the Gaussian (Ornstein--Uhlenbeck) noise operator 
and its properties (see, e.g.,~\cite{ODonnell:2014}). We explain this connection in the following paragraph.}

By construction of the distributions $\p_v,  \p_{v'}$, it follows that
in the directions perpendicular to both $v$ and $v'$, 
the relevant factors integrate to $1$. 
Letting $y=v \cdot \bx$ and $z=v'\cdot \bx$ and 
letting $y', z'$ be the orthogonal directions to $y$ and $z$, we need to consider the integral
$$
\int A(y)A(z)G(y')G(z')/G(\bx) \;.
$$
Fixing $y$ and integrating over the orthogonal direction, we get
$$
\int A(y)/G(y) \int A(z)G(z')dy' \;.
$$
Now, if $v$ and $v'$ are (exactly) orthogonal, $z=y'$ and the inner integral equals $G(y)$. 
When this is not the case, the $A(z)$ term is not quite vertical 
and the $G(z')$ term not quite horizontal, so instead what we get 
is only \emph{nearly} Gaussian. In general, the inner integral is equal to
$$
U_{v \cdot v'}A(y) \;,
$$
where $U_t$ is the member of the Ornstein--Uhlenbeck semigroup, 
$U_t f(z) = \E[f(tz+\sqrt{1-t^2}G)].$ 
{We show} that this quantity is {\em close} 
to a Gaussian, when $v\cdot v'$ is {\em close} to $0$ (see Lemma~\ref{lem:cor}).

The core idea of the analysis relies on the fact that $U_t A$ is a smeared out version of $A$. 
As such, it only {retains} the most prominent features of $A$, namely its low-order moments. 
In fact, we are able to show that if $A$ and $G$ agree in their first $m$ moments, 
then $U_t A$ is $O_m(t^m)$-close to a Gaussian (see Lemma~\ref{lem:1d-proj}), 
and thus the integral in question is $O_m((|v\cdot v'|)^m)$-close to $1$. 
This intuition is borne out in a particularly clean way by writing $A/G$ in the basis of Hermite polynomials. 
The moment-matching condition implies that the decomposition involves none of the Hermite polynomials 
of degrees $1$ through $m$. However, the Ornstein--Uhlenbeck operator, $U_t$, is diagonalized 
by the basis $H_i G$ with eigenvalue $t^i$. Thus, if $A-G$ can be written in this basis 
with no terms of degree less than $m$, applying $U_t$ decreases 
the size of the function by a multiple of approximately $t^m$.

{So far, we have provided a proof sketch of the following statement (Lemma~\ref{lem:cor}): 
When two unit vectors $v, v'$ are {\em nearly} orthogonal, then the distributions $\p_v, \p_{v'}$ are {\em nearly} uncorrelated.}
Since, {for $0<c<1/2$}, we can pack {$2^{\Omega(n^c)}$} unit vectors $v$ onto the sphere 
so that their pairwise inner products are at most {$n^{c-1/2}$} (Lemma~\ref{lem:set-of-nearly-orthogonal}), 
we obtain an SQ-dimension lower bound of our hard family. 
In particular, to learn the distribution $\p_v$, for unknown $v$,  
any SQ algorithm requires either {$2^{\Omega(n^c)}$} queries or queries of 
accuracy better than {$O(n)^{(m+1)(c-1/2)}$ (Proposition~\ref{prop:generic-sq})}. 
This completes the proof sketch of our generic construction.

\medskip

In our main applications, we construct one-dimensional distributions $A$ 
satisfying the necessary moment-matching conditions for $m$ taken to be {\em super-constant}, 
thus obtaining {\em super-polynomial} SQ lower bounds. {For our quadratic 
statistical--computational tradeoffs, we match a constant number of moments.}
In the following paragraphs, we explain how we apply our 
framework to bound the SQ dimension for: (i) learning $k$-GMMs to constant accuracy, (ii) robustly learning an $\eps$-corrupted 
Gaussian to accuracy $O(\eps)$, {and (iii) robustly estimating a Gaussian covariance 
within constant spectral error and robustly estimating a sparse Gaussian mean to constant $\ell_2$-error.}
In all cases, we construct a distribution $A$ on the real-line that satisfies the necessary moment-matching conditions
such that the family $\mathcal{D} = \{ \p_v \mid v \in \s_n\}$ {belongs in the appropriate class,
e.g., is a $k$-GMM for (i), an $\eps$-corrupted Gaussian for (ii), etc.}

\vspace{-0.3cm}

\paragraph{SQ Lower Bound for Learning $k$-GMMs.}
The properties of our one-dimensional distribution $A$ are summarized
in Proposition~\ref{prop:mixtures}. Specifically, we construct a distribution $A$ on the real line 
that is a $k$-mixture of one-dimensional ``skinny'' Gaussians, $A_i$,
that agrees with $N(0, 1)$ on the first $m = 2k-1$ moments (condition (i)).  
For technical reasons, we require that the chi-squared divergence
of $A$ to $N(0, 1)$ is bounded from above by an appropriate quantity (condition (iv)). 
The Gaussian components, $A_i$, have the same variance and appropriately bounded means
(condition (ii)). We can also guarantee that the components $A_i$ are almost non-overlapping (condition (iii)).
This implies that the corresponding high-dimensional distributions $\p_v, \p_v'$ 
will be at total variation distance close to $1$ from each other when the directions $v, v'$ 
are nearly orthogonal, and moreover their means will be sufficiently separated.

To establish the existence of a distribution $A$ with the above properties, 
we proceed in two steps:  First, we construct (Lemma~\ref{lem:Gauss-Hermite}) 
a discrete one-dimensional distribution $B$ supported on $k$ points, lying
in an $O(\sqrt{k})$ length interval, that agrees with $N(0, 1)$ on the first $k$ moments. 
The existence of such a distribution $B$ essentially follows from standard tools on Gauss-Hermite quadrature. 
The distribution $A$ is then obtained (Corollary~\ref{cor:mixture}) 
by adding a zero-mean skinny Gaussian to an appropriately rescaled version of $B$.
Additional technical work (Lemmas~\ref{lem:mixtures-pv-sep} and~\ref{lem:mixtures-chi-squared}) 
gives the other conditions.

Our family of hard high-dimensional instances will consist of GMMs that look
like almost non-overlapping ``parallel pancakes'' and is 
reminiscent of the family of instances considered in {Brubaker and Vempala}~\cite{BV:08}.
For the case of $k=2$, consider a $2$-GMM 
where both components have the same covariance 
that is far from spherical, the vector between the means is parallel 
to the unit eigenvector with smallest eigenvalue, and the distance between the means 
is a large multiple of the standard deviation in this direction 
(but a small multiple of that in the orthogonal direction).
This family of instances was considered in~\cite{BV:08},
who gave an efficient spectral algorithm to learn them.

Our lower bound construction can be thought of as 
$k$ ``parallel pancakes'' in which the means lie in a one-dimensional subspace, 
corresponding to the smallest eigenvalue of the identical covariance matrices of the components. 
All $n-1$ orthogonal directions will have an eigenvalue of $1$, which is much larger than the smallest eigenvalue. 
In other words, for each unit vector $v$,  the $k$-GMM $\p_v$ will consist of $k$ ``skinny'' 
Gaussians whose mean vectors all lie in the direction of $v$. Moreover, each pair of components 
will have total variation distance very close to $1$ and their 
mean vectors are separated by  $\Omega(1/\sqrt{k})$.
{We emphasize once more that our hard family of instances 
is learnable with $O(k \log n)$ samples -- both for density estimation and parameter estimation.
On the other hand, any SQ learning algorithm for 
the family requires $n^{\Omega(k)}$ time.}

\vspace{-0.3cm}

\paragraph{SQ Lower Bounds for Robustly Learning Unknown Gaussian.}
In the agnostic model, 
there are two types of adversarial noise to handle:
{\em subtractive noise} -- corresponding to the good samples removed by the adversary -- 
and {\em additive noise} -- corresponding to the bad points added by the adversary.
The approach of~\cite{DiakonikolasKKLMS16} does not do anything to address
subtractive noise, but shows that this type of noise can incur ``small'' error, e.g., 
at most $O( \eps \sqrt{\log(1/\eps)})$ for the case of unknown mean. 
For additive noise,~\cite{DiakonikolasKKLMS16} uses an iterative 
spectral algorithm to filter out outliers.

{For concreteness, let us consider the case of robustly learning $N(\mu, I)$.}
Intuitively, achieving error $O(\eps)$ in the agnostic model is hard for the following reason: 
the two types of noise can collude so that the first few moments of the corrupted distribution 
are indistinguishable from those of a Gaussian whose mean vector has distance 
$\Omega(\eps  \sqrt{\log(1/\eps)})$ from the true mean.

To formalize this intuition, for our robust SQ learning lower bound, we construct a distribution $A$ on the real line 
that agrees with $N(0, 1)$ on the first $m = \Omega(\log^{1/4}(1/\eps))$ moments and is $\eps/100$-close in total variation distance
to $G' = N(\eps, 1)$ (see Proposition~\ref{prop:A-for-agnostic-learning-sq}). 
We achieve this by taking $A$ to be the Gaussian $N(\eps, 1)$ outside its effective
support, while in the effective support we add an appropriate degree-$m$ 
univariate polynomial $p$ satisfying the appropriate moment conditions. 
By expressing this polynomial as a linear combination of appropriately scaled {\em Legendre polynomials}, 
we can prove that its $L_1$ and $L_{\infty}$ norms 
within the effective support of $G'$ are much smaller than $\eps$ (see~Lemma~\ref{lem:L1inf}). 
This result is then used to bound from above the distance of $A$ from $G'$, 
which gives our SQ lower bound.

{We use a similar technique to prove our SQ lower bound for robust covariance estimation in spectral norm. 
Specifically, we construct a distribution $A$ that agrees with $N(0, 1)$ on the first 
$m = \Omega(\log(1/\eps))$ moments and is $\eps/100$-close in total variation distance
to $G' = N(0, (1-\delta)^2)$, for some $\delta = O(\eps)$ (see Proposition~\ref{prop:A-for-agnostic-learning-cov-sq}).
We similarly take $A$ to be the Gaussian $G'$ outside its effective
support, while in the effective support we add an appropriate degree-$m$ 
univariate polynomial $p$ satisfying the appropriate moment conditions. 
The analysis proceeds similarly as above.

\vspace{-0.3cm}

\paragraph{Statistical--Computational Tradeoffs for SQ algorithms.}
For robust covariance estimation in spectral norm, our one-dimensional distribution is selected to be
$A = (1-\eps) N(0, \sigma) + \eps N_1$, where $N_1$ is a mixture of $2$ unit-variance Gaussians
with opposite means. By selecting $\sigma$ appropriately, 
we can have $A$ match the first $3$ moments of $N(0, 1)$, see 
Theorem~\ref{thm:robust-covariance-tradeoff}. For robust sparse mean estimation, it suffices to take
$A = (1-\delta) N(\eps, 1) + \delta N_1$, where $N_1$ is a 
unit-variance Gaussian selected so that $\E[A]=0$. An important aspect of both these constructions is that 
the chi-squared distance $\chi^2(A, N(0, 1))$ needs to be as small as possible. Indeed, since we only match a small number of moments,
our bound on $\chi^2(A, N(0, 1))$ crucially affects the accuracy of our SQ queries (Proposition~\ref{prop:generic-sq}).}

\vspace{-0.3cm}

\paragraph{Sample Complexity Testing Lower Bounds.}
Our sample complexity lower bounds follow from standard information-theoretic
arguments, and rely on the same lower bound {instances}
and correlation bounds (i.e., bounds on $\int \p_v \p_{v'}/G$) established in our SQ lower bounds.
In particular, we consider the problem of distinguishing between 
the distribution $G \sim N(0, I)$ and the distribution $\p_v$ for a randomly chosen unit vector $v \in \s_n$ 
using $N$ independent samples. Let $G^{\otimes N}$ denote the distribution on $N$ 
independent samples from $G$, and $\p_v^{\otimes N}$ the distribution 
obtained by picking a random $v$ and then taking $N$ independent samples from $\p_v$.
If it is possible to reliably distinguish between these cases, 
it must be the case that the chi-squared divergence $\chi(\p_v^{\otimes N},G^{\otimes N})$ 
is substantially larger than $1$. 
This is
$
\int_{v,v',x_i} \littleprod_{i=1}^N \p_v(x_i)\p_{v'}(x_i)/G(x_i) dvdv'dx_i.
$
Note that after fixing $v$ and $v'$ the above integral separates as a product, giving
\begin{equation}\label{chiSquaredIntegralTestingEquation}
\int_{v,v'} \left( \int \p_v(x)\p_{v'}(x)/G(x)dx\right)^Ndvdv' \;.
\end{equation}
Note that the inner integral was bounded from above by roughly $(1+(v\cdot v')^m)$. 
{A careful analysis of the distribution of the angle between two random unit vectors allows
us to show that, unless $N = \Omega(n)$}, the chi-squared divergence 
is close to $1$, and thus that this testing problem is impossible.

\vspace{-0.3cm}

\paragraph{Algorithm for Robustly Learning Unknown Mean Gaussian.}
We give an SQ algorithm with $O(\eps)$-error for robustly learning an unknown mean Gaussian, 
showing that our corresponding SQ lower bound is qualitatively tight.
 Our algorithm builds on the filter technique 
of~\cite{DiakonikolasKKLMS16}, generalizing it to the more involved setting of higher-order tensors.

As is suggested by our SQ lower bounds, the obstacle to learning the mean robustly, 
is that there are $\eps$-noisy Gaussians that are $\Omega(\eps)$-far in variation distance 
from a target Gaussian $G$, and yet match $G$ in all of their first $O(\log^{1/4}(1/\eps))$ moments. 
For our algorithm to circumvent this difficulty, it will need to approximate all of the $t^{th}$-order
tensors for $t \leq k=\Omega(\log^{1/4}(1/\eps))$. Note that this already requires $n^k$ SQ queries.

The first thing we will need to show is that $k$ moments {\em suffice}, for an appropriate parameter
$k$. Because of our lower bound construction, we know that $k$ needs to be at least $\Omega(\log^{1/4}(1/\eps))$. 
We show that $k=O(\log^{1/2}(1/\eps))$ suffices.
Specifically, we prove a one-dimensional moment-matching lemma (Lemma~\ref{lem:moment-matching})
establishing the following: If an $\eps$-noisy one-dimensional Gaussian approximately matches 
a reference Gaussian $G$ in all of its first $k$ moments, where $k=\Theta(\log^{1/2}(1/\eps))$ (i.e., 
quadratically larger than our lower bound), then it must be $O(\eps)$-close to $G$ 
in variation distance. We note that it suffices to prove this statement in the one-dimensional case, 
as we can just project onto the line between the means.

We now proceed to describe our algorithm:
Using the basic filter algorithm from~\cite{DiakonikolasKKLMS16}, we start by learning the 
true mean to error $O(\eps\sqrt{\log(1/\eps)})$. By translating, we can 
assume that the mean is this close to $0$. We need to robustly approximate
the low-order moments of our target Gaussian $G'$. This is complicated by the fact 
that even a small fraction of errors can have a huge impact on the moments 
of the distribution. However, any large errors are easily detectable. 
In particular, if any $t^{th}$ moment tensor differs 
substantially from that of the standard Gaussian, it will necessarily 
imply the presence of errors. In particular, it will allow us to construct 
a polynomial $p$ so that $\E[p(X)]-\E[p(G')]$ (where $X$ is a noisy version of $G'$) 
is much larger than $\eps \|p(G')\|_2$. If this is the case, then many of our errors, 
$x$, must have $p(x)$ very far from the mean. By standard concentration inequalities, 
this will allow us to identify these points as almost certainly being errors. 
This in turn lets us build a filter to clean-up our distribution $X$, making it closer to $G'$.

Repeatedly applying filters as necessary, we can reduce to the case where 
the higher-order moments of $X$ are close to the higher-order moments of $G$. 
This will tell us that, in almost all directions, the first $k$ moments of $X$ 
match the corresponding moments of $G$. By our moment-matching lemma, this will imply 
that the mean of $G'$ is close to $0$ in these directions. 
We will then only need to approximate the mean of the projection of $G'$ 
onto the low-dimensional subspace $V$ in which these moments fail to match. 
This approximation can be done in a brute-force manner (in time exponential in $\dim(V)$, 
which is still relatively small), completing the description of the algorithm.

\subsection{Related Work} \label{ssec:related-work}
This work studies learning and testing 
high-dimensional structured distributions. Distribution learning and testing are two
of the most fundamental inference tasks in statistics with
a rich history (see, e.g.,~\cite{NeymanP, BBBB:72, DG85, Silverman:86,Scott:92,DL:01,  lehmann2005testing})
that date back to Karl Pearson.
The main criteria to evaluate the performance of an estimator 
are its sample complexity and its computational complexity. 
Despite intensive investigation for several decades by different communities, 
the (sample and/or computational) complexity of many learning and testing problems 
is still not well-understood, even for some surprisingly simple high-dimensional settings.
In the past few decades, a long line of work  within TCS~\cite{KMR+:94short, Dasgupta:99, FreundMansour:99short, AroraKannan:01, VempalaWang:02, CGG:02, MosselRoch:05, BV:08, KMV:10, MoitraValiant:10, BelkinSinha:10, DDS12soda, DDS12stoc, CDSS13, DDOST13focs, CDSS14, CDSS14b, ADLS15, DDS15, DDKT15, DKS15, DKS16}
has focused on designing efficient estimators in a variety of settings. 
We have already mentioned the most relevant references for the specific questions we consider
in Section~\ref{ssec:motiv}.

With respect to computational lower bounds for unsupervised estimation problems, 
the most relevant references are the works~\cite{Feldman13, FeldmanPV15, KannanV16}
that show SQ lower bounds for the planted clique and related planted-like problems.
It should be noted that, beyond the fact that we also use the concept of SQ dimension, 
our techniques are entirely different than theirs. Prior work by Feldman, O'Donnell, and Servedio~\cite{FOS:08} 
implicitly showed  an SQ lower bound of  $n^{\Omega(\log k)}$ for the problem of learning $k$-mixtures of product 
distributions over $\{0, 1\}^n$. This was obtained by a straightforward reduction from the problem
of learning $k$-leaf decision trees over $n$ Boolean variables. Our lower bound construction for learning 
GMMs is entirely different from~\cite{FOS:08} that relied on the obvious combinatorial structure of the discrete setting.

A related line of work gives statistical-computational tradeoffs 
for sparse PCA~\cite{BerthetR13, BR13, Ma:2015, wang2016}, based 
on various computational hardness assumptions. 
{These results are of similar flavor as our statistical--computational tradeoffs for SQ algorithms
(Theorems~\ref{thm:robust-covariance-tradeoff-inf} and~\ref{thm:sq-lb-sparse-mean-inf}).
An important difference between these tradeoffs and the super-polynomial SQ lower bounds 
we prove in this paper (Theorems~\ref{thm:mixtures-sq-lb-informal},~\ref{thm:robust-gaussian-sq-lb-informal}, 
and~\ref{thm:robust-cov-lb-sq-informal}) is that the aforementioned sparse problems are known to be tractable 
if we increase the sample size by a quadratic factor
beyond the information-theoretic limit. In contrast, our main SQ 
lower bound results establish a {\em super-polynomial} gap between the information-theoretic
limit and the computational complexity of any SQ algorithm.}

Finally, we remark that in the supervised setting of PAC learning Boolean functions, 
a number of hardness results are known based on various complexity assumptions, 
see, e.g., ~\cite{KKMS:08, KlivansSherstov:06, FGK+:06, KlivansK14, DanielyLS14, Daniely16} 
for the problems of learning halfspaces  and learning intersections thereof.

\subsection{Discussion and Future Directions} \label{ssec:open-problems}
The main contribution of this paper is a technique that gives essentially tight 
SQ lower bounds for a number of fundamental high-dimensional learning problems, including 
learning GMMs and robustly learning a single Gaussian. To the best of our knowledge, these are the first such 
lower bounds for high-dimensional distribution learning problems in the continuous setting. As a corollary, 
we provide a rigorous explanation of the observed
(super-polynomial) gap between the sample complexity of these problems and the runtime of the best known algorithms.

Our work naturally raises a number of interesting future directions.
A natural open problem is to extend our lower bound technique to broader families of
high-dimensional distributions. More concretely, is there a ${k^{\omega(1)}} \poly(n)$ SQ 
lower bound for learning $k$-mixtures of $n$-dimensional {\em spherical} Gaussians? 
Note that our $n^{\Omega(k)}$ lower bound does not apply for the spherical case, as it 
crucially exploits the structure of the covariance matrices. In fact, faster learning algorithms 
for the spherical case are known~\cite{SOAJ14}, albeit with exponential dependence on 
the number $k$ of components. {More broadly, can we extend our techniques to other 
families of structured high-dimensional distributions (e.g., mixtures of other distribution families)?}




\subsection{Organization} \label{sec:structure}
The structure of this paper is as follows:
In Section \ref{sec:prelims}, we introduce basic notation, definitions, 
and a number of useful facts that will be required throughout the paper.
Our SQ lower bounds are established in Sections~\ref{sec:generic-lb}--\ref{sec:tradeoffs}.
Specifically, in Section~\ref{sec:generic-lb}, we give our generic high-dimensional SQ lower bound 
construction, assuming the existence of a one-dimensional density 
satisfying the necessary moment conditions. In Sections~\ref{sec:gmm-sq-lb},~\ref{sec:robust-sq-lb}, 
and~\ref{sec:tradeoffs}, we construct the appropriate one-dimensional densities, thereby establishing 
our SQ lower bounds. Specifically, Sections~\ref{sec:gmm-sq-lb} and \ref{sec:robust-sq-lb}
give our super-polynomial SQ lower bounds for the problems of learning GMMs and robustly learning an unknown Gaussian.
Section~\ref{sec:tradeoffs} gives our quadratic statistical--computational tradeoffs (for SQ algorithms) for the problems of 
robust covariance estimation in spectral norm and robust sparse mean estimation.
Section~\ref{sec:testing} gives our (information-theoretic) sample complexity lower bounds for high-dimensional testing. 
Finally, in Section~\ref{sec:alg} we present our  (SQ) algorithm for robustly learning an unknown mean Gaussian
with optimal accuracy, whose runtime qualitatively matches our SQ lower bound from Section~\ref{sec:robust-sq-lb}.

\paragraph{Acknowledgements.} 
This project evolved over a number of years.
We would like to thank Vitaly Feldman for answering numerous questions about the Statistical Query model;
Andy Drucker for useful discussions on Question~\ref{q:gmm};  Anup B. Rao for asking a question
that motivated Theorem~\ref{thm:robust-covariance-tradeoff}; Weihao Kong and Gregory Valiant for useful
discussions on Question~\ref{q:rt}; and Ankur Moitra and Eric Price for feedback on a 
previous version of this paper.

\section{Definitions and Preliminaries} \label{sec:prelims}

\subsection{Notation and Basic Definitions} \label{sec:notation}
For $n \in \Z_+$, we denote by $[n]$ the set $\{1, 2, \ldots, n\}$.
We will denote by $\s_n$ the Euclidean unit sphere in $\R^n$.
If $v$ is a vector, we will let $\| v \|_2$ denote its Euclidean norm.
If $M$ is a matrix, we will let $\| M \|_2$ denote its spectral norm, 
and $\| M \|_F$ denote its Frobenius norm.

Our basic object of study is the Gaussian (or Normal) distribution and 
finite mixtures of Gaussians:

\begin{definition} \label{def:gaussian}
The $n$-dimensional \emph{Gaussian distribution} $N(\mu, \Sigma)$ with mean vector $\mu \in \R^n$ 
and covariance matrix $\Sigma \in \R^{n \times n}$ is the distribution with probability density function
$$f(x) = (2\pi)^{-n/2}\det(\Sigma)^{1/2}\exp\left(-\frac12 (x - \mu)^T\Sigma^{-1}(x - \mu)\right).$$
\end{definition}

\begin{definition} \label{def:gmm}
An $n$-dimensional {\em $k$-mixture of Gaussians} ($k$-GMM) is a distribution on $\R^n$ 
with probability density function defined by $F(x) = \sum_{j=1}^k w_j N(\mu_j, \Sigma_j),$
where $w_j \geq 0$, for all $j$, and $\sum_{j=1}^k w_j = 1$.
\end{definition}

Throughout the paper, we will make extensive use of the pdf 
of the standard one-dimensional Gaussian $N(0, 1)$, 
which we will denote by $G(x)$. 

\begin{definition} \label{def:dtv}
The {\em total variation distance} between two distributions (with probability density functions) $\p, \q: \R^n \to \R_+$
is defined to be
$\dtv\left(\p, \q \right) \eqdef (1/2) \cdot \| \p -\q  \|_1 = (1/2) \cdot \littleint_{ x \in \R^n} |\p(x)-\q(x)| dx.$
{The {\em $\chi^2$-divergence} of $\p, \q$ is
$\chi^2(\p, \q) \eqdef  \littleint_{ x \in \R^n} (\p(x)-\q(x))^2/\q(x)  dx =  \littleint_{ x \in \R^n} \p^2(x)/\q(x) dx -1$.}
\end{definition}




\subsection{Formal Problem Definitions} \label{sec:problem-defs}
We record here the formal definitions of the problems that we study.
Our first problem of interest is learning a mixture of $k$ arbitrary high-dimensional Gaussians:

\begin{definition}[Density Estimation/Proper Learning of GMMs] \label{def:learn-gmms}
Let ${\cal G}_{n, k}$ be the family of $n$-dimensional $k$-GMMs.
The problem of {\em density estimation} for ${\cal G}_{n, k}$ is the following:
Given $\eps>0$ and sample access to an unknown $\p \in {\cal G}_{n, k}$, 
with probability $9/10$, output a hypothesis distribution $\h$ such that 
$\dtv(\h, \p) \leq \eps$. The problem of {\em proper learning} for ${\cal G}_{n, k}$
is the same with the additional requirement that $\h \in {\cal G}_{n, k}$.
\end{definition}

\noindent 
{We also consider the more challenging task of parameter estimation: Given samples
from a distribution $\p = \sum_{i=1}^k G_i$, where $G_i$ is a weighted Gaussian, the goal is to learn
a distribution $\q$ that can be written as $\q = \sum_{i=1}^k H_i$, 
where $H_i$ is a weighted Gaussian and $\dtv(H_i, G_i)$ is small for all $i$.}

\smallskip

\noindent Our next question is the problem of robustly learning a Gaussian in the standard
agnostic model:

\begin{definition}[Robust Learning of Unknown Gaussian]  \label{def:robust-learn-gaussian}
The problem of robust (agnostic) learning of an unknown Gaussian is the following:
Given $\eps > 0$, and sample access to an unknown distribution $D$ with $\dtv (D,  N(\mu, \Sigma)) \leq \eps$, 
for some $\mu \in \R^n \textrm{ and } \Sigma \in \R^{n \times n}$, output a hypothesis distribution $\h$ such that, with probability
at least $9/10$, it holds that $\dtv (\h,  N(\mu, \Sigma)) = f(\eps)$, for some $f: \R_+ \to \R_+$.
\end{definition}

{Our SQ lower bounds apply to two special cases of this problem: when $\mu$ is unknown and $\Sigma = I$, 
and when $\mu=0$ and $\Sigma$ is unknown. For the latter case, our lower bound applies even for learning with respect to
the spectral norm (which is weaker than approximation in variation distance).}

\smallskip

We now define the problem of robustly testing a Gaussian in Huber's model.
We remind the reader that our tight sample complexity lower bound applies to this weaker model as well.

\begin{definition}[Robust Testing of Unknown Mean Gaussian] \label{def:robust-test-gaussian}
The problem of robust testing of an unknown mean Gaussian is the following:
Given {$\eps>0, 0< \delta < \eps/2$} and sample access to an unknown distribution $D$ over $\R^n$ 
with the promise that one of the following two cases is satisfied: (i) $D = N(0, I)$, 
or (ii) $D = (1-\delta)N(\mu, I) + \delta N_1$, where $N_1$ is an unknown noise distribution and 
$\|\mu\|_2 \geq \eps$, the goal is to correctly distinguish between the two cases with confidence
probability $2/3$.
\end{definition}

Finally, our problem of testing GMMs is the following:

\begin{definition}[Testing Spherical GMMs] \label{def:test-gmm}
The problem of testing of a spherical GMM is the following:
Given $\eps > 0$ and sample access to an unknown distribution $D$ with the promise that one
of the following two cases is satisfied: (a) $D = N(0,I)$, or  (b) $D$ is 
a $2$-GMM $w_1 N(\mu_1,I) + w_2 N(\mu_2, I)$ in $\R^n$ 
{which satisfies $\dtv(D, N(0,I)) \geq \eps$,}
distinguish between the two cases with probability at least $2/3$.
\end{definition}

\subsection{Basics on Statistical Query Algorithms over Distributions} \label{sec:sq-prelims}

We begin by recording the necessary definitions of Statistical algorithms for problems over distributions.
All the definitions and facts in this section are from~\cite{Feldman13}.
We start by defining a general search problem over distributions. 

\begin{definition}[Search problems over distributions] \label{def:search}
Let $\mathcal{D}$ be a set of distributions over $\R^n$, let $\mathcal{F}$
be a set of {\em solutions} and $\mathcal{Z}: \mathcal{D} \to 2^{\mathcal{F}}$ be a map from a distribution 
$D \in \mathcal{D}$ to a subset of solutions $\mathcal{Z}(D) \subseteq \mathcal{F}$ that are defined to be valid solutions
for $D$. The {\em distributional search problem} $\mathcal{Z}$ over $\mathcal{D}$ and $\mathcal{F}$
is to find a valid solution $f \in \mathcal{Z}(D)$ given access to (an oracle or samples from) an unknown $D \in \mathcal{D}$.
\end{definition}

For general search problems over a distribution, we define SQ algorithms as algorithms that do not see
samples from the distribution but instead have access to an SQ oracle. We consider two types of SQ oracles
from the literature. 
\begin{enumerate}
\item $\mathrm{STAT}(\tau)$: For a tolerance parameter $\tau >0$ and any bounded function $f: \R^n \to [-1, 1]$, 
$\mathrm{STAT}(\tau)$ returns a value $v \in \left[\E_{x \sim D} [f(x)] - \tau,  \E_{x \sim D} [f(x)] + \tau \right]$.

\item $\mathrm{VSTAT}(t)$: For a sample size parameter $t >0$ and any bounded function $f: \R^n \to [0, 1]$, 
$\mathrm{VSTAT}(t)$ returns a value $v \in \left[ \E_{x \sim D} [f(x)] - \tau,  \E_{x \sim D} [f(x)] + \tau \right]$, where 
$\tau = \max \left\{\frac{1}{t}, \sqrt{\frac{\var_{x \sim D}[f(x)]}{t}} \right\}$, 
where $\var_{x \sim D}[f(x)] = \E_{x \sim D} [f(x)] \left(1-\E_{x \sim D} [f(x)] \right)$.
\end{enumerate}
The first oracle was defined by Kearns~\cite{Kearns:98} and the second was introduced in~\cite{Feldman13}.
These oracles are known to be polynomially equivalent~\cite{Feldman13}.
Also note that these oracles can return any value within the given tolerance, and therefore can make 
adversarial choices.

The main technical tool that allows us to prove unconditional lower bounds on the complexity of SQ algorithms
is an appropriate notion of Statistical Query (SQ) dimension. Such a notion was defined in the context of PAC
learning of Boolean functions in~\cite{BFJ+:94}, and subsequently generalized to search problems
over distributions in~\cite{Feldman13}. We will require the simpler 
definition from Section 3 of that work that relies on pairwise correlations:

\begin{definition}[Pairwise Correlation] \label{def:pc}
The pairwise correlation of two distributions with probability density functions 
$D_1, D_2 : \R^n \to \R_+$ with respect to a distribution with density $D: \R^n \to \R_+$, 
where the support of $D$ contains the supports of $D_1$ and $D_2$, 
is defined as $\chi_{D}(D_1, D_2) \eqdef \int_{\R^n} D_1(x) D_2(x)/D(x) dx -1$.
\end{definition} 

We remark that when $D_1=D_2$ in the above definition, 
the pairwise correlation is identified with the $\chi^2$-divergence between $D_1$ and $D$, 
i.e.,  $\chi^2(D_1, D) \eqdef \int_{\R^n} D_1(x)^2/D(x) dx -1$. 

We will also need the following definition: 

\begin{definition} \label{def:uncor}
We say that a set of $m$ distributions $\mathcal{D} = \{D_1, \ldots , D_m \}$ over $\R^n$ 
is $(\gamma, \beta)$-correlated relative to a distribution $D$ over $\R^n$ if 
$$|\chi_D(D_i, D_j)| \leq \begin{cases} \gamma \mbox{ if } i \neq j \\ \beta \mbox{ if } i=j .\end{cases}$$
\end{definition}

We are now ready to define our notion of dimension:

\begin{definition}[Statistical Query Dimension] \label{def:sq-dim}
For $\beta, \gamma > 0$, a search problem $\mathcal{Z}$ over a set of solutions $\mathcal{F}$,
and a class of distributions $\mathcal{D}$ over $\R^n$, 
let $m$ be the maximum integer such that there exists a reference distribution $D$ over $\R^n$ 
and a finite set of distributions $\mathcal{D}_D \subseteq \mathcal{D}$ 
such that for any solution $f \in \mathcal{F}$, $\mathcal{D}_f =\mathcal{D}_D \setminus \mathcal{Z}^{-1}(f)$ 
is $(\gamma, \beta)$-correlated relative to $D$ and $|\mathcal{D}_f| \geq m.$ 
We define the {\em statistical (query) dimension} with pairwise correlations $(\gamma, \beta)$ 
of $\mathcal{Z}$ to be $m$ and denote it by $\mathrm{SD}(\mathcal{Z},\gamma,\beta)$.
\end{definition}

Our lower bounds proceed by bounding from below the statistical query dimension
of the considered distribution learning problems. The corresponding lower bounds
on the complexity of SQ algorithms for these problems are a corollary of the following
result from \cite{Feldman13}:

\begin{lemma}[Corollary 3.12 in~\cite{Feldman13}] \label{lem:sq-from-pairwise} 
Let $\mathcal{Z}$ be a search problem over a set of solutions $\mathcal{F}$
and a class of distributions $\mathcal{D}$ over $\R^n$. For $\gamma, \beta >0$, 
let $\new{s}= \mathrm{SD}(\mathcal{Z}, \gamma, \beta)$. For any $\gamma' > 0,$ any
SQ algorithm for $\mathcal{Z}$ requires at least $\new{s} \cdot \gamma' /(\beta - \gamma)$ queries to the 
$\mathrm{STAT}(\sqrt{\gamma + \gamma'})$ or $\mathrm{VSTAT}(1/(3( \gamma + \gamma')))$ oracles.
\end{lemma}

\section{Statistical Query Lower Bounds: From One--Dimension to High--Dimensions} \label{sec:generic-lb}

All our statistical query lower bounds are shown in two steps:
We first construct a one-dimensional density $A$ satisfying certain technical conditions,
and then use $A$ to 
construct a high-dimensional distribution which is Gaussian in all but one directions.
The second step is the same for all the problems that we consider. 
To formally define it, we require the following construction:
 
\begin{definition} [High-Dimensional Hidden Direction Distribution] \label{def:pv-hidden} 
{For a distribution $A$ on the real line with probability density function $A(x)$ and}
 a unit vector $v \in \R^n$, consider the distribution over $\R^n$ with probability density function 
$$\p_v(x) = A(v \cdot x) \exp\left(-\|x - (v \cdot x) v\|_2^2/2\right)/(2\pi)^{(n-1)/2}.$$ 
That is, $\p_v$ is the product distribution whose orthogonal projection onto the direction of $v$ is $A$, 
and onto the subspace perpendicular to $v$ is the standard $(n-1)$-dimensional normal distribution. 
\end{definition}

Suppose that we have constructed a one-dimensional distribution $A$ satisfying the following condition:

\begin{cond} \label{cond:moments}
{Let $m \in \Z_+$.}
The distribution $A$ on $\R$ is such that (i) the first $m$ moments of $A$ 
agree with the first $m$ moments of $N(0,1)$, and (ii) $\chi^2(A,N(0,1))$ is finite. 
\end{cond}

Note that Condition~\ref{cond:moments}-(ii) above 
implies that the distribution $A$ has a probability density function (pdf), which we will denote by $A(x)$.
We will henceforth blur the distinction between a distribution and its pdf.
The main result of this section is the following:

\begin{proposition} \label{prop:generic-sq} 
Given a distribution $A$ on $\R$ that satisfies Condition~\ref{cond:moments} {for some $m \in \Z_+$} 
and a constant $0< c < 1/2$, consider the set of distributions 
$\{\p_v\}_{v \in \mathbb{S}_n}$, for $n \geq m^{\Omega(1/c)}$. For a given $\eps >0$,
suppose that $\dtv(\p_v,\p_{v'}) > 2 \eps$ whenever $|v \cdot v'|$ is at most $1/8$. 
Then, any SQ algorithm which, given access to $\p_v(\bx)$ for an unknown $v \in  \mathbb{S}_n$, 
outputs a hypothesis $\q$ with $\dtv(\q, \p_v) \leq \eps$ needs at least 
$2^{\Omega(n^{c/2})} \geq n^{m+1}$ queries to $\mathrm{STAT}\large(O(n)^{-(m+1)(1/4-c/2)} \sqrt{\chi^2(A,N(0,1))}\large)$ 
or to $\mathrm{VSTAT}\left(O(n)^{(m+1)(1/2-c)} /\chi^2(A,N(0,1))\right)$. 
\end{proposition}

\noindent 
{An intuitive interpretation of the proposition is as follows: 
If we do not want our SQ algorithm to use a number of queries exponential in $n^{\Omega(1)}$, 
then we would need $\Omega(n)^{\Omega(m+1)}$ samples to simulate a single statistical query.}

The rest of this section is devoted to the proof of Proposition~\ref{prop:generic-sq}.
In Section~\ref{ssec:cor}, we prove a correlation bound which is the main technical ingredient for the proof.
In Section~\ref{ssec:generic}, we show a simple packing for unit vectors over the sphere and put the pieces
together to complete the proof.

\subsection{Main Correlation Lemma} \label{ssec:cor}

The main technical result of this section is the following:
\begin{lemma}[Correlation Lemma] \label{lem:cor}  
Let $m \in \Z_+$.
If the distribution $A$ over $\R$ agrees with the first $m$ moments of $N(0,1)$, then for all $v,v' \in \R^n$, we have that
\begin{equation} \label{eqn:corr-pv}
|\chi_{N(0,I)}(\p_v, \p_{v'})| \leq |v \cdot v'|^{m+1} \chi^2(A, N(0,1)) \;.
\end{equation}
\end{lemma}

Note that we may assume that $\chi^2(A,N(0,1))$ is finite, otherwise the lemma statement is trivial. 
Hence, we can henceforth assume that Condition~\ref{cond:moments} is satisfied. 
In particular the distributions $A$, $\p_v$ and $\p_{v'}$ all have probability density functions.

To prove Lemma~\ref{lem:cor} we proceed as follows:
We start by bounding the $\chi^2$-divergence between the one-dimensional projection 
of $\p_v$ onto $v'$ and $N(0,1)$. 
As well as being a critical component towards the proof of Lemma \ref{lem:cor}, 
this fact can be used to show that random projections of $\p_v$ 
are close to $N(0,1)$ with high probability.  Specifically, we show:
\begin{lemma} \label{lem:1d-proj} 
Let $\q$ be  the distribution of $v' \cdot X$, for $X \sim \p_v$. 
Then, we have that  
$$\chi^2(\q, N(0,1)) \leq (v \cdot v')^{2(m+1)} \chi^2\left(A,N(0,1)\right) \;.$$ 
\end{lemma}
\begin{proof}
Let $\theta$ be the angle between $v$ and $v'$. 
Let $x, y$ be orthogonal coordinates for the plane spanned 
by $v$ and $v'$, with the $x$-axis in the $v'$ direction. 
Note that $\p_v$ is a product of a distribution on this plane 
and a standard Gaussian perpendicular to it. 
On this plane, $\p_v$ is a product of $A$ and $N(0,1)$.
Thus, we have that
\begin{align*}
\q(x) & = \int_{\bx:v' \cdot \bx = x} \p_v(\bx) d\bx \\
&= \int_{y \in \R} A(x \cos \theta + y \sin \theta) G(x \sin \theta - y \cos \theta) dy \;.
\end{align*}
Let $U_\theta$ be the linear operator that maps $f:\R \rightarrow \R$ to
$$\int_{y \in \R} f(x \cos \theta + y \sin \theta) G(x \sin \theta - y \cos \theta) dy \;,$$
so that $\q = U_\theta (A)$. We will show that we can expand $A$ 
as a linear combination of eigenfunctions of $U_\theta$.

Let $He_i(x)$ denote the $i$-th (probabilists') Hermite polynomial.
We note that the functions $He_i(x)G(x)/\sqrt{i!}$, for $i \in \N$, 
are orthonormal with respect to the inner product 
$\langle f, g \rangle = \int_{x \in \R} f(x)g(x)/G(x) dx$. 
Indeed, by using the fact that the Hermite functions, which can be written as $He(x)\sqrt{G(x)}$, 
are a complete orthonormal family for $L_2(\R)$, we get that any function 
$f:\R \rightarrow \R$ such that $\int_{\R} f(x)g(x)/G(x) dx < \infty$
is almost everywhere equal to a linear combination of these $He_i(x)G(x)/i!$. 
Since $\int_{-\infty}^\infty A(x)^2/G(x) dx = 1 + \chi^2(A,N(0,1))$ is finite, we can write
\begin{equation} \label{eq:orthonormal-expansion}
A(x) = \sum_{i=0}^\infty a_i He_i(x) G(x)/\sqrt{i!} \;.
\end{equation}
Using the orthogonality of $He_i(x)$ we can extract these coefficients, 
since
\begin{equation} \label{eq:coefficients-orthonormal-expansion}
\E_{X \sim A}\left[ He_i(X)/\sqrt{i!} \right]  = \int_{\R} a_i He_i(x)^2 G(x)/i! dx = a_i \;.
\end{equation}
Since $A$ agrees with the first $m$ moments of the standard Gaussian, for $0 \leq i \leq m $, 
we have that $$\E_{X \sim A}[He_i(X)/\sqrt{i!}] = \E_{X \sim N(0, 1)}[He_i(X)/\sqrt{i!}] = \delta_{i,0} \;.$$ 
This implies that $a_0=1$ and $a_1,\dots, a_m = 0$. Thus, we have
\begin{equation} \label{eq:A-hermite}
A(x) = G(x) + \sum_{i=m+1}^\infty a_i He_i(x) G(x)/\sqrt{i!} \;.
\end{equation}
The orthonormality of these functions with respect to the inner product 
$\langle f, g \rangle =\int_{\R} f(x)g(x)/G(x) dx$ also allows us to 
express the $\chi^2$-divergence in terms of these coefficients:
\begin{align}
\chi^2(A, N(0,1) ) 
& = \int_{-\infty}^\infty (A(x)-G(x))^2/G(x) dx \nonumber \\
& = \int_{-\infty}^\infty \left(\sum_{i=m+1}^\infty a_i He_i(x) G(x)/\sqrt{i!} \right)^2/G(x) dx \nonumber \\
& = \sum_{i=m+1}^{\infty} a_i^2 \;. \label{eqn:chi2-A}
\end{align}
Now we consider the effect of $U_\theta$ on this orthogonal family. 
From the definition of $U_\theta$, we have 
$$U_\theta( He_i G) (x) = \int_{-\infty}^\infty He_i(x \cos \theta + y \sin \theta) G(x \sin \theta - y \cos \theta) dy \;.$$ 
We will use the well-known fact that $He_i(x) G(x)$ is an eigenfunction of $U_\theta$:
\begin{fact} \label{fact:eigenfunction} 
We have that:
$U_\theta(He_i G) (x) = \cos^i (\theta) He_i(x) G(x) \;.$
\end{fact}
\noindent For completeness, we include a proof in Appendix~\ref{app:om}.

We can now use these eigenfunctions and eigenvalues 
with Equation (\ref{eq:A-hermite}) to get an expression for $U_\theta A$:
\begin{equation} \label{eq:UA-hermite}
U_\theta A(x) = G(x) + \sum_{i=m+1}^\infty a_i \cos^i \theta He_i(x) G(x)/\sqrt{i!} \;,
\end{equation}
which can be used to express its $\chi^2$-divergence:
\begin{align}
\chi^2(U_\theta A, N(0,1)) & = \sum_{i=m+1}^\infty a_i^2 \cos^{2i} \theta \nonumber \\
& \leq \cos^{2(m+1)} \theta \sum_{i=m+1}^\infty a_i^2  \nonumber \\
& = \cos^{2(m+1)} \theta \cdot \chi^2(A, N(0,1))  \label{eqn:chi2-NoiseA} \;,
\end{align}
where the last line uses (\ref{eqn:chi2-A}).
Recalling that $\q= U_\theta A$ and $\cos \theta = v \cdot v'$, 
this completes the proof. 
\end{proof}

\begin{proof}[Proof of Lemma \ref{lem:cor}]
We first show that the correlation between 
the high-dimensional densities $\p_v$ and $\p_{v'}$ 
needed for Lemma~\ref{lem:cor} 
can be reduced to a one-dimensional correlation.
Just as in the proof of Lemma \ref{lem:1d-proj}, 
let $\theta = \arccos (v \cdot v')$ and let $x, y$ be coordinates for the plane spanned 
by $v$ and $v'$ with the $x$-axis in the $v'$ direction. 
Each of $\p_v$ and $\p_{v'}$ is a product of a distribution on this plane 
and a standard Gaussian perpendicular to it. On this plane, they are both products of $A$ and $N(0,1)$ 
with different rotations applied. Thus, we have that
\begin{align*}
\chi_{N(0,I)}(\p_v,\p_{v'}) + 1 & = \int_{\R^n} \p_v(\bx) \p_{v'}(\bx)/G(\bx) d\bx \\
& = \int_{-\infty}^\infty \int_{-\infty}^\infty  A(x) G(y) A(x \cos \theta + y \sin \theta) G(x \sin \theta - y \cos \theta)  /G(x)G(y) dx dy \\
 & = \int_{-\infty}^\infty A(x)/G(x) \cdot \int_{-\infty}^\infty A(x \cos \theta + y \sin \theta) G(x \sin \theta - y \cos \theta) dy dx \\
 & = 1 + \chi_{N(0,1)}(A, U_\theta A) \;.
\end{align*}
Now we can bound from above 
this correlation in terms of the $\chi^2$-divergences of both distributions from $N(0,1)$, 
one of which we can bound using Lemma \ref{lem:1d-proj}:
 \begin{align*}
 |\chi_{N(0,1)}(A, U_\theta A)| & \leq \int_{\R} |A(x)-G(x)||U_\theta A(x) - G(x)|/G(x) dx \\
 & \leq \sqrt{\int_{\R} (A(x)-G(x))^2/G(x) dx} \cdot \sqrt{\int_{\R} (U_\theta A(x) - G(x))^2/G(x) dx} \\
 & = \sqrt{\chi^2(A,N(0,1)) \cdot \chi^2(U_\theta A,N(0,1))} \\
 & \leq |\cos^{m+1} (\theta)| \cdot \chi^2(A,N(0,1)) \;,
 \end{align*}
 where the first line follows by triangle inequality, the second inequality is Cauchy-Schwarz, 
 and the last line uses Lemma \ref{lem:1d-proj}.
 The proof of Lemma~\ref{lem:cor} is now complete.
 \end{proof}

\subsection{Proof of Proposition~\ref{prop:generic-sq}} \label{ssec:generic}
We note that our distribution learning problem can be expressed as a search problem in the sense of \cite{Feldman13}.
Consider the following search problem $\mathcal{Z}$: given access to (an oracle or samples from) $\p_v$, for an 
unknown unit vector $v$, find a distribution $f$ such that $\dtv(\p_v,f) \leq \eps$.
Thus, for us, the set of solutions $\mathcal{F}$ is the set of all distributions on $\R^n$, 
and $\mathcal{D} \subseteq \mathcal{F}$ is the set of $\p_v$ for all unit vectors $v$. 
For  a unit vector $v$, $\mathcal{Z}(\p_v)$ is the set of all distributions $f$ 
such that $\dtv(\p_v,f) \leq \eps$. 
For a distribution $f$ on $\R^d$, $\mathcal{Z}^{-1}(f)$ 
is the set of $\p_v$ such that $\dtv(f,\p_v) \leq \eps$.

To prove a lower bound on the statistical dimension, we will take 
$D={N}(0, I)$, $\mathcal{D}=\{\p_v: v \in \s^n\}$ 
and construct a suitable finite set $\mathcal{D}_D$.

\begin{lemma} \label{lem:set-of-nearly-orthogonal} 
For any $0 < c < 1/2$, there is a set $S$ of at least $2^{\Omega(n^{c})}$ unit vectors in $\R^n$ 
such that for {each pair of distinct} $v, v' \in S$, it holds $|v \cdot v'| \leq O(n^{c-1/2})$.
\end{lemma}

\noindent The lemma can be shown by a probabilistic argument.
Specifically, if we take $|S| = 2^{\Omega(n^{c})}$ unit vectors uniformly at random from the unit sphere $\s_n$, 
then the desired event happens with positive probability. 
The proof is given in Appendix~\ref{app:om}.

\begin{proof}[Proof of Proposition \ref{prop:generic-sq}]
For a given constant $0<c<1/2$, 
fix a set $S$ of $2^{\Omega(n^c)}$ unit vectors in $\R^n$ satisfying the statement of Lemma~\ref{lem:set-of-nearly-orthogonal}.
Let $\mathcal{D}_D = \{ \p_v \mid v \in S \}$. 
Then, by Lemma \ref{lem:cor}, we have that, for $v, v' \in S$ with 
$v \neq v'$,  it holds 
$$\chi_{N(0,I)}(\p_v,\p_{v'}) \leq |v \cdot v'|^{m+1}  \chi^2(A,N(0,1)) = \Omega(n)^{-(m+1)(1/2-c)} \chi^2(A,N(0,1)) \;.$$
If $v=v'$, then $\chi_{N(0,I)}(\p_v,\p_v)=\chi^2(\p_v,N(0,I)) = \chi^2(A,N(0,1))$. 
We thus have that, for 
$$\gamma \eqdef \Omega(n)^{-(m+1)(1/2-c)} \chi^2(A,N(0,1)) \textrm{ and } \beta \eqdef \chi^2(A,N(0,1)) \;,$$ 
$\mathcal{D}_D$ is  $(\gamma, \beta)$ 
correlated with respect to $D=N(0,I)$.

Since $\dtv(\p_v,\p_{v'}) > 2 \eps$ for distinct $v$ and $v'$ in $S$, 
for any distribution $f$ over $\R^n$, we have that 
$\mathcal{Z}^{-1}(f) = \left\{ \p_v : v \in S \text{ and } \dtv(\p_v,f) \leq \eps \right\}$ has 
$|\mathcal{Z}^{-1}(f)| \leq 1$ using the triangle inequality for $\dtv$. 
We thus have that $\mathcal{D}_D\setminus \mathcal{Z}^{-1}(f)$ is 
$(\gamma, \beta)$ 
correlated with respect to $D=N(0,I)$, 
and $|\mathcal{D}_D\setminus \mathcal{Z}^{-1}(f)| \geq |S|-1 = 2^{\Omega(n^c)}$. 
That is, 
$$SD\left( \mathcal{Z}, \gamma, \beta \right) \geq 2^{\Omega(n^c)}\;.$$
An application of Lemma~\ref{lem:sq-from-pairwise} for $\gamma'  \eqdef \gamma = \Omega(n)^{-(m+1)(1/2-c)} \chi^2(A,N(0,1))$, 
we obtain that any SQ algorithm requires at least $2^{\Omega(n^c)} n^{-(m+1)(1/2-c)}$ calls to the 
$$\mathrm{STAT}\left(O(n)^{-(m+1)(1/4-c/2)} \sqrt{\chi^2(A,N(0,1))}\right) \textrm{ or } \mathrm{VSTAT}\left( O(n)^{(m+1)(1/2-c)}/\chi^2(A,N(0,1)) \right)$$ 
oracle to solve $\mathcal{Z}.$ 
From our assumption that $n \geq m^{\Omega(1/c)}$, it follows that $n \geq \Omega((m+1) \log n)^{2/c}$, 
and therefore $2^{\Omega(n^{c/2})} \geq n^{m+1}$. Hence, the total number of required queries is at least 
$2^{\Omega(n^{c/2})} \geq n^{m+1}$.
This completes the proof.
\end{proof}

\section{SQ Lower Bound for Learning Gaussian Mixtures} \label{sec:gmm-sq-lb}

The main result of this section is the following:

\begin{theorem} \label{thm:mixtures-sq} 
Fix $0 < \eps <1$.
Any SQ algorithm that given SQ access to a $k$-mixture $\p$ of $n$-dimensional 
Gaussians, $N(\mu_i,\Sigma_i)$, $i \in [k]$,
for $n \geq \Omega(k^8 \log(1/\eps)^3)$, 
which are promised to satisfy $\dtv(\p_i, \p_j) \geq 1-\eps$, for all $i \neq j$, 
and moreover are such that
$$\max \left\{ \max_{i,j} \|\mu_i-\mu_j\|_2, \max_i \|\Sigma_i\|_2^{1/2} \right\}  
\leq \poly(k, \log(1/\eps)) \left( \min_i 1/\|\Sigma_i^{-1}\|_2^{1/2} \right) \;,$$
and outputs a distribution $\q$ with $\dtv(\p,\q) \leq 1/2$, 
needs at least \new{$2^{\Omega(n^{1/8})} \geq n^{2k}$} 
calls to $\mathrm{STAT}\left(O(n)^{-k/6}\right)$ or to $\mathrm{VSTAT}\left(O(n)^{k/3}\right)$.
\end{theorem}

\noindent {\bf Remark.}
We remark that the \new{well-conditioned assumption} in Theorem~\ref{thm:mixtures-sq}
(i.e., that the distances between the means and the largest and smallest eigenvalues of any covariance matrix are bounded) 
guarantees that an SQ algorithm with a bounded number of SQ queries is possible.

\medskip

The proof of Theorem~\ref{thm:mixtures-sq} 
follows by an application of the framework developed in 
Section~\ref{sec:generic-lb} and the following proposition:

\begin{proposition} \label{prop:mixtures} 
For any $\eps > 0$, there exists a distribution $A$ on $\R$
that is a mixture of $k$ Gaussians $A_i$, $i \in [k]$, 
and satisfies the following conditions: 
\begin{itemize}
\item[(i)] $A$ agrees with $N(0,1)$ on the first $2k-1$ moments. 
\item[(ii)] \new{Each Gaussian component $A_i$ has}
variance {$\Theta\left(\frac{1}{k^2 \log^2(k+1/\eps)}\right)$} and mean of magnitude $O(\sqrt{k})$.
\item[(iii)] It holds $\dtv(A_i,A_j) \geq 1-\eps$, for all $i \neq j$.
\item[(iv)] We have $\chi^2(A, N(0,1)) \leq \exp(O(k)) \log(1/\eps)$.
\item[(v)] In the high-dimensional construction \new{of Definition~\ref{def:pv-hidden}}, 
we have that $\dtv(\p_v, \p_{v'}) \geq 1/2$ whenever $|v \cdot v'| \leq 1/2$.
\end{itemize}
\end{proposition}

\noindent Given Proposition~\ref{prop:mixtures}, 
the proof of Theorem~\ref{thm:mixtures-sq}  follows easily.

\begin{proof}[Proof of Theorem~\ref{thm:mixtures-sq}]
First note that the distribution $A$ given by Proposition~\ref{prop:mixtures}, 
satisfies Condition~\ref{cond:moments} for $m = 2k-1$.
By Proposition \ref{prop:generic-sq}, \new{applied for $c=1/4$,} any algorithm that is 
given SQ access to $\p_v$, for an unknown \new{unit vector $v \in S$}, 
and outputs a distribution $\q$ with $\dtv(\p,\q) \leq \eps$, needs at least 
 \new{$2^{\Omega(n^{1/8})} \geq n^{2k}$} calls to 
 $$\mathrm{STAT}\left(O(n)^{-\new{k/4}} \cdot \exp(O(k))\sqrt{\log(1/\eps)} \right) \textrm{ or } 
 \mathrm{VSTAT}\left(O(n)^{\new{k/2}} {/(\exp(O(k))\log(1/\eps))}\right) \;.$$ 
 For \new{$n=\Omega( k^8 \log(1/\eps)^3)$}, 
 we have \new{$n^{k/2} \geq \left ( \exp(O(k))\log(1/\eps) \right)^3$}, 
 and so we need  precision $O(n)^{-k/6}$ for $\mathrm{STAT}$ or $O(n)^{k/3}$ for $\mathrm{VSTAT}$.
 
It remains to show that $\p_v$ is a mixture of $k$ Gaussians that satisfies the necessary conditions. 
Note that $\p_v$, when expressed in an appropriate basis, is a product of the mixture of $k$ 
univariate Gaussians and the standard $(n-1)$-dimensional normal distribution. 
Recall that the product of two Gaussians is a Gaussian. 
\new{If $A = \sum_{i=1}^k w_i N(\mu'_i, \delta)$, where $\delta = \Theta \left(\frac{1}{k^2 \log(k+1/\eps)}\right)$ 
(by Proposition \ref{prop:mixtures} (ii))}, 
then we have that \new{$\p_v =  \sum_{i=1}^k w_i N\left(v\mu'_i, I - (1-\delta)v v^T\right)$.}
We can bound the variation distance between two components by: 
$$\dtv\left( N(v\mu'_i, I - (1-\delta)v v^T), N(v\mu'_j, I - (1-\delta)v v^T) \right) =
 \dtv\left( N(\mu'_i, \delta), N(\mu'_j, \delta)\right) \geq 1-\eps \;,$$ 
by Proposition \ref{prop:mixtures} (iii). 
Also, we have that 
$$ \frac{\max \left\{ \max_{i,j} \|\mu_i-\mu_j\|_2, \max_i \|\Sigma_i\|_2^{1/2} \right\}}{ \left( \min_i 1/\|\Sigma_i^{-1}\|_2^{1/2} \right)}  
= \frac{\max \left\{ \max_{i,j} \|\mu_i-\mu_j\|_2, 1\right\}}{\delta} 
\leq O(\sqrt{k}/\delta) \leq \poly(k\log(1/\eps)) \;.$$
This completes the proof.
 \end{proof}

\subsection{Proof of Proposition~\ref{prop:mixtures}}

We start with the following lemma:

\begin{lemma} \label{lem:Gauss-Hermite} 
There is a discrete distribution $B$ on the real line, supported on $k$ points, 
that agrees with $N(0,1)$ on the first $2k-1$ moments. 
All points $x$ in the support of $B$ have $|x|=O(\sqrt{k})$. 
\end{lemma}
\begin{proof}
This lemma essentially follows from standard techniques for Gaussian quadrature~\cite{AbramowitzStegun:72}.
Given a (possibly infinite) interval $[a, b]$, a weighting function $\omega(x)$, 
and an integer $k >0$, we can find $x_i$ and $w_i$ for $1 \leq i \leq k$ such that
$$\int_a^b \omega(x) p(x) dx = \sum_{i=1}^k w_i p(x_i) \;,$$
for all polynomials $p(x)$ of degree at most $2k-1$. 
The Gauss-Hermite quadrature is a standard implementation of this general scheme
on the interval $(-\infty,\infty)$ with $\omega(x)=e^{-x^2}$. 
Here, we take the $x_i$'s to be the roots of the $k$-th (physicist's) Hermite polynomial $H_k(x)$. 
Then, we have that  $w_i= \frac{2^{k-1} k! \sqrt{\pi}}{k^2H_{k-1}(x_i)^2}$.

We would like to take $\omega(x) = G(x) := \frac{1}{\sqrt{2 \pi}} e^{-x^2/2}$, the pdf of $N(0,1)$. 
To do this, we need to rescale the above $w_i$ and $x_i$, 
and use the probabilist's Hermite polynomials 
$He_k(x) \eqdef 2^{-k/2} H_k(x/\sqrt{2})$. 
We claim that we can take the \new{$x_i'$}'s to be the roots 
of $He_k(x)$, \new{i.e., $x'_i = \sqrt{2} x_i$} 
and \new{$w'_i = \frac{k!}{k^2 He_{k-1}({x'_i})^2}$}. Indeed, we have
\begin{align*}
\sum_{i=1}^k w'_i p(x_i')  
& = \sum_{i=1}^k \frac{k!}{k^2 He_{k-1}(\new{\sqrt{2}}x_i)^2} p(\new{\sqrt{2}}x_i) \\
& = \sum_{i=1}^k \frac{2^{k-1} k! \sqrt{\pi}}{k^2 H_{k-1}(x_i)^2} \cdot \frac{1}{\sqrt{\pi}} \cdot  p(\new{\sqrt{2}}x_i) \\
& = \frac{1}{\sqrt{\pi}} \int_{-\infty}^\infty p(\sqrt{2}y) e^{-y^2} dy \\
& = \frac{1}{\sqrt{\pi}} \int_{-\infty}^\infty p(x) e^{-x^2/2} (1/\sqrt{2}) dx \\
& = \int_{-\infty}^\infty p(x) G(x) dx \;,
\end{align*}
for all polynomials $p(x)$ of degree at most $2k-1$. 

Note that all the weights are nonnegative by definition. 
Also note that $\sum_{i=1}^k w'_i = \int_{-\infty}^\infty 1 \cdot G(x) = 1$. 
We take \new{$B$} to be the probability distribution with probability \new{$w'_i$} of being \new{$x'_i$}, 
for each $1 \leq i \leq k$. 
Then we have 
$$\E_{X \sim B}[X^j]=\sum_{i=1}^k w'_i {x'}_i^j = \int_{-\infty}^\infty x^j G(x) dx =  \E_{X \sim N(0,1)}[X^j] \;,$$ 
for all integers $1 \leq j \leq \new{2k-1}$.

It is known (see, e.g.,~\cite{Szego:39}) that all roots of $H_k(x)$ 
have absolute value $O(\sqrt{k})$, 
and so all roots of $He_k$. 
Hence, all points $x$ in the support of  $B$ have $|x|=O(\sqrt{k})$.
This completes the proof.
\end{proof}

On the other hand, if we want $\chi^2(A,N(0,1))$ to be finite, 
we need to have a mixture of Gaussians each with positive variance $\delta > 0$.

\begin{corollary} \label{cor:mixture} 
For any $0 < \delta < 1$, there is a distribution $A$ on $\R$ that is a mixture of $k$ Gaussians 
each with variance $\delta$ that agrees with $N(0,1)$ on the first $2k-1$ moments. 
The means of all the Gaussian components have magnitude $O(\sqrt{k})$. 
\end{corollary}
\begin{proof}
By rescaling the distribution $B$ given by Lemma \ref{lem:Gauss-Hermite}, 
we can find a discrete distribution $B'$ supported on $k$ points 
with absolute value no bigger than $O(\sqrt{k})$ 
that agrees with the first $2k-1$ moments of $N(0,1-\delta)$. 
The rescaled distribution $B'$ assigns probability mass \new{$w'_i$} to the points \new{$\sqrt{1-\delta} x'_i$}, for $1 \leq i \leq k$.
Let $X \sim B'$, $X' \sim N(0,1-\delta)$, and $Y \sim N(0,\delta)$ that is independent of $X, X'$.
We take $A$ to be the distribution of $X+Y$. 
Then, we have
$$\E[(X+Y)^j] = \sum_{i=0}^j {i \choose j} \E[X^{i}]\E[Y^{j-i}] =  \sum_{i=0}^j {i \choose j} \E[X'^{i}]\E[Y^{j-i}] = \E[(X'+Y)^j] \;,$$
for all integers $1 \leq j \leq \new{2k-1}$. 
By standard facts about Gaussians, 
$X'+Y$ is distributed as $N(0,1)$. 
Finally, note that the distribution of $X+Y$ is a mixture of $k$ Gaussians 
$N(\sqrt{1-\delta}x'_i, \delta)$ with weights $w'_i$.
\end{proof}

To appropriately set the parameter $\delta$, 
we need to consider the high-dimensional construction (Definition~\ref{def:pv-hidden}):

\begin{lemma} \label{lem:mixtures-pv-sep} 
For $v, v' \in \s_n$ with $|v \cdot v'| \leq 1/2$, we have that 
$\dtv(\p_v, \p_{v'}) \geq 1 - O\left(k\sqrt{\delta} \log(1/\delta) \csc \theta\right)$.
\end{lemma}
\begin{proof}
We write $A_i$, for $1 \leq i \leq k$, for the Gaussians 
$N(\mu_i, \delta)$ that $A$ is a mixture of. 
Fix $\eps > 0$. 
By a Chernoff bound, $A_i$ is within the interval $[\mu_i - a, \mu_i + a]$, 
where $a=2\sqrt{\delta \log(1/\eps)}$ 
with probability at least $1-\eps$. 

We again consider the plane spanned by $v$ and $v'$. 
Let $x, y$ be the orthogonal coordinates with $v$ in the direction of the $x$-axis. 
Similarly, let  $x', y'$ be the orthogonal coordinates with $v'$ in the direction of the $x'$-axis. 
Let $\theta$ be the angle between $v$ and $v'$. We have that
\begin{align*} 
\int_\bx \min\{\p_v(\bx),\p_{v'}(\bx)\} d\bx 
& = \int_{x=-\infty}^\infty \int_{y=-\infty}^\infty \min \{A(x)G(y), A(x') G(y') \} dx dy \\
& = \int_{x=-\infty}^\infty \int_{x'=-\infty}^\infty \min \{A(x)G(y), A(x') G(y') \} \csc \theta dx dx' \\
& \leq  k \max_{i,j}  \int_{x=-\infty}^\infty \int_{x'=-\infty}^\infty \min \{A_i(x)G(y), A_j(x') G(y') \} \csc \theta dx dx' \\
& \leq k \eps +  k \max_{i,j}  \int_{x=\mu_i - a}^{\mu_i + a} \int_{x'=\mu_j-a}^{\mu_j+a} \min \{A_i(x)G(y), A_j(x') G(y') \} \csc \theta dx dx' \\
& \leq k \eps +  k \max_{i,j} a^2 \csc \theta \max_{x,x' \in \R} \min \{A_i(x)G(y), A_j(x') G(y') \} \\
& \leq k \eps + k a^2 \csc \theta/(2\pi \sqrt{\delta}) \\
& = k \eps + k \sqrt{\delta} \csc \theta \log(1/\eps)/\pi \;.
\end{align*}
Taking $\eps = \sqrt{\delta}$, 
we obtain that $\int_\bx \min\{\p_v(\bx),\p_{v'}(\bx)\} d\bx \leq  O(k \sqrt{\delta} \log(1/\delta) \csc \theta)$.
On the other hand,
\begin{align*}
\dtv(\p_v, \p_{v'}) 
& = \frac{1}{2} \int_\bx |\p_v(\bx)- \p_{v'}(\bx)| d\bx \\
& = \frac{1}{2} \int_\bx \left( \max\{\p_v(\bx),\p_{v'}(\bx)\} - \min\{\p_v(\bx),\p_{v'}(\bx)\} \right) d\bx \\
& =  \frac{1}{2} \int_\bx \left( \p_v(\bx) + \p_{v'}(\bx) - 2\min\{\p_v(\bx),\p_{v'}(\bx)\} \right) d\bx \\
& = 1 - \int_\bx \min\{\p_v(\bx),\p_{v'}(\bx)\} d\bx \\ 
& \geq 1 - O(k\sqrt{\delta} \log(1/\delta) \csc \theta) \;.
\end{align*}
This completes the proof.
\end{proof}
This gives an upper bound on $\delta$. 
We don't want $\delta$ to be too small, because of the following lemma:
\begin{lemma} \label{lem:mixtures-chi-squared} 
We have that 
$\chi^2(A,N(0,1)) \leq \exp(O(k))/\sqrt{\delta}$.
\end{lemma}
\begin{proof}
Each component $A_i$, for $1 \leq i \leq k$, satisfies the following:
\begin{align*}
& 1 + \chi^2(A_i,N(0,1)) =  \int_x A_i(x)^2/G(x) dx\\
&=   \frac{1}{\sqrt{2 \pi}\delta} \int_x \exp\left(-(x-\mu_i)^2/\delta + x^2/2\right) dx \\
& = \frac{1}{\sqrt{2 \pi}\delta}  \int_x \exp\left(-x^2(1/\delta - 1/2) +2\mu_i x/\delta - \mu_i^2/\delta\right) dx  \\
& =  \frac{1}{\sqrt{2 \pi}\delta}  \int_x \exp\left(-(x-2\mu_i/(2 - \delta))^2((2-\delta)/2\delta) +2\mu_i^2/(\delta(2-\delta))  - \mu_i^2/\delta\right) dx  \\
& = \frac{1}{\sqrt{2 \pi}\delta}  \int_x \exp\left(-(x-2\mu_i/(2 - \delta))^2((2-\delta)/2\delta) +2\mu_i^2/(\delta(2-\delta))  - (2-\delta)\mu_i^2/\delta(2-\delta)\right) dx  \\
& =  \frac{\sqrt{2}\exp(\mu_i^2/(2-\delta))}{\sqrt{(2-\delta)\delta}}  \int_x \left(1/\sqrt{2 \pi (2\delta/(2-\delta))}\right) \exp\left(-(x-2\mu_i/(2 - \delta))^2(1/\delta - 1/2)\right) dx \\
& =  \frac{\sqrt{2}\exp\left(\mu_i^2/(2-\delta)\right)}{\sqrt{(2-\delta)\delta}} \\
& \leq \exp(O(k)) /\sqrt{2\delta}  \; .
\end{align*}
Thus, for the mixture $A$ we have that:
\begin{align*}
 1 +\chi^2(A,N(0,1)) &  = \sum_i \sum_j w_i w_j/(1-\delta) \int_x A_i(x)A_j(x)/G(x) dx \\
& \leq \sum_i \sum_j w_i w_j/(1-\delta) \sqrt{(1+\chi^2\left(A_i, N(0,1) \right)(1+\chi^2\left(A_j, N(0,1)\right)} \\
& \leq \sum_i \sum_j w_i w_j/(1-\delta) \cdot \exp(O(k)) /\sqrt{2\delta} \\
& = \exp(O(k)) /\sqrt{2\delta} \cdot \sum_i \sum_j w_i w_j \\
& = \exp(O(k)) /\sqrt{2\delta} \cdot 1 \;.
\end{align*}
This completes the proof.
\end{proof}

The following simple lemma helps us 
enforce the condition that the Gaussian components are well-separated:

\begin{lemma} \label{lem:mixture-separation} 
Given $\eps > 0$, if $\delta \leq O\left(\frac{1}{\new{k} \log(1/\eps)}\right)$, 
then $\dtv(A_i,A_j) \geq 1 - \eps.$ 
\end{lemma}
\begin{proof}
It is known (see, e.g.,~\cite{Szego:39}) 
that the difference between two roots of 
$H_k(x)$ is $\Omega(1/\sqrt{k})$. 
Thus, the same is true of $He_k(x)$ and by our construction, 
we have that $|\mu_i - \mu_j| \geq \Omega((1-\delta)/\sqrt{k})$, for $i \neq j$. 
By standard Chernoff bounds, with probability at least $1-\eps/2$, 
$A_i$ lies in the range $(\mu_i - \sqrt{2\delta \ln(2/\eps)}, \mu_i + \sqrt{2\delta \ln(2/\eps)})$. 
Similarly, with probability at least $1-\eps/2$, 
$A_j$ lies in the range $(\mu_j - \sqrt{2\delta \ln(2/\eps)}, \mu_j + \sqrt{2\delta \ln(2/\eps)})$. 
If these intervals are disjoint, we have $\dtv(A_i, A_j) \geq 1-\eps$. 
This holds when $(1-\delta)/\sqrt{k}) = \Omega(\sqrt{\delta \log(1/\eps)})$, 
which is true when $\delta \leq O(1/\new{k} \log(1/\eps))$.
\end{proof}

We now have all the necessary tools to prove Proposition \ref{prop:mixtures}.
We take $\delta=C/(k^2 \log^2(k+1/\eps))$ for a sufficiently small constant $C>0$. 
Combined with Corollary~\ref{cor:mixture}, this gives condition (ii). 
For condition (i), note that, by Corollary \ref{cor:mixture}, $A$ agrees 
with $N(0,1)$ on the first $2k-1$ moments. 
Since $\delta$ was selected to be smaller than $O(1/\new{k} \log(1/\eps))$,
Lemma \ref{lem:mixture-separation} gives condition (iii). 
Lemma \ref{lem:mixtures-chi-squared} gives condition (iv).
Finally, by our choice of $\delta$ and  Lemma \ref{lem:mixtures-pv-sep}, we get condition (v).
This completes the proof. \qed

\section{SQ Lower Bounds for Robust Learning of a Gaussian} \label{sec:robust-sq-lb}

In this section, we prove our super-polynomial SQ lower bounds for robustly learning a high-dimensional Gaussian.
In Section~\ref{ssec:robust-mean}, we show our lower bound for robustly learning 
an unknown mean spherical Gaussian. In Section~\ref{ssec:robust-covariance}, 
we give our lower bound for robustly learning 
a zero mean unknown covariance Gaussian with respect
to the spectral norm.

\subsection{Robust Learning Lower Bound for Unknown Mean Gaussian}  \label{ssec:robust-mean}
In this subsection, we use the framework of Section~\ref{sec:generic-lb}
to prove the following theorem:

\begin{theorem} \label{thm:robust-lb-sq}
Let $0< \eps <1$ and $n \geq \log(1/\eps)^{\Omega(1)}$.
\new{Fix any $M \in \Z_+$ such that $M = O(\log^{1/2}(1/\eps))$, where the universal constant 
in the $O(\cdot)$ is assumed to be sufficiently small.} 
Any algorithm that, given SQ access to a distribution $\p$ on $\R^n$
which satisfies $\dtv(\p, N(\mu,I)) \leq \eps$ for an unknown $\mu \in \R^n$ with $\|\mu\|_2 \leq \poly(n/\eps)$, and
returns a hypothesis distribution $\q$ with $\dtv(\q, \p) \leq O(\eps \log(1/\eps)^{1/2} \new{/M^2})$,
requires at least  $2^{\Omega(n^{1/12})} \geq n^{\new{M}}$ calls to
$\mathrm{STAT}\left(O(n)^{-\new{M}/6}\right)$
or to $\mathrm{VSTAT}\left(O(n)^{\new{M}/3}\right)$.
\end{theorem}

The theorem will follow from the following proposition:

\begin{proposition} \label{prop:A-for-agnostic-learning-sq}
For any $\delta>0$ and $m \in \Z_+$, there is a distribution $A$ on $\R$ satisfying
the following conditions:
\begin{itemize}
\item[(i)] $A$ and $N(0,1)$ agree on the  first $m$ moments.
\item[(ii)] $\dtv(A, N(\delta,1)) \leq O(\delta m^2 /\sqrt{\log(1/\delta)})$.
\item[(iii)] $\chi^2(A,N(0,1)) = O(\delta).$
\end{itemize}
\end{proposition}

Before we prove Proposition~\ref{prop:A-for-agnostic-learning-sq}, we 
show how Theorem~\ref{thm:robust-lb-sq} easily follows from it using the 
machinery developed in Section~\ref{sec:generic-lb}.

\begin{proof}[Proof of Theorem~\ref{thm:robust-lb-sq}]
We can assume without loss of generality that $\eps>0$ is smaller than a sufficiently small universal constant.
We apply Propositions~\ref{prop:generic-sq} and~\ref{prop:A-for-agnostic-learning-sq}, 
with the parameter $c$ set to $c=1/6$,
$m= \new{M}$,
and $\delta= C \eps \ln(1/\eps)^{1/2} \new{/M^2}$,
where $C>0$ is a sufficiently large constant.
By Proposition~\ref{prop:A-for-agnostic-learning-sq},
we have that (i) $A$ and $N(0,1)$ agree on the first $m$ moments,
(ii) $\dtv(A,N(\delta,1)) \leq O(\delta m^2 /\sqrt{\log(1/\delta)})=O(\eps)$
and (iii) $\chi^2(A,N(0,1))=O(\delta)$.
Note that for any unit vector $v \in \s_n$, it holds that
\begin{equation} \label{eqn:dtv-mean-gaussian}
\dtv(\p_v,N(v \delta, I))=\dtv(A,N(\delta,1)) \leq O(C\eps) \;.
\end{equation}
Therefore, for any unit vectors $v, v'$ with $|v \cdot v' | \leq 1/8$ we have that:
\begin{align*}
\dtv(\p_v,\p_{v'}) & \geq \dtv(N(\delta v, I),N(\delta v',I)) - \dtv(\p_v,N(v \delta, I)) - \dtv(\p_{v'},N(v' \delta, I)) \\
& \geq \Omega(\delta \|v - v'\|_2) - O(C\eps) \\
& =  \Omega(\delta \sqrt{2-2v \cdot v'}) - O(C\eps) \\
& = \Omega(\delta) - O(C\eps) = \Omega(\delta) \;,
\end{align*}
where the first line is the triangle inequality, 
the second line uses (\ref{eqn:dtv-mean-gaussian}) and that 
$\dtv(N(\mu_1, I), N(\mu_2, I))  = \Omega (\|\mu_1 - \mu_2 \|)$ when $\|\mu_1 - \mu_2 \|$ is smaller
than an absolute constant, the third line uses the assumption that $|v \cdot v' | \leq 1/8$, and the last line 
follows from the definition of $\delta$.

We want to apply Proposition \ref{prop:generic-sq} with its ``local parameter'' $\eps$ taken to be $\Omega(\delta)$.
The assumption on $n$ in the statement of Theorem~\ref{thm:robust-lb-sq}, i.e., 
$n \geq \log(1/\eps)^{\Omega(1)}$, implies the condition $n \geq m^{\Omega(1)}$
in the statement of Proposition \ref{prop:generic-sq}.
By our choice of $c=1/6$,
we conclude that any SQ algorithm for our learning problem 
requires at least $2^{\Omega(n^{1/12})} \geq \new{n^{M}}$
queries to $\mathrm{STAT} \left(O(n)^{-M/6} \sqrt{\delta}\right)$ or to
$\mathrm{VSTAT}\left(O(n)^{M/3}/\delta\right)$
to produce a hypothesis distribution
$\q$ with $$\dtv(\q,\p_v) \leq O(\delta) \leq O(\eps \ln(1/\eps)^{1/2} \new{/M^2}) \;,$$
where we used the assumption that $C$ is sufficiently large.
This completes the proof.
\end{proof}

\paragraph{Proof of Proposition~\ref{prop:A-for-agnostic-learning-sq}}
The rest of this section is devoted to the proof of Proposition~\ref{prop:A-for-agnostic-learning-sq}.
We start by describing the outline of the proof. We then provide a number of intermediate useful lemmas
that we subsequently combine to complete the proof.

The proof plan proceeds as follows.
For some $C = \Theta(\sqrt{\log(1/\delta)})$, we define the one-dimensional distribution $A$ to be:
\begin{itemize}
\item For $x \notin [-C, C]$, we define $A(x)=G(x-\delta)$.
\item For $ x \in [-C,C]$, we define $A(x)= G(x-\delta) + p(x)$, 
where $p(x)$ is the degree-$m$ polynomial 
with $\int_{-C}^C p(x) dx = 0$ and 
\begin{equation} \label{eqn:p-moments}
\int_{-C}^C p(x) x^i dx = \int_{-\infty}^\infty (G(x)-G(x-\delta)) x^i dx \;,
\end{equation} 
for $1 \leq i \leq m$. (We note that $p$ is unique after fixing $m$, $C$ and $\delta$.)
\end{itemize}
We need to show that we can find appropriate values for the parameters $m$, $C$, and $\delta$
such that the $L_1$-norm
of $p(x)$ is at most $O(\delta m^2 /\sqrt{\log(1/\delta)})$
and that $A(x)$ is non-negative.
To achieve that, we will express $p(x)$ as a linear combination of (appropriately scaled) {\em Legendre polynomials},
a family of orthogonal polynomials on $[-C,C]$.
Rather than directly showing that the first $m$ moments agree, we will instead want that
the expectations of the first $m$ scaled Legendre polynomials agree.
Bounds on the coefficients of the Legendre polynomials in $p(x)$ allow us
to obtain bounds on the $L_1$ and $L_\infty$ norms of $p(x)$ on $[-C,C]$.
Choosing $m$, $\delta$, and $C$ appropriately will complete the proof of the proposition.

\paragraph{Properties of Legendre Polynomials}
We start by recording the properties of Legendre polynomials that we will need:
\begin{fact} \label{fact:Legendre-props} ~\cite{Szego:39}
The Legendre polynomials, $P_k(x)$, for $k \in \Z_+$, satisfy the following properties:
\begin{itemize}
\item[(i)] $P_k(x)$ is a degree-$k$ polynomial,  $P_0(x)=1$, and $P_1(x)=x$.
\item[(ii)] $\int_{-1}^1 P_i(x) P_j(x) dx = (2/(2i+1)) \delta_{i,j}$ for all $i,j \geq 0.$
\item[(iii)] $|P_k(x)| \leq 1$ for all $|x| \leq 1.$
\item[(iv)] $P_k(x) = (-1)^k P_k(-x).$
\item[(v)] $P_k(x)= (1/2^k) \sum_{i=0}^{\lfloor k/2 \rfloor} {k \choose i} {2k - 2i \choose k} x^{k-2i}.$
\end{itemize}
\end{fact}

\noindent As a simple corollary we obtain the following lemma:
\begin{corollary} \label{lem:legendre-lem}
We have:
\item[(i)]  $|P_k(x)| \leq (4|x|)^k$ for all $|x| \geq 1$.
\item[(ii)] $\int_{-1}^1 |P_k(x)| dx \leq O(1/\sqrt{k}).$
\end{corollary}

\noindent
We are now ready to proceed with the formal proof.
The main technical result of this section is the following lemma:

\begin{lemma} \label{lem:p-legendre-properties}
We can write $p(x) = \sum_{k=0}^m a_k P_k(x/C)$,
where $|a_k|=O(\delta k^{3/2}/C^2)$, for $0 \leq k \leq m$.
\end{lemma}

Before we give the proof of Lemma~\ref{lem:p-legendre-properties}, we deduce 
two corollaries that will be useful in the proof of Proposition~\ref{prop:A-for-agnostic-learning-sq}.
First, we can obtain bounds on the $L_1$ and $L_\infty$ norms of $p(x)$ on $[-C,C]$.
As an immediate corollary of Lemma~\ref{lem:p-legendre-properties} and the aforementioned properties of Legendre polynomials, 
we deduce:

\begin{corollary} \label{lem:L1inf}
We have that:
$\int_{-C}^C |p(x)| dx \leq O(\delta m^2/C)$
and $|p(x)| \leq \delta m^{5/2}/C^2$,
for all $x \in [-C,C].$
\end{corollary}

We now bound from above the desired $\chi^2$-divergence:

\begin{lemma} \label{lem:chi2-agnostic}
$\chi^2(A,N(0,1)) = O\left(\delta^2 + \delta m^{5/2}/C^2 \cdot (C^2 \delta^{\new{2}} + \max_{|x| \leq C} |p(x)|/G(x))\right)$.
\end{lemma}

\begin{proof}
We have the following:
\begin{align*}
\chi^2(A,N(0,1)) & = \int_{-\infty}^\infty A(x)^2/G(x) dx -1\\
& = \int_{-\infty}^\infty G(x-\delta)^2/G(x) dx - 1 + \int_{-C}^C 2 p(x) G(x-\delta)/G(x) dx + \int_{-C}^C p(x)^2/G(x) dx \;.
\end{align*}
For the first term, we note that:
\begin{align*}
\int_{-\infty}^\infty G(x-\delta)^2/G(x) dx  
= \int_{-\infty}^\infty G(x-2\delta) \exp(\delta^2) dx 
= \exp(\delta^2)  \leq 1 + 2\delta^2 \;.
\end{align*}
We bound the second term from above as follows:
\begin{align*}
\left| \int_{-C}^C  2p(x) G(x-\delta)/G(x) dx \right|
&= \left| \int_{-C}^C 2 p(x) \exp(x \delta - \delta^2/2) dx \right|\\
& = \left| \int_{-C}^C  p(x) \cdot (1+x\delta + O(C^2 \delta^2)) dx \right| \\
& \leq C|a_0| + O(C \delta |a_1|) + C^2 \delta^2 \int_{-C}^C |p(x)| dx \\
& \leq 0 + O(\delta^2/C) +  O(\delta^3 m^{5/2}) \;,
\end{align*}
where the last lines uses Lemma~\ref{lem:p-legendre-properties} and Corollary~\ref{lem:L1inf}.
Finally, for the third term we have:
$$\int_{-C}^C p(x)^2/G(x) dx \leq \delta m^{5/2}/C^2 \max_{x \in [-C,C]} |p(x)|/G(x) \;,$$
where the inequality follows from  Corollary~\ref{lem:L1inf}.
This completes the proof of Lemma~\ref{lem:chi2-agnostic}.
\end{proof}

\paragraph{Proof of Lemma~\ref{lem:p-legendre-properties}}
We first note that we can express $p(x)$ as a linear combination
of scaled Legendre polynomials whose coefficients are explicitly given by integrals:
\begin{claim} \label{claim:legendre-expansion}
We can write $p(x) = \sum_{k=0}^m a_k P_k(x/C)$,
where $a_k= ((2k+1)/2C) \int_{-C}^C P_k(x/C) p(x) dx$.
\end{claim}
\begin{proof}
Since $p(x)$ has degree at most $m$
and the set of polynomials $P_k(x/C)$, $0 \leq k \leq m$,
contains a polynomial of each degree from $0$ to $m$,
there exists $a_k \in \R$ such that $p(x) = \sum_{k=0}^m a_k P_k(x/C)$.

It follows from Fact \ref{fact:Legendre-props} (ii) and a change of variables that
$\int_{-C}^C P_i(x/C) P_j(x/C) dx = (2C/(2i+1)) \delta_{i,j}$,
for all $i,j \geq 0.$ We can use this to extract the $a_k$'s.
For $1 \leq k \leq m$, we have
$$\littleint_{-C}^C P_k(x/C) p(x) dx = \littleint_{-C}^C P_k(x/C) \littlesum_{i=0}^m a_i P_i(x/C) dx 
= \littlesum_{i=0}^m  a_i  \littleint_{-C}^C P_k(x/C) P_i(x/C) dx = (2C/(2k+1)) a_k \;.$$
\end{proof}

\noindent Since the first $m$ moments of $p$ are fixed, via \ref{eqn:p-moments}, we obtain:
\begin{equation} \label{eqn:ak-expression}
\int_{-C}^C p(x) P_k(x/C)  dx = \int_{-\infty}^\infty (G(x)-G(x-\delta)) P_k(x/C) dx \;,
\end{equation}
for any $0 \leq k \leq m$.

Since we will apply this with the parameter $1/\delta$  exponential in $m$ and $C$,
we will be able to ignore  $O(\delta^2)$ terms. We use Taylor's theorem to expand 
$(G(x)-G(x-\delta))$ up to second order terms:

\begin{fact} \label{fact:taylor}
$G(x)-G(x-\delta) = x G(x) \delta + (\xi(x)^2 -1)/2 \cdot G(\xi(x)) \delta^2$, for some $x \leq \xi(x) \leq x + \delta.$
\end{fact}

By (\ref{eqn:ak-expression}) and Fact~\ref{fact:taylor}, 
to bound the magnitude of the $a_k$'s, it suffices to bound the terms 
$\int_{-\infty}^\infty P_k(x/C) x G(x) dx$ and $\int_{-\infty}^\infty P_k(x) (\xi(x)^2 -1)/2 \cdot G(\xi(x)) dx$. 
This is done in the following two lemmas.

\begin{lemma} \label{lem:first-order-integral}
For $k \leq 4C$, we have that $\int_{-\infty}^\infty P_k(x/C) x G(x) dx \leq O(\sqrt{k}/C)$.
\end{lemma}
\begin{proof}
When $k$ is even, using Fact \ref{fact:Legendre-props} (iv),
we have that $P_k(x/C) x G(x) = -(P_k(-x/C)\cdot (-x) G(-x))$,
and so the integral is zero. When $k$ is odd, we can rewrite
Fact \ref{fact:Legendre-props} (v) in ascending order of terms,
by using the change of variables $j=(k+1)/2 - i$, as
$$P_k(x)= (1/2^k) \sum_{j=1}^{(k+1)/2} {k \choose (k-1)/2 + j} {k + 2j-1 \choose 2j-1} x^{2j-1}.$$
By standard results about the moments of Gaussians,
for all $j \geq 1$, we have that
$\int_{-\infty}^\infty x^{2j} G(x) dx = (2j+1)!! := \prod_{i=1}^j (2i + 1).$
Thus, we can write
\begin{align*}
\int_{-\infty}^\infty P_k(x/C) x G(x) dx
& = \int_{-\infty}^\infty 1/2^k \sum_{j=1}^{(k+1)/2} {k \choose (k-1)/2 + j} {k + 2j-1 \choose 2j-1} x^{2j} G(x)/C^{2j-1} dx \\
& = (1/2^k) \sum_{j=1}^{(k+1)/2} {k \choose (k-1)/2 + j} {k + 2j-1 \choose 2j-1} (2j+1)!!/C^{2j-1} \;.
\end{align*}
Note that this quantity is non-negative.
We can bound it from above as follows:
\begin{align*}
\int_{-\infty}^\infty P_k(x/C) x G(x) dx
& = (1/2^k) \sum_{j=1}^{(k+1)/2} {k \choose (k-1)/2 + j} {k + 2j-1 \choose 2j-1} (2j+1)!!/C^{2j-1} \\
& \leq \sum_{j=1}^{(k+1)/2} (1/\sqrt{k}) {k + 2j-1 \choose 2j-1} (2j+1)!! /C^{2j-1}\\
& \leq \sum_{j=1}^{(k+1)/2} (1/\sqrt{k}) (k + 2j-1)^{2j-1} (2j+1)!!/C^{2j-1} (2j-1)! \\
& \leq \sum_{j=1}^{(k+1)/2} (1/\sqrt{k}) 2(k + 2j-1)^{2j-1}/C^{2j-1} \\
& \leq (2/\sqrt{k}) \sum_{j=1}^{(k+1)/2} (2k/C)^{2j-1} \leq 8\sqrt{k}/C \;.
\end{align*}
The proof of Lemma~\ref{lem:first-order-integral} is now complete.
\end{proof}

\begin{lemma}
For any integer $1 \leq k \leq 4C$, we have that
$\int_{-\infty}^\infty P_k(x) (\xi(x)^2 -1)/2 \cdot G(\xi(x)) dx \leq O(1) \;,$
for any function $\xi(x)$ with $x \leq \xi(x) \leq x + \delta$, for all $x \in \R$.
\end{lemma}
\begin{proof}
We separate this integral into the interval $[-C,C]$ and the tails.
We can use Fact \ref{fact:Legendre-props} (iii) to bound the integral on $[-C,C]$, as follows:
\begin{align*}
\left|\int_{-C}^C P_k(x) (\xi(x)^2 -1)/2 \cdot G(\xi(x)) dx \right|
&\leq \left|\int_{-C}^C (\xi(x)^2 -1)/2 \cdot G(\xi(x)) dx \right| \\
&\leq \left|\int_{-C}^C (|x|+\delta + 1)^2/2 \cdot G(\min \{0, |x| - \delta \}) dx \right| \\
&\leq O(\delta) + \left|\int_{-C-\delta}^{C+\delta} (|x|+2\delta + 1)^2/2 \cdot G(x) dx \right| \\
&\leq O\left(\delta + \E_{X \sim G}[1] + \E_{X \sim G}[|X|] + \E_{X \sim G}[X^2]\right) \\
&= O(1) \;.
\end{align*}
For the tails, we need Corollary~\ref{lem:legendre-lem}(i).
For the right tail, $[C,\infty)$,  we have
\begin{align*}
\left|\int_{C}^\infty P_k(x) (\xi(x)^2 -1)/2 \cdot G(\xi(x)) dx \right|
&\leq \left|\int_{C}^\infty (4|x|/C)^k (\xi(x)^2 -1)/2 \cdot G(\xi(x)) dx \right| \\
&\leq \left|\int_{C}^\infty (4|x|/C)^k (x+\delta)^2 G(x-\delta) dx \right| \\
&\leq \left|\int_{C}^\infty (4/C)^k |x+2\delta|^{k+2} G(x-\delta) dx \right| \\
&\leq  \left|\int_{C-\delta}^\infty (4/C)^k |x|^{k+2}  \cdot (1+2\delta/C)^k G(x-\delta) dx \right| \\
&\leq 2  \left|\int_{-\infty}^\infty (4/C)^k |x|^{k+2}  G(x-\delta) dx \right| \\
&\leq  O((4/C)^k (k+3)!!) \\
&\leq O((4\sqrt{k}/C)^k) \leq O(1) \;.
\end{align*}
A similar bound holds for the left tail, which completes the proof.
\end{proof}

\noindent Putting everything together, gives Lemma~\ref{lem:p-legendre-properties}. \qed

\medskip

To prove Proposition~\ref{prop:A-for-agnostic-learning-sq}, we need to set $C$ appropriately
and check the bounds on $m$ needed for $A(x)$ to satisfy the necessary properties.
\begin{proof}[Proof of Proposition \ref{prop:A-for-agnostic-learning-sq}]

Note that unless $\delta$ is sufficiently small
and $m^2 \leq O(\sqrt{\log(1/\delta)})$,
taking $A=N(0,1)$, instead of using our construction,
satisfies the proposition.
We will take $C=\Theta(\sqrt{\log (1/\delta)})$,
and so we can assume that $m \leq \sqrt{C}$.

Recall that $A(x)$ is defined to be $G(x-\delta) + p(x)$ on $[-C,C]$
and $G(x-\delta)$ outside of $[-C,C]$.
Firstly, $A(x)$ needs to be the pdf of a distribution.
Since $$\int_{-C}^C P_k(x/C) p(x) dx = \int_{-\infty}^\infty (G(x)-G(x-\delta)) P_k(x/C) dx$$
for $k=0$, when $P_k(x)=1$,
we have that $\int_{-\infty}^\infty A(x) dx = 1$.
We also need that $A(x)$ is non-negative,
i.e., that $A(x)=G(x-\delta) + p(x) \geq 0$ for all $x \in [-C, C]$.
Note that
$$G(x) + p(x) \geq G(C+\delta) - \delta m^{5/2}/C^2 \;,$$
using Corollary \ref{lem:L1inf}.
Since $m^2  \leq C$, we need $G(C+\delta) \geq \delta C^{3/4}$.
This holds when $C=\sqrt{\ln(1/\delta)} -\delta$,
since then we have $G(C+\delta) = \sqrt{\delta/2 \pi} \geq \delta \sqrt{\ln(1/\delta)}^{3/4}$
for sufficiently small $\delta$.  Note that this also implies that
$A(x) \leq 2 G(x - \delta)$ for all $x$,
and $|p(x)| \leq G(x)$  for all $-C \leq x \leq C$.

The second of these and Lemma \ref{lem:chi2-agnostic} imply (iii).
For (i), by construction, we have that the first $m$ moments agree.
%

$A$ satisfies (ii), since by Lemma \ref{lem:L1inf} ,
$$\dtv(A,N(\delta,1)) = \frac{1}{2} \int_{-C}^C |p(x)| dx \leq O(\delta m^2/C) = O(\delta m^2/\sqrt{\log(1/\delta)}).$$
The proof of Proposition~\ref{prop:A-for-agnostic-learning-sq} is now compete.
\end{proof}

\subsection{Robust Learning Lower Bound for Unknown Covariance Gaussian}  \label{ssec:robust-covariance}

In this subsection, we prove an SQ lower bound for robustly learning the covariance matrix of a high-dimensional 
Gaussian with known mean. We note that our lower bound applies even for spectral norm approximation.
In particular, we show:

\begin{theorem} \label{thm:robust-cov-lb-sq}
Let $\eps > 0$ and $n \geq \Omega(\log^{2}(1/\eps))$.
\new{Fix any $M \in \Z_+$ such that $M = O(\log^{1/4}(1/\eps))$, where the universal constant in the $O(\cdot)$
is assumed to be sufficiently small.}
Any algorithm that, given SQ access to a distribution $\p$ on $\R^n$
which has $\dtv(\p, N(0,\Sigma)) \leq \eps$ for some positive-definite $\Sigma \in \R^{n\times n}$ with $I/2 \preceq \Sigma  \preceq 2I$, 
and returns a matrix $\widetilde \Sigma$ with $\|\widetilde \Sigma - \Sigma\|_2 \leq O(\eps \log(1/\eps) \new{/M^4})$,
requires at least $2^{n^{2/15}} \geq n^{\new{M}}$ calls to 
$\mathrm{STAT}\left(O(n)^{-\new{M}/6}\right)$
or to $\mathrm{VSTAT}\left(O(n)^{\new{M}/3}\right)$.
\end{theorem}

The theorem will follow from the following proposition:

\begin{proposition} \label{prop:A-for-agnostic-learning-cov-sq}
For any $1/3>\delta>0$ and integer $\log(1/\delta)^{1/4} \gg m >0$, there is a distribution $A$ on $\R$ satisfying
the following conditions:
\begin{itemize}
\item[(i)] $A$ and $N(0,1)$ agree on the  first $m$ moments.
\item[(ii)] $\dtv(A, N(0,(1-\delta)^2) \leq O(\delta m^4/\log(1/\delta))$.
\item[(iii)] $\chi^2(A,N(0,1)) = O(1+m^8 \delta^{3/2} / \new{\log(1/\delta)^{5/2}})$.
\end{itemize}
\end{proposition}

As in the previous subsection, 
Theorem \ref{thm:robust-cov-lb-sq} follows easily from 
Proposition \ref{prop:generic-sq} and Proposition \ref{prop:A-for-agnostic-learning-cov-sq}.

\begin{proof}[Proof of Theorem~\ref{thm:robust-cov-lb-sq}]
We apply Propositions~\ref{prop:generic-sq} and~\ref{prop:A-for-agnostic-learning-cov-sq}
with $m=\new{M}$ and $\delta= 2 \eps \ln(1/\eps)\new{/M^4}$.
By Proposition~\ref{prop:A-for-agnostic-learning-cov-sq},
we have that (i) $A$ and $N(0,1)$ agree on the first \new{$M$} moments,
(ii) $\dtv(A, N(0,(1-\delta)^2)) \leq O(\delta m^4 /\log(1/\delta))=O(\eps)$
and  (iii) $\chi^2(A,N(0,1))=O(1)$.

{
Note that we cannot directly apply Proposition \ref{prop:generic-sq}, 
since we are not aiming to learn within small total variation distance. 
Instead, we are interested in a different search problem, 
that of finding an approximation $\widetilde \Sigma$ to the covariance $\Sigma$ with 
$\|\widetilde \Sigma - \Sigma\|_2 \leq \eps \log(1/\eps) \new{/ M^4}$, 
where $\Sigma$ is the covariance of a mean $0$ Gaussian 
within $\eps$ total variation distance. Note that for $\p_v$, 
we have that $\Sigma = I - (1-(1-\delta)^2) v v^T$.
We need to argue that this search problem has at most one solution in $S$, 
i.e., that for any $\widetilde \Sigma$, the set 
$\mathcal{Z}^{-1}(\widetilde  \Sigma)= \{ \p_v : v \in S \text{ and } 
\|\widetilde \Sigma - I - (1-(1-\delta)^2) v v^T\|_2 \leq \eps \log(1/\eps)\new{/ M^4} \}$
has $|\mathcal{Z}^{-1}(\widetilde \Sigma)| \leq 1$, 
where $S$ is as in Lemma~\ref{lem:set-of-nearly-orthogonal}.

\begin{lemma} \label{lem:covariance-unique-S}
For $S$ as in Lemma \ref{lem:set-of-nearly-orthogonal} with $c=1/6$ 
and with $n$ larger than a sufficiently large constant, 
$|\{\p_v : v \in S \textrm{ and } \|\widetilde \Sigma - I - (1-(1-\delta)^2) v v^T\|_2 \leq \eps \log(1/\eps)\new{/ M^4}\}| \leq 1$, 
for all $ \widetilde \Sigma$. 
\end{lemma}
\begin{proof}
Suppose for a contradiction that this set has size at least $2$ 
for some $\widetilde \Sigma$ and let $v, v'$ be distinct elements. 
Let $\Sigma_v= I - (1-(1-\delta)^2) v v^T$ and define $\Sigma_{v'}$ similarly.
Then we have  $\|\widetilde \Sigma - \Sigma_v\|_2 \leq \eps \log(1/\eps)\new{/ M^4}$ 
and $\|\widetilde \Sigma - \Sigma_{v'}\|_2 \leq \eps \log(1/\eps)\new{/M^4}$. 
By the triangle inequality, we have
$\| \Sigma_v -  \Sigma_{v'}\|_2 \leq \new{2} \eps \log(1/\eps)\new{/M^4}$.
However, we also have that $|v \cdot v'| \leq O(n^{-1/3}) \leq 1/2$. 
Now we get that $v^T \Sigma v = (1-\delta)^2$, but 
$v^T \Sigma_{v'} v = (1 - |v \cdot v'|^2) \cdot 1 +  |v \cdot v'|^2 \cdot (1-\delta)^2 \geq 3/4 + (1/4)(1-\delta)^2$. 
We thus obtain
$$\| \Sigma_v -  \Sigma_{v'}\|_2 \geq v^T (\Sigma_{v'} -  \Sigma_{v}) v \geq (3/4) (1-(1-\delta)^2) \geq (3/2)\delta - O(\delta^2) > \delta \;,$$
where the last inequality assumes that $\eps$ is at most an appropriately small universal constant. 
Since $\delta= 2 \eps \ln(1/\eps)\new{/ M^4}$, this leads to a contradiction.
\end{proof}

Now we use the the proof of Proposition \ref{prop:generic-sq}  with $c=1/6$ and $\eps$ taken to be $\Omega(\delta)$.
Our condition on $n$, $n \geq \ln(1/\eps)^2$,
implies the condition $n \geq \Omega(m^8)$.
We conclude that it requires at least $2^{n^{2/15}} \geq n^{\new{M}}$
calls to $\mathrm{STAT} \left(O(n)^{-\new{M}/6} \right)$ or to
$\mathrm{VSTAT}\left(O(n)^{\new{M}/3} \right)$
to produce a a $\widetilde \Sigma$ such that
$\|\widetilde \Sigma - \Sigma\|_2 \leq \eps \log(1/\eps)\new{/M^4}$. 
}
\end{proof}

\begin{proof}[Proof of  Proposition \ref{prop:A-for-agnostic-learning-cov-sq}]
Similarly to the previous subsection, we choose to define the univariate distribution $A$ 
to have probability density function given by
$$
A(x) = G(x/(1-\delta))/(1-\delta) - p(x)\mathbb{1}_{[-C,C]} \;,
$$
where $C$ is a sufficiently small multiple of $\sqrt{\log(1/\delta)}$ and $p(x)$ 
is the unique degree-$m$ polynomial that causes $A$ and $G(x)$ (the pdf of $N(0, 1)$) to have the same first $m$ moments. 
Once again, we may write $p(x)=\sum_{k=0}^m a_k P_k(x/C)$, 
where $a_k = ((2k+1)/2C)\int_{-\infty}^\infty (G(x)-G(x/(1-\delta))/(1-\delta))P_k(x/C) dx$. 
The bulk of our proof will now be in bounding the $a_k$'s.

The first thing to note is that since $G(x)-G(x/(1-\delta))$ is even, 
$a_k$ is $0$ for $k$ odd. For $k$ even, we will need to compute this expression 
using Fact \ref{fact:Legendre-props} (v). In particular, we have that
\begin{align*}
a_k & = \frac{2k+1}{2C}\int_{-\infty}^\infty (G(x)-G(x/(1-\delta))/(1-\delta))P_k(x/C) dx \\
& = \frac{2k+1}{2C}\int_{-\infty}^\infty(G(x)-G(x/(1-\delta))/(1-\delta)) 2^{-k}\sum_{i=0}^{\lfloor k/2\rfloor} \binom{k}{i} \binom{2k-2i}{k} (x/C)^{k-2i} dx\\
& = \frac{2k+1}{2C}\int_{-\infty}^\infty(G(x)-G(x/(1-\delta))/(1-\delta)) 2^{-k}\sum_{j=0}^{\lfloor k/2\rfloor} \binom{k}{k/2+j} \binom{k+2j}{k} (x/C)^{2j} dx\\
& = \frac{2k+1}{2^{k+1}C} \sum_{j=0}^{\lfloor k/2\rfloor} \binom{k}{k/2+j} \binom{k+2j}{k} C^{-2j} \int_{-\infty}^\infty(G(x)-G(x/(1-\delta))/(1-\delta)) x^{2j}dx\\
& = \frac{2k+1}{2^{k+1}C} \sum_{j=0}^{\lfloor k/2\rfloor} \binom{k}{k/2+j} \binom{k+2j}{k} C^{-2j} (2j-1)!! (1-(1-\delta)^2j)\\ 
& \leq \frac{\delta(2k+1)}{C} \sum_{j=0}^{\lfloor k/2\rfloor} \binom{k+2j}{2j} C^{-2j} (2j-1)!! j\\
& \leq \frac{\delta(2k+1)}{C} \sum_{j=1}^{\infty} \left(\frac{2k}{C} \right)^{2j}\\
& \leq \delta 10k^3 C^{-3} \;,
\end{align*}
where in the last step we assume that $k$ is less than a sufficiently small multiple of $C$.

It is now clear that $A$ is a pseudo-distribution that matches its first $m$ moments with $N(0,1)$. 
Firstly, in order to check that $A$ is a distribution, it is clear that $A(x)>0$ for $|x|>C$. 
For $|x|\leq C$ we have that $|A(x)-G(x/(1-\delta))/(1-\delta)| \leq \sum_{k=0}^m |a_k| \leq \delta 10m^4 C^{-3}$. 
Since this is smaller than $\delta^{1/2}<G(x/(1-\delta))/(1-\delta)$, we have that $A(x)\geq 0$ everywhere.

Next, we need to bound from above $\dtv(A,N(0,1-\delta))$, i.e., 
the $L_1$-distance between $A(x)$ and $G(x/(1-\delta))/(1-\delta)$. 
This in turn is at most
$$
\sum_{k=0}^m \int_{-C}^C |a_k P_k(x/C)| dx \leq \sum_{k=0}^m \delta 20k^3 C^{-2} = O(\delta m^4/\log(1/\delta)).
$$

Finally, we need to bound from above $\chi^2(A,N(0,1))$. 
Note that 
$$\chi^2(A,N(0,1)) \leq O(\chi^2(N(0,1-\delta),N(0,1))+\chi^2(A-N(0,1-\delta),N(0,1))) \;.$$ 
It is easy to see that $\chi^2(N(0,1-\delta),N(0,1)) = O(1+\delta)$. On the other hand, we have that
\begin{align*}
\chi^2(A-N(0,1-\delta),N(0,1))) & \leq (m+1) \sum_{k=0}^m a_k^2 \int_{-C}^C P_k(x/C)^2 / G(x) dx\\
& = O(m^7 \delta^2 C^{-6}) \sum_{k=0}^m \int_{-C}^C G(x)^{-2} dx\\
& = O(m^8 \delta^{3/2} C^{-5}) \;.  
\end{align*}
This completes the proof.
\end{proof}

\section{Statistical and Computational Tradeoffs} \label{sec:tradeoffs}
In this section, we prove our SQ lower bounds establishing statistical-computational
tradeoffs for two natural robust estimation problems. In Section~\ref{ssec:robust-cov-tradeoff}, 
we give a sharp-tradeoff for the problem of robustly estimating the covariance matrix
in spectral norm. In Section~\ref{ssec:robust-sparse-mean}, we show such a tradeoff
for robust sparse mean estimation.

\subsection{Robust Estimation of Covariance Matrix in Spectral Norm} \label{ssec:robust-cov-tradeoff}
\new{
In this subsection, we establish an SQ lower bound for robust covariance estimation in spectral norm.
Our SQ lower bound provides evidence for the existence of a statistical-computational tradeoff for this problem.
Roughly speaking, we show that, for any constant $c>0$, given samples from a corrupted $n$-dimensional 
Gaussian $N(0, \Sigma)$, any computationally efficient SQ algorithm that approximates $\Sigma$ 
within a factor of $2$ requires $\Omega(n^{2-c})$ samples. 
Our lower bound applies even to the weaker Huber contamination model.

We note that the information-theoretic optimum for this problem is known 
to be $\Theta(n)$ samples (and is achievable by an exponential time SQ algorithm).
Hence, our lower bound establishes a nearly-quadratic gap in the sample complexity 
between efficient and inefficient SQ algorithms for this problem.}
Formally, we show:

\begin{theorem} \label{thm:robust-covariance-tradeoff}
Let $0 < c < 1/6$, and $n  > 125$. 
Any algorithm that, given SQ access to a distribution $\p$ on $\R^n$ 
of the form $\p = (1-\eps) N(0,\Sigma) + \eps N_1$, where $\eps \leq c/\ln(n)$ 
and $N_1$ is a noise distribution, 
for some covariance $\Sigma$ with  $\|\Sigma \|_2 \leq \poly(n/\eps)$, 
and returns a matrix $\tilde \Sigma$ with $\tilde \Sigma/2 \preceq \Sigma \preceq 2 \tilde \Sigma$, 
requires at least $2^{\Omega(n^{c/3})}$ calls to 
$\mathrm{STAT}\left(O(n)^{-(1-5c/2)}\right)$
or to $\mathrm{VSTAT}\left(O(n)^{2-5c}\right)$. 
Furthermore, the result holds even when the noise distribution $N_1$ 
is a mixture of $2$ Gaussians.
\end{theorem}

\begin{proof} 
Let $\eps = c/\ln(n)$.
We consider the following mixture of $3$ Gaussians: 
$$A=(1-\eps)N\left(0,(1/5-\eps)/(1-\eps)\right) + (\eps/2) \cdot N(\sqrt{4/(5\eps)},1) + \eps/2 \cdot N(-\sqrt{4/(5\eps)},1) \;.$$ 
Note that $A$ is symmetric about $0$ and so, for $X \sim A$, 
we have $\E_{X \sim A}[X]=\E_{X \sim A}[X^3]=0$. 
The variance of $A$ is 
$\var_{X \sim A}[X] = \E_{X \sim A}[X^2]= (1/5-\eps) + \eps \cdot 4/(5\eps) + \eps = 1$. 
That is, $A$ agrees with $N(0,1)$ on the first 3 moments. We need a bound on $\chi^2(A,N(0,1))$. 
For this, we use the following three easy facts (see Appendix~\ref{app:om} for the simple proofs):

\begin{fact} \label{clm:chi-squared-mixtures}
For distributions $B, C, D$ and $w \in [0, 1]$, we have that 
$\chi^2\left(wB+(1-w)C, D\right)= w^2 \chi^2(B, D) + (1-w)^2 \chi^2(C, D) + 2w(1-w) \chi_D(B,C)$. 
\end{fact}

\begin{fact} \label{clm:correlation-different-mean}
For $\mu, \mu' \in \R$, we have that
$\chi_{N(0,1)}(N(\mu',1), N(\mu,1))=\exp(\mu \mu')-1$.
\end{fact}

\begin{fact} \label{clm:correlation-different-variance}
We have that 
$\chi^2(N(0,\sigma^2),N(0,1)) = \sqrt{2/\sigma^4 - 1/\sigma^2} - 1$. 
\end{fact}

Note that $1/6  \leq (1/5-\eps)/(1-\eps) \leq 1/5$.
Fact~\ref{clm:correlation-different-variance} now yields 
$$\chi^2(N(0,(1/5-\eps)/(1-\eps)),N(0,1)) \leq \sqrt{2/(1/6)^2 + 1/(1/6)}-1 = \sqrt{78} - 1 \leq 8 \;.$$
Using Fact~\ref{clm:chi-squared-mixtures}, we can write:
\begin{align*}
\chi^2(A, N(0,1)) 
&\leq (1-\eps)^2 \chi^2(N(0,(1/5-\eps)/(1-\eps)), N(0,1)) + (\eps^2/2) \chi^2(N(\sqrt{4/5\eps},1), N(0,1)) \\
&+ (\eps^2/2) \chi_{N(0,1)}(N(\sqrt{4/5\eps},1), N(-\sqrt{4/5\eps},1)) + (1-\eps) \eps \chi_{N(0,1)}(N(0,(1/5-\eps)),N(\sqrt{4/5\eps},1)) \\
& \leq 8 + (\eps^2/2) (\exp(4/(5\eps))-1) + 0 \\
& + (1-\eps) \eps \sqrt{\chi^2(N(\sqrt{4/5\eps},1), N(0,1)) \chi^2(N(0,(1/5-\eps)/(1-\eps)),N(0,1))} \\
& \leq 8 + \eps^2 \exp(4/(5\eps)) + \eps \sqrt{8} \exp(2/(5\eps)) \leq O(1+\exp(1/\eps)) \\
& \leq O(n^c) \;.
\end{align*}
Note that we cannot directly apply Proposition \ref{prop:generic-sq}, 
since we are not aiming to learn within small variation distance. 
Instead, we are interested in a different search problem, 
that of approximating the covariance $\Sigma_v$ of the $(1-\eps)$ weight component of $\p_v$ to within a factor of $2$. 
We need to argue that this search problem has at most one solution in $S$, 
i.e., that for any $\Sigma$, the set $\mathcal{Z}^{-1}(\Sigma)= \{ \p_v : v \in S \text{ and } \Sigma \preceq \Sigma_v \preceq 2 \Sigma \}$
has $|\mathcal{Z}^{-1}(\Sigma)| \leq 1$, where $S$ is as in Lemma~\ref{lem:set-of-nearly-orthogonal}.

\begin{lemma} 
For $S$ as in Lemma \ref{lem:set-of-nearly-orthogonal}, 
$|\{\p_v : v \in S \text{ and } \Sigma \preceq \Sigma_v \preceq 2 \Sigma \}| \leq 1$ for all $\Sigma$. 
\end{lemma}
\begin{proof}
Suppose for a contradiction that $|\mathcal{Z}^{-1}(\Sigma)| \geq 2$ for some $\Sigma$. 
Then there are distinct $v,v' \in S$ with  $ \Sigma \preceq \Sigma_v \preceq 2 \Sigma$ 
and $ \Sigma \preceq \Sigma_{v'} \preceq 2 \Sigma$. However, we have that $|v \cdot v'| \leq O(n^{c-1/2}) \leq n^{-1/3}$. 
Now $v^T\Sigma_v v=(1/5-\eps)/(1-\eps) < 1/5$, but 
$$v^T\Sigma_{v'} v = (1-|v \cdot v'|^2) \cdot 1 + |v \cdot v'|^2 \cdot (1/5-\eps)/(1-\eps) \geq 1 - 5|v \cdot v'|^2/6 \geq 1 - n^{-1/3} > 4/5 \;.$$
Thus, we need $v^T \Sigma v \leq 2 v^T \Sigma_v v < 2/5$, 
but $v^T \Sigma v \geq v^T \Sigma_{v'} v/2 > 2/5$. 
This is a contradiction and so $|\mathcal{Z}^{-1}(\Sigma)| < 1$.   
\end{proof}

Now the proof of Proposition \ref{prop:generic-sq} applies
and we obtain that any algorithm that outputs a $\tilde \Sigma$ satisfying the desired conditions uses at least 
$2^{\Omega(n^{c/2})} \geq n^{m+1}$ queries to 
$\mathrm{STAT}(O(n)^{-(m+1)(1/4-c/2)} \sqrt{\chi^2(A,N(0,1))})$ 
or to $\mathrm{VSTAT}(O(n)^{(m+1)(1/2-c)} /\chi^2(A,N(0,1)))$, where $m=3$. 
Now substituting $\chi^2(A,N(0,1) \leq O(n^c)$, we get that we need at least 
$2^{\Omega(n^{c/2})}$ queries to $\mathrm{STAT}(O(n)^{-(1-5c/2)})$ 
or to $\mathrm{VSTAT}(O(n)^{(2-5c)})$. The proof is now complete.
\end{proof}

\subsection{Robust Sparse Mean Estimation} \label{ssec:robust-sparse-mean}

\new{
In this subsection, we establish an SQ lower bound for robust sparse mean estimation.
Our SQ lower bound gives evidence for the existence of a statistical-computational tradeoff for this problem.
Roughly speaking, we show that, for any constant $c>0$, given samples from a corrupted $n$-dimensional 
Gaussian $N(\mu, I)$, where the mean vector $\mu$ is $k$-sparse,
any computationally efficient SQ algorithm that approximates the true mean requires 
$\Omega(k^{2-c})$ samples. 
Our lower bound applies even to the weaker Huber contamination model.

We note that the information-theoretic optimum for this problem is known 
to be $\Theta(k \log n)$ (and is achievable by an exponential time SQ algorithm).
Hence, our lower bound establishes a nearly-quadratic gap in the sample complexity 
between efficient and inefficient SQ algorithms.}
Formally, we show:

\begin{theorem} \label{thm:sq-lb-sparse-mean} Fix any constant $0< c <1$.
Let $k , n \in \Z_+$ be sufficiently large positive integers satisfying
\new{$n \geq 8 k^2$}.
Any algorithm which, given SQ access to a distribution $\p$ on $\R^n$ such that
$\p = (1-\delta) N(\mu,I) + \delta N_1$, where $N_1$ is an arbitrary distribution, 
\new{$\delta =  \eps / k^{c/4}$}, and $\mu \in \R^n$ is promised to be $k$-sparse 
\new{with  $\|\mu\|_2  =  \eps$}, 
and outputs a hypothesis vector $\wh{\mu}$ satisfying \new{$\|\wh{\mu} - \mu\|_2 \leq \eps/2$}, 
requires at least \new{$\Omega(n^{c k^c / 8})$} 
queries to $\mathrm{STAT}(\new{O(k)^{3c/2-1}})$ or to $\mathrm{VSTAT}(\new{O(k)^{2-3c}})$.
\end{theorem}

To prove our result, we will use the framework of Section~\ref{sec:generic-lb}:
We will construct a suitable one-dimensional distribution $A$ 
and consider an appropriate collection of distributions $\p_v$,
but this time only for $k$-sparse unit vectors $v$ on $\R^n$. 
We start by showing an analogue of Lemma \ref{lem:set-of-nearly-orthogonal} for $k$-sparse vectors, 
and then use it to prove an analogue of Proposition \ref{prop:generic-sq}. 
Our analogue of Lemma \ref{lem:set-of-nearly-orthogonal} is the following:

\begin{lemma} \label{lem:set-of-nearly-orthogonal-sparse} 
\new{Fix a constant $0<c<1$.}
There exists a set $S$ of $k$-sparse unit vectors on $\R^n$ of cardinality $|S|= \new{\lfloor n^{c k^c / 8} \rfloor}$
such that for each pair of distinct vectors $v,v' \in S$ we have that 
$|v \cdot v'| \leq 2 \new{k^{c-1}}$.
 \end{lemma}
\begin{proof}
Let $D$ be the uniform distribution over the set of vectors $v$ on $\R^n$ 
that have exactly $k$ coordinates equal to $1/\sqrt{k}$ and the rest $n-k$ coordinates equal to zero. 
Consider the distribution $D'$ of the inner product $v \cdot v'$, with $v$ and $v'$ independently drawn from $D$. 
Then, we have that $v \cdot v' = i/k$, where $i$ is the number of non-zero coordinates 
that $v'$ and $v'$ have in common. Note that the distribution $D'$ does not change if we fix $v'$
and therefore we have that 
$$\Pr[v \cdot v' = i/k] = {k \choose i} {n-k \choose k-i}/{n \choose k} \;.$$
In other words the random variable $k (v \cdot v')$ is distributed as the 
hypergeometric distribution with parameters $(n, k, k)$. 
By standard tail bounds on the hypergeometric distribution, 
for $t > 0$, we have \new{
$$\Pr\left[k (v \cdot v') \geq k(k/n + t)\right] \leq \exp\left( -k \cdot \mathrm{KL} \left(t+k/n || k/n \right)   \right) \;,$$
where $\mathrm{KL}(a || b) = a \ln(a/b) + (1-a) \ln(\frac{1-a}{1-b})$.
We apply this concentration inequality in a regime where $a \geq 9b$, in which case
 $\mathrm{KL}(a || b) \geq a \ln(a/b) -1 \geq (a/2) \ln(a/b)$.

Fix any constant $0< c < 1$. We apply the above concentration bound for 
$t \eqdef k^{c-1}$. Recalling the assumption $n \geq k^2$, we have that $t+k/n \geq 9k/n$, and therefore
\begin{align*}
\Pr\left[k (v \cdot v') \geq k(k/n + t)\right] 
&\leq \exp\left( -k \cdot \mathrm{KL} \left(t+k/n || k/n \right)   \right) \\
&\leq \exp\left( -k t \ln\left(t n/k\right) \right) \\
& = \left(\frac{tn}{k}\right)^{-kt} = \left(\frac{n}{k^{2-c}}\right)^{-k^c} \\
&\leq n^{-c k^c / 2} \;.
\end{align*}
}
Now if we let $S$ be a set of \new{$\lfloor n^{c k^c / 8} \rfloor$} unit vectors drawn independently from $D$, 
there are \new{${\lfloor n^{c k^c / 4} \rfloor \choose 2} < \lfloor n^{c k^c/2} \rfloor$}
distinct pairs of $v, v' \in S$, 
and by a union bound the probability there exist distinct $v, v' \in S$ 
with  $(v \cdot v') \geq \new{2 k^{c-1}}$ is less than \new{$ \lfloor n^{c k^c/2} \rfloor  n^{-c k^c / 2} < 1$}. 
Thus, there exists a set $S$ such that all distinct pairs $v, v' \in S$ satisfy 
$|v \cdot v'| \leq \new{2 k^{c-1}}$.  This completes the proof.
\end{proof}

Before we proceed with the proof of Theorem~\ref{thm:sq-lb-sparse-mean},
we make a useful observation:
By following the proof of Proposition \ref{prop:generic-sq} using Lemma \ref{lem:set-of-nearly-orthogonal-sparse} 
instead of Lemma \ref{lem:set-of-nearly-orthogonal}, mutatis mutandis, 
we obtain: 
\begin{proposition} \label{prop:generic-sq-sparse} 
Given a distribution $A$ over $\R$ that satisfies Condition~\ref{cond:moments} \new{for some $m \in \Z_+$, and any constant $0< c < 1$}, 
consider the set of distributions 
$\p_v$ for $v \in \mathbb{S}_n$ that are $k$-sparse, with $n \geq \max\{2(m+1) \ln n, \new{8}k^2\}$. 
For a given $\eps >0$, suppose that $\dtv(\p_v,\p_{v'}) > 2 \eps$ whenever $|v \cdot v'|$ is smaller than $1/8$. 
Then, any SQ algorithm which, given access to $\p_v(\bx)$ for an unknown \new{$k$-sparse} $v \in  \mathbb{S}_n$, 
outputs a hypothesis $\q$ with $\dtv(\q, \p_v) \leq \eps$ needs at least 
$\new{\lfloor n^{c k^c / 8} \rfloor} \geq k^{2(m+1)}$ queries to $\mathrm{STAT}(O(k)^{-(m+1)\new{(1/2-c)}} \sqrt{\chi^2(A,N(0,1))})$ 
or to $\mathrm{VSTAT}(O(k)^{(m+1)\new{(1-2c)}} /\chi^2(A,N(0,1)))$. 
\end{proposition}

\new{The above proposition can be used for $m=1$ to establish a similar but quantitatively somewhat weaker 
SQ lower bound.} We can make a \new{crucial} improvement to this proposition for the specific $A$ 
we use in the proof below.

\begin{proof}[Proof of Theorem \ref{thm:sq-lb-sparse-mean}]
We select the one-dimensional distribution $A$ as follows:
$$A=(1-\delta) N(\eps,1)+\delta N\left(-(1-\delta)\eps/\delta,1 \right) \;,$$
where $\delta=\eps \new{k^{-c/4}}$. Note that $A$ has mean $0$, 
\new{i.e., matches $m=1$ moments of $N(0, 1)$.}

We could use Facts \ref{clm:chi-squared-mixtures} and \ref{clm:correlation-different-mean} 
to obtain $\chi^2(A,N(0,1) \leq O(\eps^2 \exp(\eps^2/\delta^2))$. 
However, this would require the parameter $\delta$ to be equal to $\eps/\new{\sqrt{c \ln k}}$ 
to get the required bounds from Proposition~\ref{prop:generic-sq-sparse}. 
The issue here is that $\chi^2(A,N(0,1))$ is much bigger than the variance of $A$, 
which means that the correlation inequality 
$|\chi_{N(0,1)}(\p_v,\p_{v'})| \leq (v \cdot v')^2 \chi^2(A,N(0,1))$ 
is far from tight for most $v$ and $v'$. For our choice of $A$, we prove the following lemma:

\begin{lemma} \label{lem:better-for-robust-mean}
If $A=(1-\delta) N(\eps,1)+\delta N(-(1-\delta)\eps/\delta,1)$, then for $v, v' \in \mathbb{S}_n$, we have
$$1 + |\chi_{N(0,1)}(\p_v,\p_{v'})| \leq \exp\left(\eps^4 (v \cdot v')^2/\delta^4\right) \;.$$
\end{lemma}
\begin{proof}
Let $\theta$ be the angle between $v$ and $v'$.
As in (\ref{eq:orthonormal-expansion}), we will use the expansion 
$A(x) = \sum_{i=0}^\infty a_i He_i(x) G(x)/\sqrt{i!}$. As in (\ref{eq:UA-hermite}), we also have 
the expansion $U_\theta A(x) = \new{G(x) + \sum_{i=2}^\infty a_i \cos^i \theta He_i(x) G(x)/\sqrt{i!}}$. 
Thus, we can write
\begin{align*}
\chi_{N(0,I)}(\p_v,\p_{v'}) & = \chi_{N(0,1)}(A,U_{\theta} A) \\
& = \int_{-\infty}^\infty (A(x)-G(x))(U_{\theta} A - G(x))/G(x) dx \\
& = \int_{-\infty}^\infty \left(\sum_{i=\new{1}}^\infty a_i He_i(x) G(x)/\sqrt{i!} \right) 
\left(\sum_{i=2}^\infty a_i \cos^i \theta He_i(x) G(x)/\sqrt{i!} \right) /G(x) dx \\
& = \sum_{i=\new{2}}^\infty a_i^2 \cos^i \theta\;.
\end{align*}

Since  $A$ is a distribution with mean zero, we have $a_0=1$, $a_1=0$. 
We need to take advantage of the fact that for our selected probability density function $A$, 
the coefficient $a_2^2$ is much smaller than $\chi^2(A,N(0,1))$. 
We can find the $a_i$ explicitly using (\ref{eq:coefficients-orthonormal-expansion}), 
which gives that $a_i = \E_{X \sim A}[ He_i(x)/\sqrt{i!}]$. 
We have the following well-known fact:

\begin{fact} For $\mu \in \R$, we have 
$\E_{X \sim N(\mu,1)}[He_i(x)] = \mu^i$. 
\end{fact}
\begin{proof}
Note that the $i$-th derivative of $G(x)$ 
is $(-1)^i He_i(x)G(x)$. Using Taylor's theorem, we can expand $G(x-\mu)$ around $x$ to obtain
$G(x-\mu) = \sum_{i=0}^\infty \mu^i He_i(x) G(x)/i! \;.$
Taking the expectation of $He_i(x)$ extracts the $i$-th term, establishing the fact.
\end{proof}

\noindent Thus, we have
\begin{align*} 
\sqrt{i!} a_i & = \E_{X \sim A}[He_i(x)] \\
& = (1-\delta) \E_{X \sim N(\eps,1)}[He_i(x)]+\delta \E_{X \sim N(-(1-\delta)\eps/\delta,1)}[He_i(x)] \\
& = (1-\delta) \eps^i + \delta (-(1-\delta)\eps/\delta)^i \;.
\end{align*}
In addition to $a_0=1$,$a_1=0$, we can derive the bound $|a_i| \leq (\eps/\delta)^i/\sqrt{i!}$.
Recalling the special case $a_1=0$ and summing over $i$, 
we have
\begin{align*}
|\chi_{N(0,I)}(\p_v,\p_{v'})| &\leq  \sum_{i=1}^\infty a_i^2 |\cos \theta|^i \\
& \leq \sum_{i=2}^\infty (\eps/\delta)^{2i} |\cos \theta|^i / i! \\
& = \exp\left(\eps^2 |v \cdot v'|/\delta^2\right) - 1 - \eps^2 |v \cdot v'|/\delta^2 \;.
\end{align*}
To complete the proof of the lemma, it is sufficient to show that 
$\exp(x)-x \leq \exp(x^2)$ for all $x \geq 0$. We note that both expressions are $1$ 
and have derivative $0$ at $x=0$. It suffices to show that $d^2(\exp(x)-x )/dx^2 \leq d^2(\exp(x^2)/dx^2$ for $x \geq 0$. 
Note that
$$ \frac{d^2 e^{x^2}}{dx^2} / \frac{d^2(e^x-x)}{dx^2} = (2e^{x^2}+4x^2e^{x^2})e^{-x} \geq 2 e^{x^2-x} = 2e^{(x-1/2)^2-1/4} \geq 2e^{-1/4} > 1 \;.$$
This completes the proof.
\end{proof}

We now have all the necessary ingredients to complete the proof of Theorem~\ref{thm:sq-lb-sparse-mean}.
For distinct $k$-sparse unit vectors $v, v' \in S$, where $S$ is given by Lemma \ref{lem:set-of-nearly-orthogonal-sparse},
we have that
\begin{align*}
|\chi_{N(0,1)}(\p_v,\p_{v'})| & \leq \exp\left(\eps^4 (v \cdot v')^2/\delta^4\right) - 1 \\
& \leq \exp\left(\eps^4  \new{4k^{2c-2}} /  \delta^4\right) - 1 \\
& \leq \exp\left( \new{4k^{2c-2} \cdot k^c} \right) - 1\\
& = e^4  \exp\left(\new{k^{3c-2}} \right) - 1  \\ 
&\leq  e^4  \new{k^{3c-2}}\;.
\end{align*}
Following the proof of Proposition \ref{prop:generic-sq}, we now have that it takes 
\new{$\Omega(n^{c k^c / 8})$} queries to $\mathrm{STAT}(O(k)^{\new{3c/2-1}})$ 
or to $\mathrm{VSTAT}(O(k)^{\new{2-3c}})$ to \new{learn $\p_v$}.
\end{proof}

\section{Sample Complexity Lower Bounds for High--Dimensional Testing} \label{sec:testing}

In this section, we use our framework to prove information-theoretic lower bounds
on the sample complexity of our two high-dimensional testing problems:
(i) robustly testing the mean of a single \new{unknown mean identity covariance} 
Gaussian \new{in Huber's contamination model}, and (ii) (non-robustly) testing between
a single \new{spherical} Gaussian and a mixture of \new{$2$ spherical Gaussians}.

Both these statements follow from the structural results established in the previous sections
using the following proposition:

\begin{proposition} \label{prop:generic-test} 
Let $A$ be a distribution on $\R$ such that $A$ has mean $0$ and $\chi^2(A,N(0,1))$ is finite.
Then, there is no algorithm that, for any $n$, given $N < n/(\new{8}\chi^2(A,N(0,1)))$ samples from a distribution $D$ 
over $\R^n$ 
which is either $N(0,I)$ or $\p_v$, for some unit vector $v \in \R^n$, 
correctly distinguishes between the two cases with probability at least $2/3$. 
\end{proposition}
\begin{proof}
At a high-level, the proof of the proposition uses the structure of the set of $\p_v$'s and 
standard information-theoretic arguments.

Suppose that, after fixing the dimension $n$, the algorithm takes at most $N$ samples. 
We can consider the testing algorithm as a  (possibly randomized) function from $N$-tuples of samples to its output. 
For a distribution $D$, let $D^{\otimes N}$ denote the distribution over independent $N$-tuples drawn from $D$.
We write $f(D^{\otimes N})$ for the Bernoulli distribution 
that gives the output of the algorithm given a single sample of $D^{\otimes N}$. 
Let $\q_N$ be the distribution obtained by choosing $v$ uniformly at random over the unit sphere $\s_n$, 
and then drawing $N$ samples from $\p_v$. Then, $f(\q_N)$ should be ``NO'' with probability at least $2/3$, 
since the probability that each $f(\p_v^{\otimes N})$ is ``NO'' is at least $2/3$. 
On the other hand, $f(N(0,I)^{\otimes N})$ is ``YES'' with probability at least $2/3$. 
By the data processing inequality, it follows that 
$$\dtv(\q_N, N(0,I)^{\otimes N}) \geq \dtv(f(\q_N), f(N(0,I)^{\otimes N})) \geq 1/3 \;.$$
Suppose for the sake of contradiction that $N < n/(\new{8}\chi^2(A,N(0,1)))$. 
Then, we claim that $$\dtv(\q_N, N(0,I)^{\otimes N}) < 1/3 \;.$$
Indeed, we have that:
\begin{align*}
& 4 \dtv(\q_N, N(0,I)^{\otimes N})^2 + 1  
\leq \chi^2(\q_N, N(0,I)^{\otimes N}) +1 = \\
& = \int_{\bx^{(1)} \in \R^n} \dots \int_{\bx^{(N)} \in \R^n} \q_N(\bx^{(1)}, \dots \bx^{(N)})^2 / \prod_{i=1}^N G(\bx^{(i)}) d\bx^{(N)} \dots d\bx^{(1)} \\
& = \int_{\bx^{(1)} \in \R^n} \dots \int_{\bx^{(N)} \in \R^n} \int_{v \in \s_n} \int_{v' \in \s_n} \p_v^N(\bx^{(1)},\dots, \bx^{(N)}) \p_{v'}^N(\bx^{(1)},\dots, \bx^{(N)})/
\prod_{i=1}^N G(\bx^{(i)}) dv' dv d\bx^{(N)} \dots d\bx^{(1)} \\
& = \int_{v \in \s_n} \int_{v' \in \s_n}  \int_{\bx^{(1)} \in \R^n} \dots \int_{\bx^{(N)} \in \R^n} \prod_{i=1}^N \p_v(\bx^{(i)}) \p_{v'}(\bx^{(i)})/G(\bx^{(i)})  d\bx^{(N)} \dots d\bx^{(1)} dv' dv \\
& = \int_{v \in \s_n} \int_{v' \in \s_n} (1+\chi_{N(0,I)}(\p_v,\p_{v'}))^N dv' dv \\
& \leq \int_{v \in \s_n} \int_{v' \in \s_n} \left(1+|v \cdot v'|^2 \chi^2\left(A,N(0,1)\right)\right)^N dv' dv \;,
\end{align*}
where the last line follows from Lemma~\ref{lem:cor}, 
since $A$ satisfies Condition~\ref{cond:moments} for $m=1$. 
We will need the following facts about the Beta function $B(x,y)$:
\begin{fact} \label{fact:beta}
\begin{itemize}
\item[(i)] For $x >-1, y>-1$ we have that: $\int_{0}^{\pi/2} \sin^{x}(\theta) \cos^{y}(\theta) = B((x+1)/2,(y+1)/2)/2$.
\item[(ii)] For all $x, y \in \R$, we have that $B(x,y+1) = B(x,y) \cdot y/(x+y)$.
\end{itemize}
\end{fact}

Consider choosing $v$ and $v'$ independently uniformly at random over $\s_n$.
\new{To obtain our tight sample complexity results, we will need a more precise analysis for the distribution of the angle $\theta$
between $v$ and $v'$. Specifically, we show the following:}

\begin{lemma} \label{lem:angle-beta}
If we choose $v$ and $v'$ uniformly at random from $\s_n$, 
the angle $\theta$ between them is distributed with the probability density function 
$\sin^{n-2}(\theta)/B((n-1)/2,1/2)$, where $B(x,y)$ is the Beta function.
\end{lemma}
\begin{proof}
Since a rotation of the sphere moves both $v$ and $v'$, 
$\theta$ is invariant under such rotations. Thus, we get the same distribution by fixing $v'=e_1$, 
the unit vector in the $x_1$-direction, and choosing $v$ uniformly at random over the sphere. 
Now we have that $\cos \theta = x_1$. 
Let $S_n(r)$ denote the surface area of the sphere of radius $r$ in $(n+1)$ dimensions 
and note that $S_n(r)=r^n S_n(1)$.

For any measurable function $f$,we have that
$$\int_{\s_n} f(x_1) d\Sigma = \int_{\s_2} f(x_1) S_{n-2}(x_2) ds = 
\int_0^{\pi} f(\cos \theta) \sin^{n-2}(\theta) S_{n-2}(1) d\theta  \;. $$
We thus have that the pdf of $\theta$ is proportional to $\sin^{n-2}(\theta)$. 
Taking $f(x_1)\equiv 1$, note that 
$$S_{n-\new{2}}(1)=\int_{\s_n} 1 d\Sigma =  \int_0^{\pi} \sin^{n-2}(\theta) S_{n-2}(1) d\theta = B((n-1)/2,1/2) \;,$$ 
using Fact \ref{fact:beta} (i). 
We thus have that the pdf of $\theta$ is 
$$\sin^{n-2}(\theta)/S_{n-2}(1) = \sin^{n-2}(\theta)/B((n-1)/2,1/2) \;.$$
This completes the proof.
\end{proof}

We now have that
\begin{align*}
4\dtv^2(\q_N, N(0,I)^{\otimes N}) +1 & \leq \int_{v \in \s_n} \int_{v' \in \s_n} \left(1+|v \cdot v'|^{2} \chi^2\left(A,N(0,1)\right)\right)^N dv' dv \\
& = \int_{0}^{\pi} \left(1+|\cos \theta|^{2} \chi^2\left(A,N(0,1)\right)\right)^N \sin^{n-2}(\theta)/B((n-1)/2,1/2) d\theta\\
&= 2 \int_{0}^{\pi/2} \left(1+\cos^{2}(\theta) \chi^2\left(A,N(0,1)\right)\right)^N \sin^{n-2}(\theta)/B((n-1)/2,1/2) d\theta\\
& = \sum_{i=0}^N {N \choose i} 2 \int_{0}^{\pi/2} \chi^2(A,N(0,1))^i \cos^{2i}(\theta) \sin^{n-2}(\theta)/B((n-1)/2,1/2) d\theta\\
& = \sum_{i=0}^N {N \choose i} \chi^2(A,N(0,1))^i B((n-1)/2 , i + 1/2)/B((n-1)/2,1/2) \; .
\end{align*}
We rewrite the above sum as $\sum_{i=0}^N b_i$, where 
$$b_i = {N \choose i} \chi^2(A,N(0,1))^i B((n-1)/2 , i + 1/2)/B((n-1)/2,1/2) \;.$$
Note that $b_0=1$. 
Now consider the ratio $b_{i+1}/b_i$. 
Note that that 
${N \choose i+1}/{N \choose i}=(N-i)/(i+1)$ and using Fact \ref{fact:beta}, we have that 
$B((n-1)/2 , i + 3/2)/B((n-1)/2 , i + 1/2)= (i+1/2)/(i+n/2)$. Therefore, it follows that
$$b_{i+1}/b_i=\chi^2(A,N(0,1)) \cdot (N-i)/(i+n/2) \cdot (i+1/2)/(i+1) \;.$$
For all $i$, we have 
$$b_{i+1}/b_i \leq \chi^2(A,N(0,1)) \cdot N/(n/2) = 2N \chi^2(A,N(0,1))/n \;.$$
When $N \leq n/(\new{8}\chi^2(A,N(0,1)))$, we have $b_{i+1}/b_i \leq 1/4$ for $i \geq 0$, and therefore
$$\sum_{i=0}^N b_i \leq 1/(1-1/4) = 4/3 \;.$$
Hence, we have
$$4\dtv^2(\q_N, N(0,I)^{\otimes N}) +1 \leq 1/3 < 4/9 \;.$$
This implies that $\dtv(\q_N, N(0,I)^{\otimes N}) < 1/3$, 
which is the desired contradiction. 
This completes the proof.
\end{proof}

\begin{remark}
{\em It is worth noting that matching $m > 1$ many moments does not seem to help in the setting of the previous proposition, 
as long as $\chi^2(A,N(0,1)) \leq 1$. This may seem to some extent unsurprising, given that $O(n)$ samples suffice 
for some of the learning problems we consider here. On the other hand, we consider it somewhat 
surprising looking at the proof of Proposition~\ref{prop:generic-test}. Specifically, for general $m$, we would have 
that 
$$\new{4\dtv^2(\q_N, N(0,I)^{\otimes N}) +1 \leq }  \sum_{i=0}^N {N \choose i} \chi^2(A,N(0,1))^i B((n-1)/2 , i(m+1)/2 + 1/2)/B((n-1)/2,1/2) \; .$$
Note that the ratio of one term to the next approximately grows as $N\chi^2(A,N(0,1)) i^{(m-1)/2}/n^{(m+1)/2}$. 
For this to be less than $1/2$, for all $0 \leq i \leq N$, 
we need $N^{(m+1)/2} \chi^2(A,N(0,1)) \leq O(n^{(m+1)/2})$. Thus, 
we need at least $N=\Omega(n/(\chi^2(A,N(0,1)))^{2/m})$ samples . This suggests that 
we should be able to obtain a tighter lower bound if $\chi^2(A,N(0,1)) > 1$ using this technique. 
We omit the details here, as we are mainly interested in the regime $\chi^2(A,N(0,1)) \leq O(1)$ for our applications in this paper.}
\end{remark}

Using Proposition~\ref{prop:generic-test}, we establish the two main results of this section:

\begin{theorem}[Sample Complexity Lower Bound for Robustly Testing Unknown Mean Gaussian] \label{thm:robust-testing-lb}
There is no algorithm that, for every $\eps > 0$ and positive integer $n$, 
given fewer than $\Omega(n)$ samples from a distribution $\p$ on $\R^n$ 
which is promised to satisfy either (a) $\p = N(0 ,I)$ or (b) \new{$\p = (1-\delta) N(\mu,I) + \delta N_1$},
\new{where $\delta = \eps/100$}, $\|\mu\|_2 \geq \eps$ and \new{the noise distribution $N_1$ is a spherical Gaussian}, 
can distinguish between the two cases with probability at least $2/3$.

If instead, for any constant $0< c < 1$, we are promised that  \new{$\p = (1-\delta) N(\mu,I) + \delta N_1$}, 
where  $\delta =  \eps/n^{c/4}$ in case (b), then no algorithm that takes less than $\Omega(n^{1-c})$ 
samples can distinguish between (a) and (b) with probability at least $2/3$.
\end{theorem}

\begin{proof}
Let $\delta$ be the noise rate. We will take $\delta=\eps/100$ or $\delta= \eps/n^{c/4}$.
In both cases, we select our one-dimensional distribution to be the following: 
$$A=(1-\delta) N(\eps,1)+\delta N\left(-(1-\delta)\eps/\delta,1\right) \;.$$
Now Lemma \ref{lem:better-for-robust-mean}
 yields that for $v,v' \in \mathbb{S}_n$, we have
$$1 + |\chi_{N(0,1)}(\p_v,\p_{v'})| \leq \exp\left(\eps^4 (v \cdot v')^2/\delta^4\right) \;.$$
 
We will not apply Proposition~\ref{prop:generic-test} directly but follow its proof
using the aforementioned stronger correlation bound. We have: 
\begin{align*}
4 \dtv(\q_N, N(0,I)^{\otimes N})^2 + 1  & \leq \int_{v \in \s_n} \int_{v' \in \s_n} (1+\chi_{N(0,I)}(\p_v,\p_{v'}))^N dv' dv \\
& \leq \int_{v \in \s_n} \int_{v' \in \s_n} \exp\left(N \eps^4 (v \cdot v')^2/\delta^4\right) dv' dv \\
& = \int_{0}^{\pi}  \exp\left(N \eps^4 (\cos \theta)^2/\delta^4\right) \sin^{n-2}(\theta)/B\left((n-1)/2,1/2\right) d\theta\\
&= \sum_{i=0}^{\infty} (N \eps^4/\delta^4)^i B((n-1)/2,i + 1/2) /i! B((n-1)/2,1/2) \;.
\end{align*}
Now note that the ratio of the $(i+1)$-th term to the $i$-th term 
of the corresponding series is
$$\frac{N \eps^4/\delta^4}{i+1} \cdot \frac{i+1/2}{i+n/2} \leq \frac{2N \eps^4}{n\delta^4} \;.$$ 
When $N \leq n \delta^4/2\eps^4$, the $0$-th term is $1$ and the ratio of the $(i+1)$-th to $i$-th term is less than $1/4$.
Therefore, the above sum is less than $1/(1-1/4) = 4/3$, which implies that 
$$\dtv(\q_N, N(0,I)^{\otimes N}) \leq 1/\sqrt{12} < 1/3 \;.$$
Following the proof of  Proposition~\ref{prop:generic-test}, we conclude that no algorithm satisfying the necessary conditions exists.
To complete the proof, note that for $\delta=\eps/100$, we need at least $\Omega(n)$ samples. 
And for  $\delta=n^{-c/4} \eps$, we need at least $\Omega(n^{1-c})$ samples.
\end{proof}

\begin{theorem}[Sample Complexity Lower Bound for Testing GMMs] \label{thm:test-mixture-lb}
Fix $0<\eps< 1$ and $n \in \Z_+$. 
There is no algorithm that, given less than $\Omega(n/\eps^2)$ samples from a distribution $\p$ on $\R^n$ 
that is promised to be either (a) $\p = N(0,I)$, or  (b) $\p$ is a mixture of two Gaussians each with weight $1/2$ and identity covariance, 
\new{such that $\dtv(\p, N(0,I)) \geq \eps$,}
distinguishes between the two cases with probability at least $2/3$.
\end{theorem}
\begin{proof}
We choose our one-dimensional distribution as 
$$A= (1/2) N(-\delta,1) + (1/2) N(\delta,1) \;,$$
\new{where we set (with hindsight) $\delta = \Theta(\eps^{1/2})$.} 

Note that $A$ has mean $0$.
By Claims \ref{clm:chi-squared-mixtures} and \ref{clm:correlation-different-mean}, we have that
$$\chi^2(A,N(0,1)) = (1/2) \left( \exp(\delta^2)/2 + \exp(-\delta^2)/2 - 1 \right)= (1/2) \left(\cosh(\delta^2)-1 \right) = \Theta(\delta^4) \;.$$
\new{Similarly, it can be shown that $\dtv(A, N(0, 1)) = \Theta(\delta^2)$, which is $\Omega(\eps)$ by our choice of $\delta$.}
Applying Proposition~\ref{prop:generic-test} completes the proof.
\end{proof}

\section{SQ Algorithms for Robustly Learning and Testing a Gaussian} \label{sec:alg}

The structure of this section is as follows:
In Section~\ref{ssec:struct-alg}, we prove a moment--matching structural result
that forms the basis of our algorithms. In Section~\ref{ssec:test-alg}, we present
our robust testing algorithm, and in Section~\ref{ssec:learn-alg} we give our robust learning algorithm.

\subsection{One-Dimensional Moment Matching Lemma} \label{ssec:struct-alg}
The main result of this section is the following structural result:
\begin{lemma} \label{lem:moment-matching}
Let $G \sim N(0, 1)$.
For $\delta > \eps >0$, define $k= 2 \lceil \eps \sqrt{\ln(1/\eps)}/\delta \rceil$.
Let $G'$ be an $\eps$-noisy one-dimensional Gaussian with unit variance
so that for all $t \leq k$ we have that the $t^{th}$ moments of $G$ and $G'$
agree to within an additive $(t-1)!(\delta/\eps)^t \eps/t$. Then, we have that
$\dtv(G,G') = O(\delta)$.
\end{lemma}

\begin{proof}
Let $\widetilde{G}=N(\mu,1)$ be such that $\dtv(G', \widetilde{G}) \leq \eps$.
We can assume without loss of generality that $\mu > 0$.
Looking at just the mean and variance suffices to get $\dtv(G, G') = O(\delta \sqrt{\log(1/\delta))}$
using techniques similar to the proof of the $O(\eps\sqrt{\log(1/\eps)})$ filter algorithm from \cite{DiakonikolasKKLMS16}.
This allows us to focus on the case that $\mu \geq 1$.
Formally, we have the following claim:

\begin{claim} \label{claim:large-mean}
Lemma~\ref{lem:moment-matching} holds when the mean $\mu$
of $\widetilde{G}$ is at least $1$.
\end{claim}
\begin{proof}
We start by noting that we can assume $\eps>0$ is smaller
that a sufficiently small universal constant. Assuming otherwise
and recalling that $\delta> \eps$ gives that
$\dtv(G,G') \leq 1 = O(\delta)$, in which case the lemma statement is trivial.

Let $\mu'$ be the mean of $G'$.
We can write
\begin{equation} \label{eqn:noise-def}
G' = \widetilde{G} +\eps' E - \eps' L \;,
\end{equation}
for distributions $E, L$ with disjoint supports
where $\eps'=\dtv(G',\widetilde{G}) \leq \eps$.
Moreover, it holds
\begin{equation} \label{eqn:noise-cor}
\widetilde{G} \geq \eps' L \;.
\end{equation}

Since $k \geq 2$ by definition, the lemma assumptions imply that
$|\mu'| \leq \delta$ and $\E_{X \sim G'}[X^2] \leq \delta^2/\eps$.
By (\ref{eqn:noise-def}) we have that
$\mu'   = \mu + \eps'    \E_{X \sim E}[X] - \eps'  \E_{X \sim L}[X]$,
and similarly
$\E_{X \sim G'}[(X-\mu')^2]  = 1 + (\mu-\mu')^2 + \eps'\E_{X \sim E}[(X-\mu')^2] - \eps' \E_{X \sim L}[(X-\mu')^2]$.
Therefore, we get
$$\eps'\E_{X \sim E}[(X-\mu')^2] \leq \delta^2/\eps -1 - (\mu'-\mu)^2  + \eps' \E_{X \sim L}[(X-\mu)^2] \;.$$


Note that $\pr_{X \sim L}[|X-\mu'| \geq T] \leq \pr_{X \sim \widetilde{G}}[|X-\mu'| \geq T]/\eps'$.
As in Corollary 8.8 of~\cite{DiakonikolasKKLMS16}, we have that
\begin{align*}
\E_{X \sim L}[(X-\mu')^2] & = \int_0^\infty \pr_{X \sim L}[|X-\mu'| \geq T] 2T dT \\
& \le \int_0^\infty 2T \min \left\{1 ,\pr_{X \sim N(\mu,1)}[|X-\mu'| \geq T] /\eps' \right\} dT \\
& \le \int_0^\infty 2T \min \{1 ,\exp(-((T-|\mu-\mu'|)^2/2) 2T/\eps'  \} dT \\
& = \int_0^{\sqrt{2\ln (1/\eps')} + |\mu-\mu'|} 2T dT + \int_{\sqrt{2\ln (1/\eps')}}^\infty \exp(-T^2/2) 2(T+|\mu-\mu'|)/\dtv(G',\widetilde{G}) dT \\
& = O(\ln (1/\eps') + 1 + |\mu-\mu'|^2) \;.
\end{align*}
By combining the above, we get that
$$\eps'\E_{X \sim E}[(X-\mu')^2] \leq \delta^2/\eps -1 - (1-O(\eps'))|\mu-\mu'|^2 + O(\eps' \ln (1/\eps')) \;,$$
and thus $\E_{X \sim E}[(X-\mu')^2] \leq \delta^2/(\eps\eps')$.

However, the means of $L$ and $E$, $\mu_L$ and $\mu_E$ have  
$|\mu_L-\mu'|^2 \leq \E_{X \sim L}[(X-\mu')^2]$ and 
$|\mu_E-\mu'|^2 \leq \E_{X \sim E}[(X-\mu')^2]$, 
and therefore
$\eps'|\mu_L-\mu'| \leq O(\eps'\sqrt{\ln 1/\eps'} + \eps'|\mu'-\mu|) \leq O(\eps \sqrt{\ln 1/\eps} + \eps|\mu'-\mu|)$ 
and $\eps'|\mu_E-\mu| \leq O(\eps' \sqrt{\delta^2/\eps\eps'}) = O(\delta)$.

We thus have that $\mu'-\mu=\eps'\mu_E - \eps' \mu_L$ satisfies 
$|\mu'- \mu| \leq  O(\delta) + O(\eps \sqrt{\ln 1/\eps}) + O(\eps|\mu-\mu'|)$. 
Since $\eps$ is sufficiently small, we have $|\mu'- \mu| \leq  O(\delta) + 1/2$.
Since $|\mu'| \leq \delta$, We thus have that $\mu  \leq O(\delta)+ 1/2$. 
Since we assumed that $\mu \geq 1$, we get that 
$\delta=\Omega(\mu-1/2)=\Omega(1/2)$ and so $\mu = O(\delta)$. 
We therefore conclude that $\dtv(G,G')=O(\mu)=O(\delta)$.
\end{proof}

We will henceforth assume that $\mu \leq 1$ and thus
we have that $\dtv(G, \widetilde{G}) = \Theta(\mu)$.
Suppose for the sake of contradiction that $\dtv(G, G')  =  \Omega(\delta)$, for a sufficiently large constant in the big-$\Omega$.
Then, $\dtv(G, \widetilde{G}) \geq \dtv(G,\widetilde{G}) - \eps \gg \delta$.
We may assume that $\mu \geq C^2 \delta$ for a sufficiently large constant $C > 0$.
Thus, we have that $\eps \leq \mu/C^2 \leq 1 /C^2$.

The proof will proceed as follows:
Let $f(x)=\sin(x C \eps/\mu)$.
We note that $f(x)$ has a simple expectation under $G$ or $\widetilde G$,
and we can easily get a lower bound on their difference.
We will also use the Taylor series for $f(x)$ and our moment bounds
to derive an upper bound on this difference which contradicts this lower bound.

For $x \in \R_+$, we want an expression for
$\E_{X \sim N(x,1)}[f(X)]$. By standard facts on the Fourier transform, we have that
$$\E_{X \sim N(x,1)}[\exp(-i\omega X)] = \exp(-i\omega x) \E_{X \sim N(0,1)}[\exp(-i\omega X)] = \exp(-\omega^2/2 -i\omega x) \;,$$
for any $\omega \in \R$.
Since $\sin(x C \epsilon/\mu) = (\exp(-i x C \epsilon/\mu) - \exp(i x C \epsilon/\mu))/2i$,
we obtain that
$$
\E_{X \sim N(x,1)}[f(X)] = \exp(-(C\epsilon/\mu)^2/2) \sin(xC\epsilon/\mu) \;.
$$
Therefore,
$\E_{X \sim N(0,1)}[f(X)]=0$
and $\E_{X \sim N(\mu, 1)}[f(X)] \geq \exp(1/C) \sin(C \eps) > (C/2)\epsilon$,
and thus
\begin{equation} \label{eqn:contr1}
\E_{X \sim G'}[f(X)] \geq \E_{X \sim \widetilde{G}}[f(X)] - \eps > (C/3)\epsilon.
\end{equation}
Let $h$ be the degree-$(k-1)$ Taylor polynomial of $f$
plus the term $(Cx\eps/\mu)^k/k!$.
By the Lagrange form of the remainder in Taylor's theorem,
we have that $\left| h(x) - (Cx\eps/\mu)^k/k!- f(x) \right| \leq f^{(k)}(\xi) x^k/k!$,
for some $\xi \in [0, x]$.
Since $k$ is even, we have that the
$k$-th derivative of $f$,
$|f^{(k)}(\xi)| = (C\eps/\mu)^k |\sin(\xi C\eps/\mu)| \leq (C\eps/\mu)^k$. Thus, we get
$$
f(x) \leq h(x) \leq f(x)+2(xC\epsilon/\mu)^k/k! \;.
$$
Our goal will be to show that $\E_{X \sim G'}[h(X)]$ is substantially larger than $\E_{X \sim N(0,1)}[h(X)]$,
which will contradict the assumption about approximately matching moments.
We start by considering
$\E_{X \sim N(\mu, 1)}[h(X)]$
versus $\E_{X \sim N(0,1)}[h(X)]$.
We can write
\begin{align*}
& \E_{X \sim N(\mu,1)}[h(X)]-\E_{X \sim N(0,1)}[h(X)] \\
& = \E_{X \sim N(\mu,1)}[f(X)]- \E_{X \sim N(0,1)}[f(X)] + O\left(\E_{X \sim N(\mu,1)}\left[(X C\epsilon/\mu)^k/k! \right] - \E_{X \sim N(0,1)}\left[(X C\epsilon/\mu)^k/k! \right] \right) \\
 & \geq (C/3)\epsilon + O\left(\E_{X \sim N(\mu,1)}\left[(X C\epsilon/\mu)^k/k! \right] - \E_{X \sim N(0,1)}\left[(X C\epsilon/\mu)^k/k! \right] \right) \;,
\end{align*}
where the last inequality follows from (\ref{eqn:contr1}).
To bound this latter term, we make the following claim:
\begin{claim}
We have that
$$ |G(x-\mu) - G(x)| \leq O(\mu)G(x/\sqrt{2}).$$
\end{claim}
\begin{proof}
First, we note that $G(x-\mu)/G(x)=\exp(-2x\mu + \mu^2/2)$.

Recalling our assumption that $0\leq \mu < 1$,
for $|x| \leq O(1/\mu)$, we have that
$|G(x-\mu) - G(x)| \leq O(|x|\mu+\mu^2) G(x) \leq O((|x|+1)\mu) G(x)$.
Then, since $G(x)/G(x/\sqrt{2})=O(G(x)) \leq O(1/(|x|+1))$,
we get $|G(x-\mu) - G(x)| \leq O(\mu)G(x/\sqrt{2})$.

For $|x|\geq 5/\mu$, since $\sqrt{2\ln(1/\mu)} + 2 \leq 2\ln(1/\mu) +3 \leq 2/\mu + 3 \leq x$,
we have that $G(|x|-2) \leq \mu$. Thus, 
$G(x)/G(x/\sqrt{2}) \leq O(G(x)) \leq O(\mu)$ and
\begin{align*}
G(x-\mu)/G(x/\sqrt{2})
& = O(\exp(-x^2/4+\mu x - \mu^2/2)) \\
& =O(G((x-2\mu)/\sqrt{2}) \exp(\mu^2/2)) \\
& \leq O(G(|x|-2)) \leq O(\mu) \;,
\end{align*}
and hence $|G(x-\mu) - G(x)| \leq  O(\mu)G(x/\sqrt{2})$.
\end{proof}
Using the previous claim, we note that
$$
\int_{-\infty}^\infty (xC\epsilon/\mu)^k/k! \left|G(x-\delta) - G(x) \right| dx =
O(\eps) \E_{X \sim N(0,1/2)}[(C\epsilon/\mu)^k/k!] = O(\epsilon/(C/\sqrt{2})^k) = O(\eps).
$$
Therefore, we have that
$$
\E_{X \sim \widetilde{G}} [h(X)]\geq \E_{X \sim G}[h(X)] + (C/3)\epsilon.
$$
Next, we wish to compare $\E_{X \sim G'}[h(X)]$ to $\E_{X \sim \widetilde G}[h(X)]$.
Using (\ref{eqn:noise-def}), we have the following for $\eps' = \dtv(G',\widetilde G)$:
\begin{align*}
\E_{X \sim G'}[h(X)] & = \E_{X \sim \widetilde G}[h(X)] + \eps' \left(\E_{X \sim E}[h(X)] - \E_{X \sim L}[h(X)]\right) \\
 &\geq \E_{X \sim \widetilde G}[h(X)] +
 \eps' \left(\E_{X \sim E}[f(X)] - \E_{X \sim L}[f(X)] - \E_{X \sim L}\left[(X\epsilon/\delta)^k/k!\right]\right) \\
 & = \E_{X \sim \widetilde G}[h(X)] + O(\epsilon) - \eps' \E_{X \sim L}\left[(XC\epsilon/\mu)^k/k!\right] \;.
\end{align*}
From (\ref{eqn:noise-cor}), i.e., $\widetilde G \geq \eps' L$, it follows
that $L$ satisfies the concentration inequality
$$\Pr_{X \sim L}[|X - \mu| > T] \leq 2\exp(-T^2/2)/ \eps' \;.$$
We now proceed to bound the subtractive term from above:
\begin{align*}
& \eps' \E_{X \sim L}\left[(XC\epsilon/\mu)^k/k!\right] \\
& \leq \eps' \sum_{J=0}^\infty \Pr_{X \sim L}\left[|X - \mu| \geq 2J\sqrt{\log(1/\eps)}\right]
\left(2(J+1)\sqrt{\log(1/\eps)} + \mu\right)^k(C\epsilon/\mu)^k/k! \\
& \leq \eps \left(3\sqrt{\log(1/\eps)} C\eps/\mu\right)^k/k! + \sum_{J=1}^\infty 2 \eps^J (2(J+2)\sqrt{\log(1/\eps)}C\eps/\mu)^k/k! \\
& \leq \eps \cdot O\left(\sqrt{\log(1/\eps)}C\eps/\mu\right)^k/k! \cdot \left(1 + \sum_{J=1}^\infty \eps^J (J+2)^k\right) \\
& \leq \eps \cdot O\left(\sqrt{\log(1/\eps)}\eps/C\delta\right)^k/k!
\leq \eps \cdot O\left( \sqrt{\log(1/\eps)}\eps/C\delta \right)^k/(\sqrt{2 \pi k} (k/e)^k) \\
& \leq \eps O(\sqrt{\log(1/\eps)}\eps/C \delta k)^k/\sqrt{k}
\leq O(\eps/\sqrt{k} C^k)
\leq O(\eps) \;.
\end{align*}
Therefore, we conclude that
$$
\E_{X \sim G'}[h(X)] \geq \E_{X \sim N(\mu,1)}[h(X)] +O(\eps) \geq \E_{X \sim N(0,1)} [h(X)] + (C/4)\eps \;.
$$
On the other hand, recalling that $h$ is the degree-$(k-1)$ Taylor expansion of
$f(x)=\sin(x C \eps/\mu)$, we can write
$h(x) = \sum_{i=0}^{k-1} a_i x^i$,
with $a_i = O((C \eps/\mu)^i/i! )=O((\eps/C \delta)^i/i!)$.
Therefore, the difference $\E_{X \sim G'}[h(X)]-\E_{X \sim N(0,1)}[h(X)]$
is the sum over $i$ of $a_i$ times
the difference in the $i^{th}$ moments,
which by assumption is at most
$$
\eps \sum_{i=0}^{k-1} (i-2)! (\delta/C\eps)^i O((\eps/\delta)^i/i!)
= O(\eps) \sum_{i=1}^{k-1} (1/i)^2 =O(\eps) \;.
$$
This contradicts the fact that their difference
is at least $(C/4)\epsilon$,
and concludes the proof.
\end{proof}

\subsection{Robust Testing Algorithm} \label{ssec:test-alg}

In this subsection, we give a robust testing algorithm, i.e., an algorithm that 
distinguishes between an  $\eps$-noisy Gaussian and $N(0, I)$.
This algorithm will form the basis for our robust learning algorithm
of the following subsection.

\begin{theorem}[Robust Testing Algorithm] \label{thm:robust-testing-alg}
Let $G'$ be an $\eps$-noisy version of an $n$-dimensional Gaussian with identity covariance.
Let $\delta$ be at least a sufficiently large constant multiple of $\epsilon$.
There exists an SQ algorithm that makes $O(n^k)$ queries
to  $\mathrm{STAT}(\eps \cdot O(n\log(n/\eps)^2)^{-k})$ where $k=2 \lceil O(\eps\sqrt{\log(1/\eps)}/\delta) \rceil,$ 
and
distinguishes between the cases that $G'$ is the standard normal distribution $N(0, I)$,
and the case that $G'$ is at least $\delta$-far from $N(0, I)$. 
The algorithm has running time $n^{O(k)}$.
\end{theorem}

By simulating the statistical queries with samples,
we obtain:

\begin{corollary} \label{cor:test-alg}
Given sample access to $G'$, an $\epsilon$-noisy version of an $n$-dimensional Gaussian 
with identity covariance and $\eps,\delta >0 $ with $\delta$ be at least a sufficiently large constant multiple of $\epsilon$,
there is an algorithm that with probability $9/10$
distinguishes between the cases that $G'$ is the standard normal distribution $N(0, I)$,
and the case that $G'$ is at least $\delta$-far from $N(0, I)$ 
and requires at most $(n \log(1/\eps))^{O(k)}/\eps^2$ samples
and running time where $k=2 \lceil O(\eps\sqrt{\log(1/\eps)}/\delta) \rceil.$
\end{corollary}

\begin{proof}
In the case when $G'$ is $\delta$-far from $N(0, I)$,
let $\widetilde{G}=N(\mu,I)$ be such that
$\dtv(G', \widetilde G) \leq \eps$.
We need to show that we can distinguish between the cases
$G'=N(0,I)$ and $G'$ is an $\eps$-noisy version of a Gaussian
$\widetilde G=N(\mu,I)$ with $\dtv(\widetilde G, G) \geq \delta -\eps$.
We assume from now on that the completeness case
is to show that $\dtv(\widetilde G, G) \geq \delta$,
since replacing $\delta$
with $\delta+\eps$ does not affect the statement of the theorem.

The algorithm is quite simple.
Let $C$ be a sufficiently large universal constant
such that the $O(\delta)$ total variation distance bound
in Lemma~\ref{lem:moment-matching} is less than $C\delta$.
We assume that $\delta > 4 C \eps$.

\medskip

\noindent {\bf Robust Testing Algorithm:}
\begin{itemize}

\item For each coordinate axis $1 \leq i \leq n$, use $\mathrm{STAT}$ to approximate 
$\Pr_{X \sim G'}[X \leq \eps]$ and $\Pr_{X \sim G'}[X \geq \eps]$ to within $\eps/2$. 
If any of these approximations are bigger than $1/2+\eps$, then output ``NO''.

\item Let $k= 2 \lceil 2C\eps \sqrt{\ln(1/\eps)}/\delta \rceil$. Let $C'$ be a sufficiently large constant.

\item Using access to $\mathrm{STAT}$,
find all the mixed moments of $X \sim G'$, conditioned on $\|X\|_2 \leq C' k \sqrt{n \log(n/\eps)}$, 
of order at most $k$ to within error $n^{-k/2} \epsilon/2$.

\item If the difference between any moment of order $t \leq k$ that we measured
and that of $N(0,I)$ is more than $((t-1)!(\delta/2C\epsilon)^t/t - 1) \cdot n^{-k/2} \eps$,
then output ``NO''.
\item Otherwise, output ``YES''.
\end{itemize}

The idea is to use Lemma~\ref{lem:moment-matching} with the approximations the moments. 
However, we have the issue that the STAT oracle can only be used to approximate the expectation of a bounded function. 
Using the condition $\|X\|_2 \leq C' k \sqrt{n \log(n/\eps)}$ allows us to avoid this. 
But we first need to show that conditioning on it 
does not affect the moments too much and 
does not move the distribution far in total variational distance.

Note that the first step of the algorithm will reject if the median of $N(\mu,I)$ 
projected onto any coordinate axis is outside of the interval $[-\eps,\eps]$. 
If this occurs, then $\mu \neq 0$. If this step does not reject, 
then $\mu$ projected onto any coordinate axis is $O(\eps)$, 
and so $\|\mu\|_2 \leq O(\eps\sqrt{n})$.

The condition $\|\bx\|_2 \leq C' k \sqrt{n \log(n/\eps)}$ ensures that the any degree less than $k$ monomial in $X$ 
is at most $(C' nk^2 \log(n/\eps)))^{k/2}$. Thus, we can approximate this expectation to precision $n^{-k/2} \epsilon$ 
using $\mathrm{STAT}(\eps \cdot O(n\log(n/\eps)^2)^{-k})$. 
However, since $\|\mu\|_2 \leq O(\eps\sqrt{n})$, by standard concentration bounds, 
the probability that $X \sim N(\mu,I)$ does not satisfy this condition is at most $\eps \cdot (C' nk^2 \log(n/\eps)))^{-k}$.
Let $G''$ be $G'$ conditioned on $\|X\|_2 \leq C' k \sqrt{n \log(n/\eps)}$. 
If $G'=N(0,I)$, then we need to show that the moments of $G''$ and $G'$ are within $\ln^{-k/2} (\epsilon/2)$. 
To show this, we use the following lemma:
\begin{lemma} \label{lem:momemts-prune} 
For  $0 \leq \delta  \leq \exp(-k)$, the difference between any mixed moment of degree at most $k$ 
of $N(0,I)$ and $N(0,I)$ conditioned on $\|x\|_2 \leq O( \sqrt{n k \log(1/\delta)})$, is at most $\delta$. 
\end{lemma}
\begin{proof}
Let $X$ be distributed as $N(0,I)$. 
Let $\eps$ be $\delta/(k+\ln 1/\delta)^{(k-1)}C$ for a sufficiently large constant $C$. 
Thus, $\ln(1/\eps) = \ln(1/\delta) + \ln C + (k-1) \ln (k + \ln(1/\delta))$. 
Let $T$ be $\sqrt{2n \log 1/\eps}$. 
Then, $T=O_C(\sqrt{n k \log(1/\delta)})$ and by standard concentration inequalities, 
we have that $\Pr[\|X\|_2 \geq T] \leq \eps$. 

For $\ba \in \N^n$ with $\|\ba\|_1 \leq k$, 
consider the monomial of degree at most $k$, $m_{\ba}(\bx) = \prod_{i=1}^n x_i^{a_i}$. 
First we consider its mean and variance. 
If any $a_i$ is odd, $\E(m_\ba)(X)= \prod_{i=1}^n \E(X_i^{a_i})=0$, 
since the odd moments of $X_i \sim N(0,1)$ are zero. 
If all $a_i$ are even, then $\E(m_\ba)(X)= \prod_{i=1}^n \E(X_i^{a_i})=\prod_i 2^{a_i/2} (a_i/2)! \leq 2^{k/2} (k/2)!$. 
For the variance, we have 
$\var[m_\ba(X)] \leq \E[m_{\ba}(X)^2] =  \prod_{i=1}^n \E(X_i^{2a_i})=\prod_i 2^a_i a_i! \leq 2^k k!$. 
Let $p_\ba(X)=m_\ba(X)/2^{k/2} \sqrt{k!}$. 
Then we have that $0 \leq \E[p_\ba(X)] \leq 1$ and $\var[p_\ba(X)] \leq 1$.

By the standard concentration inequality given in Lemma \ref{lem:hypercontractivity-conc} below, 
we have that for all $t > 0$, $\Pr[|p_\ba(X)| \geq t + 1] \leq \exp(2- (t/R)^{2/k})$ for some $R > 0$.
Thus, we have $\Pr[|p_\ba(X)| \geq c+1] \leq \eps$, 
for $c=  R  (\ln(1/\eps)-2)^{k/2}$. 
Let $I(\bx)$ be the indicator function of $\|\bx\|_2 \geq T$. Then we have that
\begin{align*}
|\E[I(X)p_\ba(X)]| & \leq  \E[I(X) |p_\ba(X)|] \\
& = \int_0^{\infty} \Pr[ I(X) |p_\ba(X)| \geq t] dt \\
& \leq \int_0^{c+1} \eps dt + \int_{c+1}^\infty  \exp(2- (t-1/R)^{2/k}) dt \\
& = (c+1) \eps + \int_c^\infty  \exp(2- (t/R)^{2/k}) dt \\
& = (c+1) \eps + \int_{\ln(1/\eps)}^\infty \exp(2-x) (dt/dx) dx \tag*{(where $x=(t/R)^{2/k}$)} \\
& = (c+1) \eps + (Rk/2) \int_{\ln(1/\eps)}^\infty \exp(2-x) x^{k/2-1} dx \\
& = (c+1) \eps + (Rk/2) \eps \cdot \sum_{j=0}^{k/2-1} ((k/2)!/(k/2-j)!) \ln(1/\eps)^{k/2 - j} \\
& \leq R  \ln(1/\eps)^{k/2} \eps + (k^2/8) \eps (k + \log(1/\eps))^{k/2-1} \\
& \leq O(k^2 \eps (k + \log(1/\eps))^{k/2-1}) \;,
\end{align*}
where the integral $ \int_{\ln(1/\eps)}^\infty \exp(2-x) x^{k/2-1} dx$ is calculated explicitly below in Claim \ref{clm:damn-integral}.
In terms of $m_a(\bx)$, we have 
$$|\E[I(X)m_\ba(X)]| \leq O(k^2 \eps (k + \log(1/\eps))^{k/2-1} 2^{k/2} \sqrt{k!}) 
\leq O(\eps (k+\log(1/\eps)^{(k-1)}) 
\leq O(\delta/C) \leq \delta/2 \;.$$
Then, for $X'$ distributed as $N(0,I)$ conditioned on $\|X'\|_2 \leq T$, we have
\begin{align*}
|\E[m_\ba(X)]-\E[m_\ba(X')]| & = \left|\E[m_\ba(X)]- \E[m_\ba(X) (1-I(X))]/(1-\Pr[\|X\|_2 > T]) \right|\\
& = \left| \left( \E[I(X)m_\ba(X)] - \E[m_\ba(X)]\Pr[\|X\|_2 > T] \right) / 1-\Pr[\|X\|_2 > T])  \right| \\
& \leq \left( \delta/2 + 2^{k/2} \sqrt{k} \eps \right) /(1-\eps) \\
& \leq 2\delta/3(1-\eps) \leq \delta \;.
\end{align*}
\end{proof}

Applying Lemma \ref{lem:momemts-prune} for $\delta=n^{-k/2} \eps$, 
noting that $C' k \sqrt{n  \log(n/\eps)} = \Omega(C'  \sqrt{n k \log(1/\delta)})$ 
yields that the moments of $G''$ and $G'=N(0,I)$ are within $\ln^{-k/2} (\epsilon/2)$.
Thus, in this case, the approximations of the moments of $G''$ are within $n^{-k/2} \eps$ of the moments of $G'$.

For the soundness case, we just note that since
$(t-1)!(\delta/2C\epsilon)^t\epsilon/t - 1 \geq (\delta/2C\epsilon)/2 -1 \geq 1$,
the bounds on the moments we need to fail are bigger
than the precision of the statistical queries
we use to approximate them, and therefore we never output ``NO'' when $G'=N(0,I)$.

Now suppose that $G'$ is an $\epsilon$-noisy version of an identity covariance Gaussian
$\widetilde G$. Then $G''$ is a $2\eps$-noisy version of $\widetilde G$.
We will denote $\mu$ the mean vector of $\widetilde G$ and will assume that $\|\mu\|_2 \geq \delta$.
We need to show that the algorithm outputs ``NO''.

Consider the unit vector $v=\mu/\|\mu\|_2$ which has
$v\cdot \mu \geq \delta$.
Consider the one-dimensional distributions $G''_v$
and $\widetilde G_v$ of the form $v\cdot X$, where either $X \sim G''$ or $X \sim \widetilde G$.
Note that $\widetilde G_v=N(\|\mu\|_2,1)$ has mean larger
than $\delta$ and that $\dtv(G''_v, \widetilde G_v) \leq 2\eps$.
We can now apply the contrapositive
of Lemma~\ref{lem:moment-matching} with $\delta/C$ in place of $\delta$ and $2\eps$ in place of $\eps$, which
implies that there is a $t \leq k$ such that the $t$-th moment of
$G'_v$ is more than $(t-1)!(\delta/2C\eps)^t\eps/t$ far from that of $N(0,1)$.
That is,
$$|\E_{X \sim G'}[(v \cdot X)^t] - \E_{X \sim N(0,I)}[(v \cdot X)^t]| \geq   (t-1)!(\delta/2C\eps)^t\eps/t \;.$$
Now consider the polynomial $(v \cdot \bx)^t$.
Note that the coefficient of a monomial $\prod_i x_i^{a_i}$,
for $\ba \in \Z_{>0}^n$, $\|\ba\|_1 =t$,
is given by the multinomial theorem as ${t \choose a_1, \dots , a_n} \prod_{i=1}^n v_i^{a_i}$.
The $L_1$-norm of its coefficients is the same as the $L_1$-norm
of entries of the rank-$t$ tensor $v^{\otimes t}$,
which has $\bi$-th entry
$\prod_{j=1}^t v_{i_j}$,
for $\bi \in \{0, \dots , n\}^t$,
since there are ${t \choose a_1, \dots , a_n}$
entries in this symmetric tensor
which are given by $\prod_{i=1}^n v_i^{a_i}$ for any $\ba$.
The Frobenius norm of $v^{\otimes t}$,
the $L_2$-norm of its entries,
is $\sum_{\bi \in \{0, \dots , n\}^t} \prod_{j=1}^t v_{i_j}^2 = \prod_{j=1}^t \sum_{i=1}^n v_i^2 = 1$.
Since there are $n^t$ entries,
the $L_1$-norm of its entries must be at most $n^{t/2} \leq n^{k/2}$.
We can write $\E_{X \sim G''}[(v \cdot X)^t] - \E_{X \sim N(0,I)}[(v \cdot X)^t]$
as a linear combination of differences in moments
with these coefficients, and we can thus bound this from above
by the product of the $L_1$-norm of the coefficients
and the $L_\infty$-norm of the differences in moments.
Hence, there must be some moment
$\E[\prod_i X_i^{a_i}]$ with $\ba \in \Z_{>0}^n$, $\|\ba\| \leq k$, such that
$$\left|\E_{X \sim G'}\left[\prod_i X_i^{a_i}\right] - \E_{X \sim N(0,I)}\left[\prod_i X_i^{a_i}\right]\right|
\geq (t-1)!(\delta/2C\eps)^t\epsilon/t \cdot n^{-k/2} \;.$$
This in turn means that the difference in the approximation
of this moment of $G''$ and that of $N(0,I)$
is at most $((t-1)!(\delta/2C\eps)^t/t - 1) \cdot n^{-k/2} \eps$.
Thus, the testing algorithm outputs ``NO''.
\end{proof}

\subsection{Robust Learning Algorithm} \label{ssec:learn-alg}

In this section, we build on the testing algorithm of the previous section
to design our robust learning algorithm. Formally, we prove:

\begin{theorem} \label{thm:upper-bound-learning}
Let $G'$ be an $\epsilon$-noisy version of an $n$-dimensional Gaussian
with identity covariance matrix, $N(\mu,I)$ with $\|\mu\|_2 \leq \poly(n,\eps)$.
There is an algorithm that, given statistical query access to $G'$,
outputs an approximation $\widetilde \mu$ to the mean $\mu$
such that $\|\mu-\widetilde \mu\|_2 \leq O(\eps)$. The algorithm uses
$n^{O(\sqrt{\log(1/\eps)})} + 2^{\log(1/\eps)^{O(\sqrt{\log(1/\eps)})}}$ calls to $\mathrm{STAT}(\eps/(n\ln(1/\eps))^{{O(\sqrt{\log(1/\eps)})}})$
and has running time $n^{O(\sqrt{\log(1/\eps)})} + 2^{\log(1/\eps)^{O(\sqrt{\log(1/\eps)})}}$.
\end{theorem}

By simulating the statistical queries with samples,
we obtain:

\begin{corollary} \label{cor:learn-alg}
Given sample access to $G'$, an $\epsilon$-noisy version of an $n$-dimensional Gaussian $N(\mu,I)$,
there is an algorithm that with probability $9/10$
outputs $\widetilde \mu$ with $\|\mu-\widetilde \mu\|_2 \leq O(\eps)$
and requires $(n\log(1/\eps))^{O(\sqrt{\log(1/\eps)})}/\eps^2$ samples
and $n^{O(\sqrt{\log(1/\eps)})}/\eps^2 + 2^{\log(1/\eps)^{O(\sqrt{\log(1/\eps)})}}$ time.
\end{corollary}

The work~\cite{DiakonikolasKKLMS16} gives algorithms which can compute
an approximation $\mu'$ with $\|\mu-\mu'\|_2 \leq O(\eps \sqrt{\log(1/\eps)})$.
These algorithms can be expressed as Statistical Query algorithms.
However, due to the model of adversary used for robustness in \cite{DiakonikolasKKLMS16},
the algorithms were expressed there in terms of operations on sets of samples
that were drawn before the execution of the algorithm.
The filtering algorithms work by successively removing samples
from this set and then computing expectations of the current set of remaining samples.
The samples that are removed are those that satisfy an explicit condition,
we say that they are rejected by a filter. We can implement these algorithms
as SQ algorithms by replacing expectations of the current set of remaining
samples with the conditional expectation of the input distribution,
conditioned on all previous filters accepting.
This is similar to the filtering algorithm for learning
binary Bayesian networks given in \cite{DiakonikolasKS16b}.
 Even there, we still used samples to compute the threshold for the filter.
We note that using arguments similar to those we use for the algorithm below, 
 all theses algorithms can be expressed as SQ algorithms. In particular,  this is the case for Algorithm 
 {\tt Filter-Gaussian-Unknown-Mean}, which we will use as a black box pre-processing step
 to approximate the mean within $O(\eps \sqrt{\log(1/\eps)})$.

Instead of dealing with moments, i.e., the expectations of monomials, directly,
we will consider expectations of Hermite polynomials,
which have a simpler form for normal distributions.

\begin{definition} \label{defn:Hermite}
We define multi-dimensional normalized Hermite polynomials as follows: for $\ba \in \Z^n$,
$He_{\ba}(\bx)=\prod_{i=1}^n He_{a_i}(\bx_i).$
We define $n(\ba)= \prod_{i=1}^n a_i !$.
\end{definition}
Thus we have the following:

\begin{fact}
$\E_{X \sim N(0, I)}[He_\ba(X) He_\bb(X)] = \delta_{\ba \bb} n(\ba)$.
\end{fact}

We define a linear combination of these Hermite
polynomials with degree $t$ associated with a tensor of rank $t$.
For $i \in \{0,\dots, n\}^t$, we define the count vector
$c(i) \in  \Z_{\geq 0}^n$ such that $c(i)_j$ is the number of coordinates of $i$ that are $j$.

\begin{definition} \label{def:tensor-Hermite}
For a rank-$t$ tensor $A$ over $\R^n$, we define
$$h_A(\bx)=\sum_{\bi \in \{0,\dots, n\}^t} He_{c(\ba)}(\bx)/\sqrt{t!} \;.$$
\end{definition}

We are now ready to describe our learning algorithm.

\medskip

\noindent {\bf Robust Learning Algorithm:}
\begin{enumerate}
\item Let $k=2 \lceil \sqrt{\ln(1/\eps)}) \rceil$.
\item \label{step:old-alg} Compute an approximation $\mu'$ with $\|\mu'-\mu\|_2 \leq O(\eps \sqrt{\log(1/\eps)})$
by iterating Algorithm \\ {\tt Filter-Gaussian-Unknown-Mean} from \cite{DiakonikolasKKLMS16}.
We  change the origin so that $\mu'=0$.

\item \label{step:naive-prune} Let $N$ be the filter that accepts when $\|\bx\|_2 \leq \sqrt{2 n \log(1/\eps)}$. 

\item For $1 \leq t \leq k$, let $\wt P_t$ be the rank-$t$ tensor with $i_1, \ldots, i_t$ entry
given by $\sqrt{t!}$ times the result of asking an SQ oracle
for $\E_{X \sim G'}[h_{c(i)}(X)]$
conditioned on $N$ accepting to within precision $\eps/n^{t/2}$.


\item \label{step:filter} While $\|\wt P_t \|_F \geq \eps  \Omega(\log(1/\eps))^{t/2}$ for some $t$,
\begin{itemize}
\item Let $t'$ be the least $t$ such that $\|\wt P_{t'} \|_F \geq \eps \Omega(\log(1/\eps))^{t'/2}$.

\item  Let $A=\wt P_{t'}/\|\wt P_{t'}\|_F$.
Let $h_A(x) = \sum_{i_1,\dots i_{t'}} A_i He_{c(i)}(x)/\sqrt{t!}$
For each positive integer $T$, approximate 
$$
\Pr_{X \sim \widetilde G}[|h_A(X)| \geq T+1]
$$
until one is found that is at least $$3\exp(2-\Omega(T)^{2/t'})) + \eps/Cn^{2t'} \;$$
for a sufficiently large constant $C$.
Let $F$ be the filter that accepts when $|h_A(x)| \leq T+1$. 
\item  Recalculate $\wt P_t$, for all $1 \leq t \leq k$,
where all expectations are conditioned $N$ and the filters $F$ from all previous iterations.
\end{itemize}
\item End while.

\item For $1 \leq t \leq k$, compute the SVD of $M(\wt P_t)$,
$\wt P_t$ considered as an $n \times n^{t-1}$ matrix,
and let $V_t \subseteq \R^n$ be the subspace spanned
by all right singular vectors of $P_t$
with singular value more than $\eps$.

\item Let $V$ be the span of $V_1, \ldots, V_k$.

\item Let $S \subset V$ be a set of unit vectors of size $|\dim (V)|^{O(\dim (V))}$
such that for any unit vector $v \in V$, there is a $v' \in S$ with $\|v-v'\|_2 \leq 1/2$.

\item For each $v \in S$, compute the median $m_v$ of $v^T X$, for $X \sim G'$,
to within $\eps/\sqrt{\dim V}$
using bisection and statistical queries to approximate the
$\Pr[v^T X \leq m]$ for $m=\mu'+O(\eps \sqrt{\log(1/\eps)})$.
(We don't need to condition on any filters here).

\item \label{step:lp} Find a feasible point $\widetilde \mu_V$ of the LP
$\widetilde \mu_V \in V$ with $|v^T \widetilde \mu_V - v^T m_v| \leq O(\eps)$ for all $v \in S.$

\item Return $\widetilde \mu_V$.
\end{enumerate}

Using similar techniques to those used to express this algorithm in terms of Statistical Queries, 
we can run  Algorithm  {\tt Filter-Gaussian-Unknown-Mean} using $\poly(n/\eps)$ time 
and calls to $\mathrm{STAT}(\widetilde{O}(\eps/\poly(n)))$.

Note that we can approximate conditional expectations easily as a ratio of expectations approximated 
by two SQ queries. Since, as we will show, our filters only throw away at most an $O(\eps)$ fraction of points,
we will not need to increase the precision beyond a constant factor to do this.

The algorithm needs approximate expectations to within $\eps/n^{{O(\sqrt{\log(1/\eps)})}}$. 
To show that we can use the oracle $\mathrm{STAT}(\eps/(n\ln(1/\eps))^{{O(\sqrt{\log(1/\eps)})}})$ to obtain this, 
we need to note that the  distributions we approximate the expectations of 
are supported in an interval of length $(n\ln(1/\eps))^{{O(\sqrt{\log(1/\eps)})}}$. 
Thanks to the naive pruning of Step \ref{step:naive-prune}, 
only $\bx$ with $\|\bx\|_2 \leq \sqrt{2 n \log(1/\eps)}$ contribute to these expectations. 
This suffices to show that the maximum value of all polynomials 
we consider on any such $\bx$ is at most $(n\ln(1/\eps))^{{O(\sqrt{\log(1/\eps)})}}$.

We need to show the following for the filter step of our algorithm:
\begin{proposition} \label{prop:filter-works}
The loop in Step~\ref{step:filter} takes $O(n^{2k})$ iterations
and all filters together accept with probability at least $1-O(\eps)$.
\end{proposition}

\subsubsection{Proof of Proposition~\ref{prop:filter-works}}

We now proceed with the proof.
By standard concentration bounds, $N$ accepts with probability $1-O(\eps)$.
Let $F'$ be the event that $N$ and all filters $F$ from previous iterations accept.
We assume inductively that $\Pr_{G'}[F'] \leq O(\eps)$,
and need to show that the same holds if we include
the filter $F$ produced in the current iteration.

Let $\widetilde{G}=N(\mu,I)$ be a Gaussian with $\dtv(G',\widetilde{G}) \leq \eps$
such that $\|\mu\|_2 \leq O(\eps \sqrt{\log(1/\eps)})$.
We write $G'|F'$ for the distribution obtained by conditioning on $F'$.
Since $\Pr_{G'}[F'] \leq C\eps$, we have that
$$\dtv(G'|F,\widetilde{G}) \leq \dtv(G',\widetilde G) + \dtv(G', G'|F') \leq (C+1) \eps \;.$$
Thus, for the first iteration, we have
$G'|F'= \widetilde{G} + \dtv(G'|F, \widetilde{G}) E - \dtv(G'|F,\widetilde{G}) L$,
for distributions $E$ and $L$ with disjoint supports.

For any iteration, we will write $G'|F'= w_{\widetilde{G}} \widetilde{G} + w_E E - w_L L$,
for distributions $E$ and $L$ with disjoint supports,
where $w_E+w_L=O(\eps)$ and
$w_{\widetilde G} = 1+O(\eps)$.
In the first iteration, we will take $w_{\widetilde G}=1$
and $w_E = w_L = \dtv(G'|F', \widetilde G)$.

We will need properties of the polynomials $h_A(x)$ for the analysis.
In particular, we show the following:
\begin{lemma} \label{lem:tensor-stuff}
Given a rank-$t$ symmetric tensor $A$ over $\R^n$,
let $h_A(x)$ be as in Definition~\ref{def:tensor-Hermite}.
Then, we have:
\begin{itemize}
\item[(i)] $\E_{X \sim N(0,I)}[h_A(X)^2] = \|A\|_F^2.$

\item[(ii)] If $B$ is a  rank-$t$ tensor with $B_i = \sqrt{t!} \E_{X \sim \p}[He_{c(i)}(X)]$,
for a distribution $\p$, then $\E_{X \sim \p}[h_A(X)]= \sum_i A_i B_i$.

\item[(iii)] We can recover $A$ from $h_A(x)$ using
$\sqrt{t!} A_{i_1,\dots ,i_t} = \frac{\partial}{\partial x_{i_1}} \cdots \frac{\partial}{\partial x_{i_t}} h_A(x)$.

\item[(iv)] If $O$ is an orthogonal matrix, then $h_A(O \bx) = h_B(\bx)$
for a symmetric rank-$t$ tensor $B$ with $\|B\|_F=\|A\|_F$.

\item[(v)] If $B$ is a  rank-$t$ tensor with $B_i = \sqrt{t!} \E_{X \sim \p}[He_{c(i)}(X)]$,
for a distribution $\p$, and $j > 0$, $v \in \R^n$,
then $\E_{X \sim \p}[He_j(v \cdot X)]/\sqrt{t!}= B(v,\ldots,v)$.
\end{itemize}
\end{lemma}

\begin{proof}
For (i), we first need to get an expression for the coefficients of each $He_\ba(x)$,
since they appear multiple times in $h_A(\bx)=\sum_{\bi \in \{0,\dots, n\}^t} A_\bi He_{c(\ba)}(\bx)/\sqrt{t!}.$
Let $c^{-1}(a)$ be a function mapping $\ba \in \Z_{\geq 0}^t$
with $\|\ba\|_1=t$ to $\bi \in \{1,\dots,n\}^t$,
with $c(c^{-1}(\ba))=\ba$ for all $\ba$.
Since $A$ is symmetric, the choice of $c^{-1}$ does not affect $A_{c^{-1}(\ba)}$ for any $\ba$.
Note that, for a given $\ba$, there are ${t \choose a_1,\dots ,a_n} = t!/n(\ba)$ possible $\bi$ with $c(\bi)=\ba$.
Thus, we have:
\begin{align*}
h_A(\bx) & = \sum_{\bi \in \{0,\dots, n\}^t} A_\bi He_{c(\ba)}(\bx)/\sqrt{t!} \\
& = \sum_{\|\ba\|_1=t}  t!/n(\ba) \cdot A_{c^{-1}(\ba)} He_\ba(\bx)/\sqrt{t!} \\
& = \sum_{\|\ba\|_1=t}  \sqrt{t!}/n(\ba) \cdot A_{c^{-1}(\ba)} He_\ba(\bx) \;.
\end{align*}
Now, by orthogonality of $He_\ba(\bx)$ with distinct $\ba$, we have that,
for $X \sim N(0,I)$, it holds:
\begin{align*}
\E[h_A(X)^2] & = \sum_{\|\ba\|_1=t}  t!/n(\ba)^2 \cdot A_{c^{-1}(\ba)}^2 \E[He_\ba(X)^2] \\
&= \sum_{\|\ba\|_1=t}  t!/n(\ba) \cdot A_{c^{-1}(\ba)}^2 \\
& = \sum_\bi A_\bi^2 = \|A\|_F^2 \;.
\end{align*}
For (ii), we now have that
\begin{align*}
\E_{X \sim \p}[h_A(X)] & = 1/\sqrt{t!} \cdot  \sum_\bi A_\bi \E_{X\sim \p}[He_{c(\bi)}(x)] \\
& = \sum_\bi  A_\bi B_\bi \;.
\end{align*}
For (iii), note that $\ba$ with $\|\ba\|_1=t$, $He_\ba(x)$
has only one monomial of degree $t$,
which is $\prod x_i^{a_i}$. Thus, given $\bi \in \{1,\dots,n\}^t$,
there is only one $\ba$ with $\|\ba\|_1=t$ and
$\frac{\partial}{\partial x_{i_1}} \cdots \frac{\partial}{\partial x_{i_t}} He_\ba(x) \neq 0$,
which is $\ba=c(\bi)$ and has
\begin{align*}
\frac{\partial}{\partial x_{i_1}} \cdots \frac{\partial}{\partial x_{i_t}} He_{c(\bi)}(x)
& = \frac{\partial}{\partial x_{i_1}} \cdots \frac{\partial}{\partial x_{i_t}} \prod_j x_j^{c(\bi)_j} \\
& = \prod_j c(\bi)_j! = n(c(\bi)) \;.
\end{align*}
Thus, we have
\begin{align*}
\frac{\partial}{\partial x_{i_1}} \cdots \frac{\partial}{\partial x_{i_t}} h_A(\bx)
& = n(c(\bi)) \cdot \sqrt{t!}/n(\ba) \cdot A_{c^{-1}}(\ba) \\
&= \sqrt{t!} A_\bi \;.
\end{align*}
For (iv), consider the linear transformation of rank-$t$ tensors $O^{\otimes t}$,
which for our purposes can be defined as the unique function such that
$(O^{\otimes t} A)(v_1,\ldots,v_t) = A(O^T v_1, \ldots, O^T v_t)$,
for all rank-$t$ tensors over $\R^n$, $A$,
and vectors $v_1 \ldots v_t \in \R^n$.
It is a fact that $\|O^{\otimes t} A\|_F=\|A\|_F$.
Now, we consider the $t$-th order directional derivatives of a function $f$,
$\nabla_{v_1},\cdots \nabla_{v_t} f(\bx)$,
where $\nabla_v g(\bx) = \sum_i v_i \frac{\partial g}{\partial x_i}(\bx)$.
These can be expressed in terms of the rank-$t$ tensor $F$,
with $F_\bi =  \frac{\partial}{\partial x_{i_1}} \cdots \frac{\partial}{\partial x_{i_t}} f(x)$,
as $\nabla_{v_1}, \dots \nabla_{v_t} f(\bx)= F(v_1, \ldots v_t)$.
We have that, for all $\bi$,
\begin{align*}
& \frac{\partial}{\partial x_{i_1}} \cdots \frac{\partial}{\partial x_{i_t}} h_A(O \bx)  \\
& = \nabla_{\be_{i_1}} \cdots \nabla_{\be_{i_t}} h_A(O \bx) \\
& = \nabla_{O^T \be_{i_1}} \cdots \nabla_{O^T \be_{i_t}} h_A( \bx) \\
& = \sqrt{t!} A(O^T \be_{i_1}, \dots, O^T \be_{i_t}) \\
& = \sqrt{t!} (O^{\otimes t} A)_i \\
& =  \frac{\partial}{\partial x_{i_1}} \dots \frac{\partial}{\partial x_{i_t}} h_{O^{\otimes t} A}(\ba) \;.
\end{align*}
Since $h_A(O \bx)$ and $h_{O^{\otimes t} A}(\bx)$ are both multivariate polynomials of degree $t$,
that these derivatives agree means that the coefficients of all monomials of degree-$t$ agree.
We thus have $h_A(O \bx)=h_{O^{\otimes t} A}(\bx) + p(x)$,
where $p$ is a polynomial of degree at most $t-1$.
Since $h_{O^{\otimes t} A}(\bx)$ is a linear combination of Hermite
polynomials of degree $t$,
which are orthogonal to all polynomials of degree smaller than $t$,
we have $\E_{X \sim N(0,I)}[h_{O^{\otimes t} A}(X) p(X)]=0$, and so
\begin{align*}
1 &= \|A\|_F = \E_{X \sim N(0,I)}[h_A(O \bx)^2] \\
& = \E_{X \sim N(0,I)}[h_{O^{\otimes t} A}(X)^2] + \E_{X \sim N(0,I)}[p(X)^2] + 2 \E_{X \sim N(0,I)}[h_{O^{\otimes t} A}(X) p(X)] \\
& = \|O^{\otimes t} A\|_F + \E_{X \sim N(0,I)}[p(X)^2]  + 0 \\
& = 1 + \E_{X \sim N(0,I)}[p(X)^2]  \;.
\end{align*}
Since $\E_{X \sim N(0,I)}[p(X)^2] =0$, we must have $p(\bx) \equiv 0$,
and thus $$h_A(O \bx)=h_{O^{\otimes t} A}(\bx) \;.$$
Taking $B= O^{\otimes t} A$ gives (iv).

For (v), let $O$ be an orthogonal matrix that gives a rotation mapping
$e_1$ to $v$, $\ba=\{t, 0, \ldots,0\}$ and $T_1$ be the rank-$t$ tensor
with $(1, \ldots ,1)$ entry $1$ and every other entry $0$.
Then, we can rewrite
$$He_t(v \cdot \bx)= He_t ( (O \bx)_1 ) = He_\ba( (O \bx)_1 ) = \sqrt{t!} h_{T_1}(O \bx) = \sqrt{t!} h_{O^{\otimes t} T_1}(\bx) \;.$$
For any $\bi$, we have
\begin{align*}
(O^{\otimes t} T_1)_\bi
&= (O^{\otimes t} T_1)(\be_{i_1}, \ldots, \be_{i_t}) \\
&=T_1(O^T \be_{i_1}, \dots , O^T \be_{i_t} ) \\
& = \prod_{j=1}^t (O^T \be_{i_j})_{1}  \\
& = \prod_{j=1}^t (O \be_1)_{i_j}\\
& = \prod_{j=1}^t v_{i_j} \;.
\end{align*}
Thus we have $O^{\otimes t} T_1 = v^{\otimes t}$,
the tensor with entries  $\prod_{j=1}^t v_{i_j}$, and so
\begin{align*}
\E_{X \sim \p}[He_t(v \cdot X)]/\sqrt{t!}
&= \E_{X \sim \p}[h_{v^{\otimes t}}(X)] \\
&= \sum_\bi B_\bi \prod_{j=1}^t v_{i_j} =  B(v, \ldots,v) \tag*{(by (ii))} \;.
\end{align*}
This completes the proof.
\end{proof}

We write $P_{t'}$ or $G_{t'}$ for the rank-$t'$ tensor
with entries  $\sqrt{t!} \E[h_{c(i)}(X)]$,
where $X$ is distributed according to $G'|F'$ or $\widetilde G$ respectively.
We know that $\|\wt P_{t'}\|_F \geq  \eps \Omega(\log(1/\eps)^{t'/2})$.

\begin{lemma}
When $\|\wt P_{t'}\|_F \geq  \eps  \Omega(\log(1/\eps)^{t'/2})$,
we have $\left|\E_{X \sim G'}[h_A(X)] \right| \geq \eps \cdot \Omega(\log(1/\eps))^{t/2}$.
\end{lemma}
\begin{proof}
The assumption on the SQ errors imply that the corresponding
entries of $P_{t'}$ and $\wt P_{t'}$ are within $\eps/n^{t'/2}$.
It follows that $$\|P_{t'} - \wt P_{t'}\|_F \leq \sqrt{t'!} \eps \;.$$
Using Lemma~\ref{lem:tensor-stuff} (iii), we have
\begin{align*}
\E_{X \sim G'}[h_A(X)]
& =  \sum_{i \in [n]^{t'}} A_i (P_{t'})_i \\
&= 1/\|\wt P_t\|_F \cdot \sum_{i \in [n]^{t'}} (\wt P_{t'})_i (P_{t'})_i \\
& = 1/\|\wt P_t\|_F \cdot  \left(\|\wt P_{t'}\|_F^2 + \sum_{i \in [n]^{t'}} (\wt P_{t'})_i (P_{t'} - \wt P_{t'})_i  \right) \\
&\geq   \|\wt P_{t'}\|_F  - \|P_{t'} - \wt P_{t'}\|_F \geq \|\wt P_{t'}\|_F  - \sqrt{t'!} \eps \\
& \geq \eps  \Omega(\log(1/\eps)^{t'/2})) - \eps \sqrt{\log(1/\eps)}^{t'} \\
& \geq \eps  \Omega(\log(1/\eps)^{t'/2}) \;.
\end{align*}
\end{proof}

\begin{lemma}
$\E_{X \sim \widetilde G}[h_A(X)]=O(\eps \sqrt{\log(1/\eps)})^t$
and $\E_{X \sim \widetilde G}[h_A(X)^2]= O(1)$.
\end{lemma}
\begin{proof}
Since for all $\ba \in \Z_{>0}^k$,
$\E_{X \sim N(0,I)}[He_\ba(X)]=0$,
we have $\E_{X \sim N(0,I)}[h_A(X]=0$.
By Lemma~\ref{lem:tensor-stuff} (i),
$\E_{X \sim N(0,I)}[h_A(X)^2]=\|A\|_F^2=1$.

We need to take these expectations under $\widetilde G = N(\mu,I)$ instead of $N(0,I)$.
Consider a rotation given by an  orthogonal matrix $O$
that maps $\|\mu\|_2 \be_1$ to $\mu$.
By Lemma \ref{lem:tensor-stuff} (i), there is a symmetric rank-$t$ tensor $B$
with $\|B\|_F=1$ such that $h_A( O X)=h_B(X)$.
Now we have that
$$\E_{X \sim G'}[h_A(X)] = \E_{X \sim N(\|\mu\|_2 \be_1,I)}[h_A(O X)]= \E_{X \sim N(\|\mu\|_2 \be_1,I)}[h_B( X)] \;,$$
and similarly for $h_A(X)^2$.

For $\ba \in \Z_{ \geq 0}^n$ with $\ba \neq 0$,
writing $\ba_{-1} \in \Z_{\geq 0}^{n-1}$ for the vector dropping the first coordinate of $\ba$,
we have
\begin{align*}
\E_{X \sim N(\|\mu\|_2 e_1,I)}[He_\ba(X)]
& = \E_{X \sim N(0, I)}[He_{\ba_{-1}}(X)] \E_{X \sim N(\|\mu\|_2,1)}[He_{a_1}(X)]  \\
& = \delta_{\ba_{-1},0} \E_{X \sim N(0,1)}[He_{a_1}(X+\|\mu\|_2)] \;.
\end{align*}
Note that there is only one index $i$ such that $c(i)$ is zero
in all except the first coordinate, and so we have
$$\E_{X \sim \widetilde G}[h_A(X)]=\E_{X \sim N(\|\mu\|_2 e_1,I)}[h_B(X)] = B_{1,\ldots,1} \E_{X \sim N(0,1)}[He_t(X+\|\mu\|_2)/\sqrt{t!}] \;.$$
Since $\|B\|_F=1$, we have
$$\left|\E_{X \sim \widetilde G}[h_B(X)] \right|
\leq |\E_{X \sim N(0,1)}[He_t(X+\|\mu\|_2)/\sqrt{t!}]|.$$
By standard results, we have that
$\frac{dHe_i}{dx}(x)=i He_{i-1}(x)$, and so by Taylor's theorem
we have $He_i(x + \|\mu\|_2) = \sum_{j=0}^i {i \choose j} \|\mu\|_2^j He_{i-j}(x)$.
Thus,
\begin{align*}
|\E_{X \sim N(0,1)}[He_t(X+\|\mu\|_2)]|
& = \left|\sum_{i=0}^t {t \choose i} \|\mu\|_2^i |\E_{X \sim N(0,1)}[He_{t-i}(X)] \right| \\
& = \|\mu\|_2^t \;.
\end{align*}
This gives
$$\left|\E_{X \sim \widetilde G}[h_A(X)]\right| \leq \|\mu\|_2^t/\sqrt{t!} = O(\eps \sqrt{\log(1/\eps)})^t/\sqrt{t!} \;,$$
as required.

Similarly, we have, for all $\ba,\bb \in \Z_{>0}^n$,
\begin{align*}
\E_{X \sim N(\|\mu\|_2 e_1,I)}[He_\ba(X)He_\bb(X)]
& = \E_{X \sim N(0,I)}[He_{\ba_{-1}}(X)He_{\bb_{-1}}(X)] \E_{X \sim N(\|\mu\|_2,1)}[He_{a_1}(X)He_{b_1}(X)]  \\
& = \delta_{\ba_{-1} \bb_{-1}} \cdot \E_{X \sim N(0,1)}[He_{a_1}(X+\|\mu\|_2)He_{b_1}(X+\|\mu\|_2)]  \;.
\end{align*}
When $\|\ba\|_1=\|\bb\|_1=t$, if $\ba_{-1} =\bb_{-1}$, then $\ba=\bb$
(since $a_1=b_1=t-\|\ba\|_1$).
For $1 \leq j \leq t$, we have:
\begin{align*}
\E_{X \sim N(0,1)}[He_j(X+\|\mu\|_2)^2]
& = \E_{X \sim N(0,1)}\left[\left( \sum_{i=0}^j {j \choose i} \|\mu\|_2^i He_{j-i}(x) \right)^2 \right ] \\
&= \sum_{i=0}^j {j \choose i}^2 \|\mu\|_2^{2i} |\E_{X \sim N(0,1)}[He_{j-i}(X)^2] \\
& =   \sum_{i=0}^j {j \choose i}^2 \|\mu\|_2^{2i}  i! \\
& = \sum_{i=0}^j \|\mu\|_2^{2i} (j!/(j-i)!)^2/i! \\
& \leq \sum_{i=0}^j \|\mu\|_2^{2i} j^2i /i!\\
& \leq 2 \sum_{i=0}^j 2^{-2i}  \tag*{(since $\|\mu\|_2 \leq 1/2k \leq 1/2j$)} \\
& \leq 3 \;.
\end{align*}
Putting these together, for $\ba,\bb$ with $\|\ba\|_1=\|\bb\|_1=t$, we have
$$\left|\E_{X \sim N(\|\mu\|_2 e_1,I)}[He_\ba(X)He_\bb(X)]\right| \leq 3 \delta_{\ba \bb} \;.$$
The sum of squares of coefficients of all
$He_\ba(x)$ in $h_B(X)$ is $\E_{X \sim N(0,1)}[h_B(X)^2]=\|B\|_F=1$,
and so we have that $\E_{X \sim N(\|\mu\|_2 e_1,I)}[h_B(X)^2]| \leq 3$.
Finally, recall that $\E_{X \sim G}[h(A)^2] = \E_{X \sim N(\|\mu\|_2 e_1,I)}[h_B(X)^2]|$,
and so this is $O(1)$, as required.
\end{proof}

Now consider the equation
$$\E_{X \sim G'}[h_A(X)] = w_{\widetilde G} \E_{X \sim \widetilde G}[h_A(X)] +
w_E \E_{X \sim E}[h_A(X)] - w_L \E_{X \sim L}[h_A(X)].$$
We know that the LHS is $\Omega(\eps \log(1/\eps)^{t/2})$
and that the first term on the RHS is smaller.
Therefore, one of the last two terms is small.
Since $w_L L \leq \widetilde G$, we will use standard concentration inequalities
to show that $w_L \E_{X \sim L}[h_A(X)]$ is
$O(\eps \log(1/\eps)^{t/2})$.
If we cannot find a filter, then $w_E E \leq G'$
must satisfy similar concentration inequalities,
which would imply that  $w_E \E_{X \sim E}[h_A(X)]$ is $O(\eps \log(1/\eps)^{t/2})$.
Since some term on the RHS must be bigger than this, we can find a filter.

\begin{lemma} \label{lem:hypercontractivity-conc}
For $X \sim N(0,I)$, if $p(x)$ is a degree-$d$ polynomial with $\E[p(X)^2] \leq 1$, we have that
$$\Pr\left[|p(X)| \geq T + \E[p(X)]\right] \leq \exp(2-\Omega(T)^{2/d})) \;.$$
\end{lemma}

\begin{lemma}\label{lossBoundLem}
We have that
$w_L |\E_{X \sim L}[h_A(X)]| \leq \eps \cdot O(\log (1/\eps))^{t'/2})$.
\end{lemma}
\begin{proof}
We start with the following claim:
\begin{claim} \label{clm:damn-integral}
For any $R > 0$, $d \in \mathbb{Z}_+$, $\eps > 0$,
and $\exp(-(a/R)^{2/d}) = \eps$, we have
$$\int_a^\infty \exp(-(T/R)^{2/d}) T dT \leq (d^2/2) \eps (d+\ln(1/\eps))^{d-1} \;.$$
\end{claim}
\begin{proof}
Note that $(a/R)^{2/d}=\ln(1/\eps)$.
First, we change variables to $x= (T/R)^{2/d}$ to obtain
\begin{align*}
\int_{\ln(1/\eps)}^\infty \exp(-(T/R)^{2/d}) T dT& = \int_{\ln(1/\eps)}^\infty \exp(-x) x^{d/2} \frac{dT}{dx} dx \\
&= \int_{\ln(1/\eps)}^\infty \exp(-x) x^{d/2} \cdot (Rd/2)x^{d/2-1} dx \\
& = (Rd/2) \int_{\ln(1/\eps)}^\infty \exp(-x) x^{d-1} dx \;.
\end{align*}
We can now integrate by parts
$$\int_{\ln(1/\eps)}^\infty \exp(-x) x^{d-1} dx = \eps \ln(1/\eps)^{d-1} + (d-1) \int_{\ln(1/\eps)}^\infty \exp(-x) x^{d-2} dx \;.$$
By a simple induction, we have
$$\int_{\ln(1/\eps)}^\infty \exp(-x) x^{d-1} dx = \eps \sum_{j=0}^{d-1} d!/(d-j)!  \ln(1/\eps)^{(d-j-1)} \;.$$
Now we have
\begin{align*}
\int_a^\infty \exp(-(T/R)^{2/d}) T dt &= (d/2) \eps \sum_{j=0}^{d-1} d!/(d-j)!  \ln(1/\eps)^{(d-j-1)}\\
& \leq (d/2) \exp(-(a/R)^{2/d}) \sum_{j=0}^{d-1} d^j  \ln(1/\eps)^{(d-j-1)} \\
& \leq (d^2/2) \eps (d+\ln(1/\eps))^{d-1} \;.
\end{align*}
\end{proof}
Let $\mu_h = \E_{X \sim \widetilde G}[h_A(X)]$.
We have the following sequence of inequalities:
\begin{align*}
w_L \E_{X \sim L|}[h_A(X)^2] & = \int_0^\infty T w_L \Pr_{X \sim L}[|h_A(X)| > T] dT \\
& \leq \int_0^\infty T \min \{ w_L, \Pr_{X \sim \widetilde G}[|h_A(X)| > T] \} dT \\
& \leq \int_0^\infty T \min \{ w_L, \exp(2-\Omega(T -\mu_h)^{2/t'})) \} dT \\
& \leq \int_0^\infty T \min \{ w_L, \exp(2-((T -\mu_h)/R)^{2/t'})) \} dT \tag*{(for some $R > 0$)}\\
& = \int_{-\mu_h}^{\infty} (T + \mu_h)  \min \{ w_L, \exp(2-(T/R)^{2/t'})) \} dT \\
& = \int_{-\mu_h}^{c} (T+\mu_h) w_L dT + \int_{c}^\infty \exp(2-(T/R)^{2/t'}) (T+\mu_h) dT  \tag*{(where $c= (R\ln 1/w_L)^{t'/2}$)}\\
& \leq w_L (c+\mu_h)^2/2 + e^2 (c+\mu_h)/c \cdot \int_{c}^\infty \exp(-(T/R)^{2/t'}) T dT  \\
& \leq w_L O(\ln(1/w_L))^{t'} + O(t'^2 w_L (t'+\ln(1/w_L))^{t'-1}) \\
&  \leq \eps \cdot O(\ln (1/\eps))^{t'} \;,
\end{align*}
where the last line follows from $t' \leq k = O(\sqrt{\log(1/\eps)}) \leq O(\log(1/\eps))$.
Then, by an application of the Cauchy-Schwarz inequality, we have that
$$w_L |\E_{X \sim L}[h_A(X)]| \leq w_L \E_{X \sim L}[|h_A(X)|] \leq
\sqrt{w_L^2 \E_{X \sim L}[h_A(X)]^2} \leq \eps \cdot O(\log (1/\eps))^{t'/2}).$$
This completes the proof of the lemma.
\end{proof}

\begin{lemma}
If  $\Pr_{X \sim G'|F'}[|h_A(X)| \geq T + 1] \leq O(\exp(2-\Omega(T)^{2/d})) + \eps/(2n)^{2t'})$, for all integers $T$,
then $w_E |\E_{X \sim E}[h_A(X)]| \leq O(\eps \ln (1/\eps)^{t'/2}).$
\end{lemma}
\begin{proof}
Since $F'$ includes the filter $M$, we have that the support of $G'|F'$
and the support of $E$ includes only $x$ with $\|x\|_2 \leq \sqrt{2n \ln(1/\eps)}$:
\begin{claim}
When $\|x\|_2 \leq \sqrt{2n \ln(1/\eps)}$, then
$|h_A(x)| \leq (2n \sqrt{\ln(1/\eps)})^t$.
\end{claim}
\begin{proof}
Note that $t' \leq k \leq \sqrt{2n \log(1/\eps)}$.
Using the explicit formula for the coefficient $He_i(x)$,
we can show that for $|x| \leq \sqrt{2n \ln(1/\eps)}$,
with $k \geq i$, the $He_i(x)$ is dominated by its leading coefficient:
\begin{align*}
|He_i(x)|
& = \left| \sum_{j=0}^{\lfloor i/2 \rfloor} i! (-1)^j x^{i-2j} /j! (i-2j)! 2^j \right| \\
& \leq   \sum_{j=0}^{\lfloor i/2 \rfloor} j! (3/2)^j |x|^{i-2j} \\
& \leq \sum_{j=0}^{\lfloor i/2 \rfloor}  j! (3/2)^j (\sqrt{2n \ln(1/\eps)})^{i-2j}   \\
& \leq \sum_{j=0}^{\lfloor i/2 \rfloor}   (2n \log(1/\eps))^{(i-j)/2}   \\
& \leq 2 \cdot  (n \log(1/\eps))^{i/2} \;.\\
\end{align*}
Therefore, for any $a \in \Z_{\geq 0}^n$ with $\sum_i a_i = t$,
and $\|x\|_2 \leq \sqrt{2n \log(1/\eps)}$, we have that
\begin{align*}
|h_a(x)| &= \prod_{i=1}^n | He_{a_i}(x)/\sqrt{a_i!} | \\
& \leq  \prod_{i=1}^n 2 \sqrt{n \log(1/\eps)}^{a_i} \\
& \leq (2\sqrt{n \log(1/\eps)})^t \;.
\end{align*}
Since $\|A\|_F=1$ and $A$ has $n^{t'}$ entries,
the $L_1$-norm of the entries is at most $n^{t/2}$.
Thus, we have that
$|h_A(x)| \leq  n^{t/2} \cdot (2\sqrt{n \log(1/\eps)})^t = (2n \sqrt{\log(1/\eps)})^t.$
\end{proof}
We note that since
$$\Pr_{X \sim G'|F'}[|h_A(X)| \geq T + 1] \leq O(\exp(2-\Omega(T)^{2/d})) + \eps/(2n)^{2t'}) \;,$$ for integers $T$, that
$$\Pr_{X \sim G'|F'}[|h_A(X)| \geq T+2] \leq O(\exp(2-\Omega(T)^{2/d})) + \eps/(2n)^{2t'}) \;,$$ for all $T$.

Similarly to the proof of Lemma \ref{lossBoundLem}, we obtain:
\begin{align*}
w_E \E_{X \sim E}[h_A(X)^2] & = \int_0^\infty T w_E \Pr_{X \sim E}[|h_A(X)| > T] dT \\
& = \int_0^{(2n \sqrt{\log(1/\eps)})^t} T w_E \Pr_{X \sim E}[|h_A(X)| > T] dT \\
& \leq \int_0^{(2n \sqrt{\log(1/\eps)})^t} T \min \{ w_E, \Pr_{X \sim G'}[|h_A(X)| > T] \} dT \\
& \leq \int_0^{(2n \sqrt{\log(1/\eps)})^t} T \min \{ w_E,  O(\exp(2-\Omega(T-2)^{2/t}) + \eps/Cn^{2t'}) \} dT \\
& \leq \int_{-2}^{(2n \sqrt{\log(1/\eps)})^t-1} (T+2) \min \{ w_E,  O(\exp(2-(T/R)^{2/t}) + \eps/Cn^{2t'}) \} dT \\
& \leq \int_0^{(2n \sqrt{\log(1/\eps)})^t} O(T \eps/(2n)^{2t'}) dT + \int_{-2}^{c} (T+2) w_E dT \\
& + (c+1)/c \cdot \int_{c}^\infty O(\exp(-(T/R)^{2/t})) T dT  \tag*{(where $c= (R\ln 1/w_E)^{t/2}$)}\\
& = (2n \sqrt{\log(1/\eps)})^{2t} \cdot \eps/(2n)^{2t'} + O(w_E c^2/2) +  O(t'^2 w_E (t' + \ln(1/w_E))^{t'-1}) \\
& \leq \eps \cdot (4\log(1/\eps)^t + w_E \cdot O(\ln 1/w_E)^t \\
&  \leq \eps \cdot O(\log (1/\eps))^t \;.
\end{align*}

Then, by the Cauchy-Schwarz inequality, we conclude that
$$w_E |\E_{X \sim E}[h_A(X)]| \leq w_E \E_{X \sim E}[|h_A(X)|] \leq \sqrt{w_E^2 \E_{X \sim E}[h_A(X)]^2} \leq \eps \cdot O(\log (1/\eps))^{t/2} \;.$$
This completes the proof.
\end{proof}

As an immediate consequence, we obtain:
\begin{corollary}\label{TexistCor}
There is an integer $0 \leq T \leq O(n \sqrt{\log(1/\eps)})^t$ such that
$$\Pr_{X \sim G'|F'}[|h_A(X)| \geq T + 1] \geq 3\exp(2-\Omega(T)^{2/d})) + 2\eps/Cn^{2t'} \;.$$
\end{corollary}

We can now prove the following crucial lemma:
\begin{lemma}
The algorithm finds a $T$ with
$$\Pr_{X \sim G'|F'}[|h_A(X)| \geq T + 1] \leq 3\exp(2-\Omega(T)^{2/d})) + \eps/Cn^{2t'}\;.$$
\end{lemma}
\begin{proof}
By Corollary \ref{TexistCor}, such a $T$ exists, and therefore our algorithm will find one after enumerating $O(n \sqrt{\log(1/\eps)})^t$ possibilities.
\end{proof}

Let $F$ be the event that the new filter accepts.
In the next iterations, we will use $G'|F' \cap F$  instead of $G'|F'$.
We need to show that the parameters $w_{\widetilde G}$,
$w_E$ and $w_L$ improve in such a way
that we only need a bounded number of iterations:
\begin{claim}
We can write $G'|F' \cap F=w'_{\widetilde G} \widetilde G + w'_E E' - w'_L L'$,
where $L'$ and $E'$ have disjoint supports $w'_E, w'_L >0$
and $w'_E + w'_L \leq w_E+w_L -   \eps/Cn^{2t'}$.
The probability that the filter rejects is at most $O(w_E+w_L-w'_E-w'_L)$.
\end{claim}
\begin{proof}

This proof is very similar to that of Claim 26 from~\cite{DiakonikolasKS16b}.
Let $\neg F$ be the event that the filter rejects,
i.e., that $|h_A(X)| \geq T + 1$.
We have that  $\Pr_{G'|F'}[\neg F] \geq 3\exp(2-\Omega(T)^{2/d})) + \eps/Cn^{2t'}$.
On the other hand, by the concentration inequality,
$\Pr_{\widetilde G}[\neg F] \leq \exp(2-\Omega(T)^{2/d}))$.
Thus, we have that
$$\Pr_{G'|F'}[\neg F] \geq 3\Pr_{\widetilde G}[\neg F]  + \eps/Cn^{2t'} \;.$$
However, the defining relation between $G'|F'$ and $\widetilde G$, $E$ and $L$
yields for the event $\neg F$ that
$$\Pr_{G'|F'}[\neg F] \geq w_{\widetilde G} \Pr_{\widetilde G}[\neg F] + w_E\Pr_E[\neg F] - w_E\Pr_E[\neg F] \;.$$
Since
$$\Pr_{G'|F'}[\neg F] \leq w_{\widetilde G} \Pr_{G'}[\neg F] + w_E \Pr_E[\neg F] \;,$$
and $w_{\widetilde G} \leq 1+O(\eps)$, we must have
$$\Pr_{G'|F'}[\neg F] \leq  (2+O(\eps)) w_E \Pr_E[\neg F] \;,$$
and $$\Pr_{G'}[\neg F] \leq (1/3 + O(\eps)) w_E \Pr_E[\neg F].$$
Then, we get that
\begin{align*}
& \left(1-\Pr_{G'|F'}[\neg F]\right) (G'|F')(x)  = \left(1-\Pr_{G'|F'}[\neg F]\right)(G'|F' \cap \neg F)(x) \\
& = w_{\widetilde G} \widetilde G(x) + w_E\left(1-\Pr_E[\neg F]\right) E(x) + w_L\left(1-\Pr_L[\neg F]\right) L(x) - w_{\widetilde G} \Pr_{G'}[\neg F] \widetilde G(x) \;.
\end{align*}
Thus, we have
\begin{align*}
w'_L & = \frac{w_L\left(1-\Pr_L[\neg F]\right) - w_{\widetilde G} \Pr_{G'}[\neg F]}{1-\Pr_{\wt P}[\neg F]} \\
& \leq w_L + (1 + O(\eps)) \Pr_{G'}[\neg F] + O\left(\eps\Pr_{\wt P}[\neg F]\right) \\
& \leq w_L + \left(1/3 + O(\eps)\right)w_E \Pr_E[\neg F] + O\left(\eps\Pr_{G'|F'}[\neg F]\right) \;.
\end{align*}
Also we have
 \begin{align*}
w'_E & = \frac{w_E\left(1-\Pr_E[\neg F]\right)}{1-\Pr_{G'|F'}[\neg F]} \\
& \leq w_E\left(1-\Pr_E[\neg F]\right) + O\left(\eps\Pr_{G'|F'}[\neg F]\right) \;.
\end{align*}
Thus,
\begin{align*}
w_L + w_E - w'_L - w'_E
& \geq \left(2/3 - O(\eps)\right)w_E \Pr_E[\neg F] - O\left(\eps\Pr_{\wt P}[\neg F]\right) \\
& \geq \left(1/3 - O(\eps)\right) \Pr_{G'|F'}[\neg F] \geq \eps/Cn^{2t'} \;.
\end{align*}
Note that the penultimate inequality also gives that
$$\Pr_{G'|F'}[\neg F] \leq \left(3+O(\eps)\right)(w_L + w_E - w'_L - w'_E).$$
This completes the proof.
\end{proof}

Proposition \ref{prop:filter-works} now follows using induction on the iterations.

\subsubsection{Completing the Proof of Correctness}

\begin{lemma}
$\dim(V) \leq O(\log(1/\eps))^{k}$.
\end{lemma}
\begin{proof}
After leaving the filter loop,  for all $1 \leq t \leq k$,
we have that $\|\wt P_t\|_F \leq \eps O(\log(1/\eps))^{t/2}$.
$M(\wt P_t)$ has the same Frobenius norm,
and thus the $L_2$-norm of its singular values,
when considered as a matrix. Thus, there are at most
$O(\log(1/\eps))^{t/2}$ singular values bigger than $\eps$.
So, we have that $\dim(V_t) = O(\log(1/\eps))^{t/2}$,
and so $\dim (V) \leq \sum_{t=1}^k \dim V_t \leq O(\log(1/\eps))^{k}$.
\end{proof}

Let $\mu_V$ be the projection of $\mu$ onto the subspace $V$.
Now we can show using our moment matching lemma
that it suffices to approximate $\mu_V$.

\begin{lemma}
We have that
$\|\mu_V - \mu\|_2 \leq O(\eps)$.
\end{lemma}
\begin{proof}
Let $v= (\mu_V - \mu)/\|\mu_V - \mu\|_2$.
Note that $v$ is a unit vector perpendicular to $V$ and we need to show that
$v^T \mu \leq O(\eps)$.
We will apply Lemma \ref{lem:moment-matching} to $G''$,
the projection $G'$ conditioned on the event that all the filters we produced accept $F'$,
onto $v$.
Note that $G'|F'$ has  $\dtv(\widetilde G,G'|F') \leq O(\eps)$,
and so we have that $\dtv(G'', N(0, v^T \mu)) \leq O(\eps)$.
We can bound the expectation of the Hermite polynomials as follows,
for $1 \leq t \leq k$:
\begin{align*}
|\E_{X \sim G''}[He_t(X)/\sqrt{t!}]| & = |\E_{X \sim G'}[He_t(v \cdot X)/\sqrt{t!}] | = |P^t(v,\dots ,v)| \\
& = |(v^{\otimes t-1})^T M(P^t) v | \leq \| v^{\otimes t-1}\|_2 \|M(P^t) v \|_2 \\
& \leq 1 \cdot O(\eps) \;.
\end{align*}
On the other hand, we have
$\E_{X \sim N(0,1)}[He_t(X)/\sqrt{t!}]=0$,
for $1 \leq t \leq k$.
For $t=0$, $He_t(X)/\sqrt{t!}=1$,
which has expectation $1$ under both $G''$ and $N(0,1)$.
We want to consider the difference in the expectations of $X^t$, for $1 \leq t \leq k$.
We can write $x^t$ as a linear combination of Hermite polynomials,
$x^t=\sum_{i=0}^t a_i He_i(x)/\sqrt{i}$.
Using the orthonormality of these polynomials,
we have that $\E_{X \sim N(0,1)}[(X^t)^2]=\sum_i a_i^2$.
On the other hand, by standard results, $\E_{X \sim N(0,1)}[(X^t)^2]=2^t t!$.
Thus, we have:
\begin{align*}
|\E_{X \sim G''}[X^t]- \E_{X \sim N(0,1)}[X^t]|
& = \left| \sum_{i=0}^t a_i \left(\E_{X \sim G''}[He_i(X)/\sqrt{i!}]- \E_{X \sim N(0,1)}[He_i(X)/\sqrt{i!}]\right) \right|\\
& \leq O(\eps) \cdot \sum_{i=0}^t a_i \\
& \leq O(\eps) \cdot \sqrt{t} \cdot \sqrt{2^t t!} \\
\end{align*}
Note that for $t \geq 10$,
$\sqrt{t 2^t t!} \leq (t-1)!/t$.
Thus, there is a constant $c>0$ such that this
$O(\eps) \cdot \sqrt{t} \cdot \sqrt{2^t t!}$
is smaller than $ (t-1)! c^t \eps /t$, for all $1 \leq t \leq k$.
Now we can apply Lemma \ref{lem:moment-matching}
with $\delta = c\eps$ and obtain that
$|v^T \mu| \leq O(\delta)=O(\eps)$.
We need to set $k$ to be a sufficiently high multiple of
$\eps \sqrt{\ln(1/\eps)}$ to make this work.
\end{proof}

It remains to analyze the rest of the algorithm and show that $\widetilde \mu_V$
it produces is close to $\mu_V$.

\begin{lemma}
We can construct  a set $S \subset V$ of unit vectors of size $\dim (V)^{O(\dim(V))}$
such that for any unit vector $v \in V$,
there is a $v' \in S$ with $\|v-v'\|_2 \leq 1/2$,
in time $\dim (V)^{O(\dim (V))}$.
\end{lemma}
\begin{proof} Let $\ell = \dim (V)$.
We will construct such a cover for $\R^{\ell}$
and translate that to $V$ by using the orthonormal basis for $V$
given by the right singular vectors of $M(P_t)$
with singular values bigger than $\eps$.
We can divide the cube $[-1,1]^{\ell}$ into $\ell^{O(\ell)}$ cubes
of side length $1/(2\sqrt{\ell})$.
For each cube, we check if it has a corner with $L_2$-norm $\geq 1$
and a corner with $L_2$-norm $\leq 1$.
If it does not, it does not contain any unit vectors so we can ignore it.
If the cube does contain any unit vectors, then if its center is $v$,
we add the normalized vector $v/\|v\|_2$ to $S$.
Since there is a unit vector $v'$ in the cube, and all vectors
in the cube have $\|v'-v\|_2 \leq \sqrt{\ell} \|v'-v\|_\infty \leq 1/4$,
we have $|\|v\|_2 -1| \leq  1/4$,
and so $\|v - v/\|v\|_2\|_2 \leq 1/4$.
Thus, for any unit vector $v'$ in this cube,
we have $\|v'-v/\|v\|_2\|_2 \leq 1/4 + 1/4 \leq 1/2$.
Since every unit $v'$ is in some cube whose normalized center we added to $S$,
we are done.
\end{proof}

Firstly, we note that in order to approximate $v^T \mu$,
it is sufficient to find an $x$ with $\Pr_{X \sim G'}[v^T X \geq x] = 1/2 + O(\eps)$:
\begin{lemma} \label{lem:close-cdf-implies-close}
For all $x \in \R$ with $|\Pr_{X \sim G'}[v^T X \geq x] -1/2| \leq 3\eps$,
we have that $|v^T \mu - x| \leq O(\eps)$ . \end{lemma}
\begin{proof}
First note that the pdf of $\widetilde G$ projected onto $v$
has $G(x-v^T \mu) \geq 1/2$ for all $x$ with $|x-v^T\mu| \leq O(\eps)$.
Supposing that $|x-v^T\mu| \geq 8\eps$,
we have that  $|\Pr_{X \sim \widetilde G}[v^T X \geq x] -1/2| \geq 4\eps$,
and so $|\Pr_{X \sim G'}[v^T X \geq x] -1/2| \geq 4\eps - \dtv(G',\widetilde G) \geq 3\eps$.
\end{proof}

To show that we can find such a point by bisection,
we need to show that there is an interval of such points
of reasonable length where we are looking for them:
\begin{lemma} \label{lem:decebt-interval}
Given a unit vector $v \in \R^n$, there is an interval $[a,b]$ such that
\begin{itemize}
\item for all $x \in [a,b]$, we have that $|\Pr_{X \sim  G'}[v^T X \geq x] -1/2| \leq 2\eps$,
\item $b-a=\Theta(\eps),$
\item and $|a|,|b| \leq O(\eps \sqrt{\log 1/\eps}).$
\end{itemize}
\end{lemma}
\begin{proof}

We can take $[a,b]$ to be the set of $x$
with $|\Pr_{X \sim \widetilde G}[v^T X \geq x] -1/2| \leq \eps$.
This is an interval since $\Pr_{X \sim \widetilde G}[v^T X \geq x]$ is monotone.
All $x$ in it have $|\Pr_{X \sim  G'}[v^T X \geq x] -1/2| \leq \eps + \dtv(G', \widetilde G) \leq 2\eps$.
Thus, by the previous lemma,
$|b-v^T \mu|,|a-v^T \mu| \leq O(\eps)$ and thus $|b-a| \leq O(\eps)$,
and since $|v^T \mu| \leq O(\eps \sqrt{\log 1/\eps})$,
we have that $|a|,|b| \leq O(\eps \sqrt{\log 1/\eps}).$
\end{proof}

Thus, we obtain:
\begin{lemma}
Given a unit vector $v \in \R^n$, we can find an $m_v$
with $|v^T \mu - m_v| \leq O(\eps)$
using $O(\log \log 1/\eps)$ statistical queries of precision $\eps/2$.
\end{lemma}
\begin{proof}
We use bisection to find a point where our SQ approximation
$\widetilde p$ to $\Pr_{X \sim G'}[v^T X \geq x]$ is within $5\eps/2$ of $1/2$.
If we find such a point, it has $|v^T \mu - m_v| \leq O(\eps)$,
by Lemma \ref{lem:close-cdf-implies-close}.
Lemma \ref{lem:decebt-interval} yields that there is an interval $[a,b]$ of length $O(\eps)$
containing such points in the interval $|x| \leq O(\eps \sqrt{\log(1/\eps)})$.
Indeed, if our test point $x$ has $\widetilde p > 1/2+ 5\eps/2$,
then $x > b$ and if $\widetilde p < 1/2 - 5\eps/2$, then $x < a$.
Thus, $[a,b]$ remains a subinterval of the interval we are considering.
\end{proof}

We now have that $\mu_v$ is a feasible point of the LP considered in Step \ref{step:lp}.
The following lemma completes the proof:

\begin{lemma}
Any  feasible point of the LP considered in Step \ref{step:lp},
$\widetilde \mu_V$ has $\|\mu_V - \widetilde \mu_V\|_2 \leq O(\eps)$.
\end{lemma}
\begin{proof}
Consider the vector $v = (\mu_V - \widetilde \mu_V)/\|\mu_V - \widetilde \mu_V\|_2$.
Note that $v$ is in $V$, since $\mu_C,\widetilde \mu_V$ are.
Since $v$ is a unit vector in $V$,
there is a $v' \in S$ with $\|v-v'\|_2 \leq 1/2$.
Since $\widetilde \mu_V$ is a solution to the LP,
$v'^T (\mu_V - \widetilde \mu_V) \leq O(\eps)$.
Thus, we have that
\begin{align*}
\|\mu_V - \widetilde \mu_V\|_2 & = v^T (\mu_V - \widetilde \mu_V) \\
							& = v'^T (\mu_V - \widetilde \mu_V) + (v-v')^T (\mu_V - \widetilde \mu_V) \\
							& \leq O(\eps) + \|\mu_V - \widetilde \mu_V\|_2/2 \;.
							\end{align*}
Therefore, $\|\mu_V - \widetilde \mu_V\|_2 \leq O(\eps)$, as required.
 \end{proof}

 \begin{proof}[Proof of Theorem \ref{thm:upper-bound-learning}]
Since the LP has a feasible point, we can find such a point
$\widetilde \mu_V$ that has $\|\mu_V - \widetilde \mu_V\|_2 \leq O(\eps)$.
By the previous lemma,
we have that $\|\mu_V - \mu\|_2 \leq O(\eps)$.
Thus, the algorithm is correct.
All statistical queries are of the claimed precision.
We need to get bounds on the running time and number of statistical queries.

Step \ref{step:old-alg} that uses the algorithm from~\cite{DiakonikolasKKLMS16},
takes $\poly(n/\eps)$ time and statistical queries.
Finding $\wt P_t$, for each $1 \leq t \leq k$, takes $n^t$ statistical queries, giving $n^O(k)$ time total.
There are at most $O(n^{2k})$ iterations of the loop.
Each iteration takes $\poly(n/\eps)$ time and statistical queries to find $T$,
and $n^t$ statistical queries to recompute $\wt P_t$.
It suffices to compute the SVD to within Frobenius norm $1/\poly(n/\eps)$,
which takes time $\poly(n/\eps)$.
The set $S$ has size $\dim(V)^{O(\dim (V))} = \log(1/\eps)^{k O(\log(1/\eps))^{k}} = 2^{\log(1/\eps)^{O(k)}}$.
Computing it takes time $2^{\log(1/\eps)^{O(k)}}$.
Approximating the medians takes $2^{\log(1/\eps)^{O(k)}}$ statistical queries and time.
The LP has $2^{\log(1/\eps)^{O(k)}}$ constraints and $\log(1/\eps)^{O(k)}$ variables.
The size of the LP is $2^{\log(1/\eps)^{O(k)}}$ bits,
and so with a polynomial time LP solver, we can get $2^{\log(1/\eps)^{O(k)}}$ time.

We thus have that the total time and statistical queries
are both at most
$$n^{O(k)}\poly(1/\eps) + 2^{\log(1/\eps)^{O(k)}} = n^{O(\sqrt{\log(1/\eps)})} + 2^{\log(1/\eps)^{O(\sqrt{\log(1/\eps)})}}.$$
 \end{proof}

\bibliographystyle{alpha}

\bibliography{allrefs}


\appendix
\section*{Appendix}

\section{Sample Complexity Upper Bound for Learning GMMs} \label{sec:sample-gmm}

In this section, we show that learning a $k$-mixture of $n$-dimensional
Gaussians to variation distance error $\epsilon$ is easy information theoretically.
In particular, we have:

\begin{theorem}\label{GMMLearnThrm}
Given $\eps>0$ and positive integers $k$ and $n$, there exists
an algorithm that, given a probability distribution $\p$ which is a $k$-mixture of
$n$-dimensional Gaussians, takes $O(n^2k^3\log^2(k)/\eps^5)$ samples from $\p$
and with probability at least $2/3$ returns a distribution $\q$ with $\dtv(\p,\q)<\eps$.
\end{theorem}

Note that the algorithm given in Theorem \ref{GMMLearnThrm} will not be computationally efficient.

The basic idea of Theorem \ref{GMMLearnThrm} will be to make many guesses as to the mixture, at least one of which is close,
and then run a tournament to find the true answer. We approximate the mixture by first guessing approximations
to the weights and then approximating each individual Gaussian.
If we had polynomially many samples from a single part of the mixture, it would be easy to learn:

\begin{lemma}[Folklore]
Let $G$ be an $n$-dimensional Gaussian, and let $\delta>0$.
There exists a polynomial time algorithm that, given $O(n^2/\delta^2)$ independent samples from $G$,
returns a probability distribution $\p$ so that, with probability at least $2/3$, $\dtv(G,\p)<\delta$.
\end{lemma}

Note that we can easily improve the success probability
in Lemma \ref{SingleGausLearnLem} to $1-\delta$
at the cost of multiplying the sample complexity by $\log(1/\delta)$.
In particular, we have:
\begin{corollary}\label{SingleGausLearnLem}
Let $G$ be an $n$-dimensional Gaussian, and let $\eps>0$.
There exists a polynomial time algorithm that given $M$
independent samples from $G$ returns a probability distribution $\p$
so that with probability at least $1-\exp(-\Omega(M\eps^2/d^2))$, we have $\dtv(G,\p)<\eps$.
\end{corollary}

Unfortunately, we cannot simply run this algorithm for each component in our mixture,
since we do not know which samples come from which component.
However, if we manage to correctly guess where each sample comes from this will not be an issue.

\begin{proposition}\label{GMMGuessProp}
Given  $\eps>0$ and positive integers $k$ and $n$, there exists an algorithm that
given a probability distribution $\p$, which is a $k$-mixture of $n$-dimensional Gaussians,
and $\Theta(n^2k^3\log(k)/\eps^3)$ independent samples from $\p$,
returns a set of $\exp(O(n^2k^3\log^2(k)/\eps^3))$ distributions $\q_i$
so that with probability at least $2/3$ there exists an $i$
so that $\dtv(\p_i,\q_i)<\eps$.
\end{proposition}
\begin{proof}
If our algorithm is given $N$ samples, it will return a $q_i$ for each function $f:[N]\rightarrow [k]$.
Intuitively, $f$ encodes our guess as to which sample came from which component of the mixture.
Note that there are only $\exp(N\log(k))$ many such $f$'s.

The algorithm is quite simple.
Let $s_1,s_2,\ldots,s_N$ be our samples, and let $S_i = \{s_j : f(j)=i\}$.
Letting $A$ be the algorithm from Corollary \ref{SingleGausLearnLem} with $\delta$ taken to be $\eps/(10k)$, we let
$$
\q_f = \sum_{i=1}^k \left(\frac{|S_i|}{N}\right) A(S_i).
$$
We claim that at least one of these works with probability $2/3$.

In particular, let $\p$ be the mixture $\sum_{i=1}^k w_i G_i$.
Consider the case where $f$ correctly guesses which part of the mixture
each sample was taken from. In particular, $f(i)=j$ if and only if $s_i$ was taken from $G_j$.
We claim that, with probability at least $2/3$, \emph{this} choice of $f$ leads to
$\dtv(\q_f,\p)<\eps$. We will henceforth use $S_i$ to denote the set of samples
actually taken from the $i^{th}$ component of the mixture.

Firstly, note that by standard concentration bounds,
we have that $\left|\frac{|S_i|}{N} - w_i \right| < \eps/(10 k)$, for all $i$,
with probability at least $9/10$.

Secondly, note that after conditioning on which samples of $\p$ were taken form which part,
the samples themselves are independent samples from the appropriate $G_i$'s,
that with probability at least $9/10$ we have that $\dtv(G_i,A(S_i))<\eps/(10k)$
for all $i$ with $|S_i| \gg n^2k^2\log(k)/\eps^2$.

We claim that if both of the above conditions hold (which happens with probability at least $2/3$)
that $\dtv(\p,\q_f)<\eps$. Letting $S$ be the set of indices $i$ so that $w_i < \eps/(5k)$,
we note that for $i\not\in S$ that $|S_i| \gg n^2k^2\log(k)/\eps^2$. We then have that
\begin{align*}
\dtv(\p,\q_f) & = \frac{1}{2}\left|\sum_{i=1}^k w_iG_i - \sum_{i=1}^k \left(\frac{|S_i|}{N}\right) A(S_i)\right|_1 \\
& \leq \frac{1}{2}\sum_{i=1}^k\left|  w_iG_i -\left(\frac{|S_i|}{N}\right) A(S_i) \right|_2\\
& \leq \frac{1}{2}\sum_{i=1}^k \left|w_i - \left(\frac{|S_i|}{N}\right) \right| + \frac{1}{2}\sum_{i=1}^k \max\left(w_i,\left(\frac{|S_i|}{N}\right)\right) |G_i-A(S_i)|_1\\
& \leq \eps/20 + \frac{1}{2}\sum_{i\in S}(w_i+\eps/(10k))|G_i-A(S_i)|_1 + \sum_{i\not\in S}\dtv(G_i,A(S_i))\\
& \leq \eps/20 + \frac{1}{2}\sum_{i\in S}3\eps/(10k)+ \sum_{i\not\in S}\eps/(10k)\\
& \leq \eps/20 + 3\eps/20 + \eps/20\\
& < \eps \;.
\end{align*}
This completes the proof.
\end{proof}

Theorem \ref{GMMLearnThrm} now follows immediately form a standard tournament argument (see, e.g.,~\cite{DL:01, DDS12stoc, DDS15}).

\section{Sample Complexity Upper Bound for Parameter Estimation of Separated GMMs} \label{sec:param-gmm}

Next we consider the more complicated task of parameter estimation. 
In particular, given samples from a distribution $\p = \sum_{i=1}^k G_i$, where each $G_i$ is a weighted Gaussian, 
we would like to learn a distribution $\q$ that is not only close to $\p$ 
but that can be written as $\q = \sum_{i=1}^k H_i$ with $\|H_i - G_i\|_1$ small for all $i$. 
Now, in general, this task will require number of samples exponential in $k$, 
simply because there are pairs of mixtures that are $\eps^{\Omega(k)}$-close in variation distance 
and yet $\eps$-far in terms of their individual components. 
However, we will show that if the components are separated, 
this cannot be the case and thus learning the distribution in variation distance will be sufficient.

Before we begin, we need to clarify our notion of separation. 
Given two pseudo-distributions, $p$ and $q$, we define their overlap as 
$V(p,q):= \int \min(dp,dq)$. We should note that if $p$ and $q$ are honest distributions, then 
$\dtv(p,q) = 1-V(p,q)$. We have the following theorem:

\begin{theorem}\label{parameterThrm}
There exists a constant $C$ so that if we have two mixtures $p=\sum_{i=1}^k G_i$ and $q = \sum_{i=1}^k H_i$, 
where $p$ and $q$ are normalized distributions with $H_i, G_i$ weighted Gaussians, 
such that $\dtv(p,q) < (\delta/k)^C$ for some sufficiently small $\delta>0$, 
and so that for any $i\neq j$, $V(G_i,G_j), V(H_i,H_j) < (\delta/k)^C$, 
then there exists a permutation $\pi:[k]\rightarrow [k]$ so that $\|G_i - H_{\pi(i)}\|_1 < \delta$ for all $i$.
\end{theorem}

We begin by producing a proxy for the overlap between distributions. 
In particular, for pseudo-distributions $p$ and $q$, we define
$$
h(p,q) = -\log\left( \int \sqrt{dpdq}\right).
$$
Notice that if $p$ and $q$ are true distributions, 
this is related to the Hellinger distance by $H(p,q)=2(1-e^{-h(p,q)})$. 
We also note the relationship to the overlap:
\begin{lemma}
If $p$ and $q$ are pseudo-distributions with $L_1$ norm at most $1$, then
$$
V(p,q) = \exp(-\Theta(h(p,q))+O(1)) \;.
$$
\end{lemma}
\begin{proof}
On the one hand, there is an easy upper bound
$$
V(p,q)=\int \min(dp,dq) \leq \int \sqrt{dpdq} = \exp(-h(p,q)).
$$
The lower bound is by Cauchy-Schwarz
$$
\exp(-h(p,q)) = \int \sqrt{dpdq} \leq \left(\int \min(dp,dq) \right)^{1/2}\left(\int \max(dp,dq) \right)^{1/2} \leq \sqrt{2V(p,q)}.
$$
This completes our proof.
\end{proof}
Ideally we would like to show that $h$ is nearly a metric for Gaussians. 
Namely that $h(A,C) = O(h(A,B)+h(B,C))$. This would imply that $H_i$ could not have 
large overlap with more than one $G_j$, since if $V(H_i,G_a)$ and $V(H_i,G_b)$ were both large, 
then $h(H_i,G_a),h(H_i,G_b)$ would be small and therefore, $h(G_a,G_b)$ would be small. 
This would contradict our assumption that $G_a$ and $G_b$ have small overlap. 
Unfortunately, this is not true. In one dimension, a very wide Gaussian may have non-trivial overlap 
with two narrow Gaussians with widely separated means, 
neither of which overlaps the other substantially. 
We will need to develop techniques to deal with this circumstance.

To do this, we introduce an intermediate notation. 
If $G_i = w_iN(\mu_i,\Sigma_i)$ are weighted Gaussians, we define
$$
h_\Sigma(G_1,G_2) := h(N(0,\Sigma_1),N(0,\Sigma_2)).
$$
This is useful because it does satisfy an approximate triangle inequality.
\begin{proposition}\label{triangleProp}
For $F,G,H$ weighted Gaussians, we have that
$$
h_\Sigma(F,H) = O(h_\Sigma(F,G)+h_\Sigma(G,H)).
$$
\end{proposition}

Before we prove this, we will first need to find an approximation to $h_\Sigma$.
\begin{lemma}\label{hSigmaLem}
If $G$ and $H$ are weighted Gaussians with covariance matrices $A$ and $B$ respectively, then
$$
h_\Sigma(G,H) = \Theta\left(\Sigma_{\lambda\textrm{ eigenvalue of }B^{-1/2}AB^{-1/2}} \min(|\log(\lambda)|,|\log(\lambda)|^2) \right).
$$
\end{lemma}
\begin{proof}
By making an appropriate change of variables, we can assume that $H$ has identity covariance 
and $G$ has covariance $B^{-1/2}AB^{-1/2}$. 
Thus, it suffices to consider the case where $B=I$. 
In this case, we may diagonalize $A$ to get $A=\mathrm{diag}(\lambda_i)$. 
We then have that
\begin{align*}
h_\Sigma(G,H) & = h(N(0,A),N(0,I))\\
& = -\log\left((2\pi)^{-n/2} \prod_{i=1}^n \lambda_i^{-1/4} \int \exp\left( - \sum_{i=1}^n x_i^2/2(1/(2\lambda_i)+1/2)\right) dx\right)\\
& = -\log\left(\prod_{i=1}^n \lambda_i^{-1/4}((1+\lambda_i^{-1})/2)^{-1/2}\right)\\
& = \sum_{i=1}^n \log(\lambda_i)/4+\log((1+\lambda_i^{-1})/2)/2.
\end{align*}
We claim that
$$
\log(\lambda)/4+\log((1+\lambda^{-1})/2)/2 = \Theta(\min(|\log(\lambda)|,|\log(\lambda)|^2)).
$$
To see this note that when $\lambda=1+\epsilon$ for small values of $\eps$, the left hand side above is
$$
(\eps/4-\eps^2/8+O(\eps^3)) +(-\eps/4+3\eps/16+O(\eps^3)) = \eps^2/16 +O(\eps^3) = \Theta(\eps^2).
$$
On the other hand, when $\lambda \gg 1$, this is asymptotic to $|\log(\lambda)|/4$, 
and when $\lambda \ll 1$, it is similarly asymptotic to $-\log(\lambda)/4$. 
Finally, since it is easily verified that $\log(\lambda)/4+\log((1+\lambda^{-1})/2)/2$ is never $0$ unless $\lambda=1$, 
this proves the claim, from which our lemma follows easily.
\end{proof}

We will also need the following fact about eigenvalues of a product of matrices:
\begin{lemma}\label{evProdLem}
Let $A$ and $B$ be symmetric matrices with eigenvalues $\nu_1\geq \nu_2 \geq \ldots \geq \nu_n >0$ 
and $\mu_1\geq \mu_2 \geq \ldots \geq \mu_n >0$, respectively. 
Let $\lambda_1\geq \lambda_2 \geq \ldots \geq \lambda_{2n}>0$ be the sorting of the $\nu_i$ and $\mu_i$ together. 
Let $M$ be a matrix with $M^TM=A$. Then, the $k^{th}$ largest eigenvalue of $M^TBM$ is at most $\lambda_k^2$.
\end{lemma}
\begin{proof}
We need to show that there is an $(n-k+1)$-dimensional subspace $V$ so that for $v\in V$ 
we have that $vA^{1/2}BA^{1/2}v \leq \lambda_k^2 |v|^2.$ 
Suppose that $\lambda_1,\ldots,\lambda_{k-1}$ contains $m$ of the $\nu_i$ 
and $k-m-1$ of the $\mu_i$. Let $V$ be the subspace of vectors $v$ so that $v$ is perpendicular 
to the top $m$ eigenvectors of $A$ and so that $A^{1/2}v$ is perpendicular to the top $m-k-1$ eigenvalues of $B$. Then
$$
vM^TBMv \leq \lambda_k |Mv|^2 = \lambda_k vAv \leq \lambda_k^2 |v|^2.
$$
This completes the proof.
\end{proof}

We are now ready to prove Proposition \ref{triangleProp}.
\begin{proof}
Let $F,G,H$ have covariance matrices $A,B,C$ respectively. 
Let $\Sigma_1 = A^{-1/2}BA^{-1/2}$, $\Sigma_2 = B^{-1/2}CB^{-1/2}$ and $\Sigma_3 = A^{-1/2}CA^{-1/2} = (A^{-1/2}B^{1/2})\Sigma_2(B^{1/2}A^{-1/2})$. 
Let the eigenvalues of $\Sigma_i$ be $\lambda^{(i)}_1 \geq \lambda^{(i)}_2 \geq \ldots \geq \lambda^{(i)}_n >0$. 
Let $f(x)=\max(0,\min(\log(x),\log^2(x))).$ We have by Lemma \ref{hSigmaLem} that
$$
h_\Sigma(F,G) =\Theta\left( \sum_{i=1}^n f(\lambda^{(1)}_i) + f(1/\lambda^{(1)}_i)\right),$$ $$ h_\Sigma(G,H) =\Theta\left( \sum_{i=1}^n f(\lambda^{(2)}_i) + f(1/\lambda^{(2)}_i)\right),$$ $$h_\Sigma(F,H) = \Theta\left(\sum_{i=1}^n f(\lambda^{(3)}_i) + f(1/\lambda^{(3)}_i)\right).
$$
On the other hand, Lemma \ref{evProdLem} says that $\lambda^{(3)}_i$ is at most the square of the $i^{th}$ largest of the $\lambda^{(1)}_j$ and $\lambda^{(2)}_j$. 
Therefore,
$$
\sum_{i=1}^n f(\lambda^{(3)}_i) = O\left(\sum_{i=1}^n f(\lambda^{(1)}_i)+\sum_{i=1}^n f(\lambda^{(2)}_i)  \right).
$$
Similarly, by considering the inverses of these matrices, we find that
$$
\sum_{i=1}^n f(1/\lambda^{(3)}_i) = O\left(\sum_{i=1}^n f(1/\lambda^{(1)}_i)+\sum_{i=1}^n f(1/\lambda^{(2)}_i)  \right).
$$
Together these complete the proof.
\end{proof}

In addition to this, we need to know what else contributes to $h(G,H)$. We define
$$
h_\mu(G,H) = h(G,H)-h_\Sigma (G,H).
$$
We make the following claim:
\begin{proposition}
$$
h_\mu(w_1N(\mu_1,\Sigma_1),w_2N(\mu_2,\Sigma_2)) = -1/2\log(w_1w_2)+\inf_x ((x-\mu_1)\Sigma_1^{-1}(x-\mu_1)+(x-\mu_2)\Sigma_2^{-1}(x-\mu_2))/4.
$$
\end{proposition}
\begin{proof}
We have that
\begin{align*}
&h(w_1N(\mu_1,\Sigma_1),w_2N(\mu_2,\Sigma_2)) \\= & -1/2\log(w_1w_2)-\log\left(\int (2\pi)^{-n/2}(\det(\Sigma_1\Sigma_2))^{-1/4}\exp(-((x-\mu_1)\Sigma_1^{-1}(x-\mu_1)+(x-\mu_2)\Sigma_2^{-1}(x-\mu_2))/4)dx \right).
\end{align*}
Letting $x_0$ achieve the minimum value of $((x-\mu_1)\Sigma_1^{-1}(x-\mu_1)+(x-\mu_2)\Sigma_2^{-1}(x-\mu_2))/4$, this is
\begin{align*}
-1/2\log(w_1w_2)&+((x_0-\mu_1)\Sigma_1^{-1}(x_0-\mu_1)+(x_0-\mu_2)\Sigma_2^{-1}(x_0-\mu_2))/4\\ &+\log\left(\int (2\pi)^{-n/2}(\det(\Sigma_1\Sigma_2))^{-1/4}\exp(-(x-x_0)(\Sigma_1^{-1}+\Sigma_2^{-1})(x-x_0)/4)dx\right).
\end{align*}
Noting that the term at the end is simply $h_\Sigma(w_1N(\mu_1,\Sigma_1),w_2N(\mu_2,\Sigma_2))$ completes the proof.
\end{proof}

We need one further proposition from which Theorem \ref{parameterThrm} will follow easily.
\begin{proposition}\label{overlapProp}
Under the assumptions of Theorem \ref{parameterThrm}, for each $i$ there exists at most one $j$ so that $h(G_i,H_j)>(\delta/k)^{\sqrt{C}}$.
\end{proposition}

To prove this, we will need one further lemma:
\begin{lemma}\label{approxSizeLem}
If $h(G_i,H_j)>(\delta/k)^{\sqrt{C}}$, with $\Sigma_G$ and $\Sigma_H$ the covariance matrices of the corresponding Gaussians, 
then for $A$ a sufficiently large constant (independent of $C$) $\Sigma_G\leq A\Sigma_H$.
\end{lemma}
\begin{proof}
Suppose for sake of contradiction that this is not the case. 
By making a change of variables, we can assume that $\Sigma_G=I$. 
This means that $\Sigma_H$ has some eigenvector $v$ with eigenvalue less than $1/A$. 
Let $H'$ be $H_j$ translated by $C^{1/4}\sqrt{\log(k/\delta)}$ in the direction closer to the mean of $G_i$. 
We have that $$h(H_j,H')=h_\mu(H_j,H') = \Theta(A\sqrt{C}\log(k/\delta)).$$ 
Therefore, $V(H_j,H') = (\delta/k)^{\Omega(A\sqrt{C})}.$ On the other hand,
\begin{align*}
h(G_i,H') & = h_\Sigma(G_i,H') + h_\mu(G_i,H')\\
& \leq h_\Sigma(G_i,H) + O(\sqrt{C}\log(k/\delta))\\
& \leq h(G_i,H) + O(\sqrt{C}\log(k/\delta))\\
& = O(\sqrt{C}\log(k/\delta)).
\end{align*}
This means that $V(G,H')=(\delta/k)^{O(\sqrt{C})}.$

This means that $p=\sum_\ell G_\ell$ has $V(p,H')=(\delta/k)^{O(\sqrt{C})}$, 
and since $p$ is close to $q=\sum_\ell H_\ell$, there must be some $\ell$ 
so that $V(H',H_\ell) > (\delta/k)^{O(\sqrt{C})}$. Note that $\ell$ here cannot be $j$. 
On the other hand, this implies that
\begin{align*}
h(H_j,H_\ell) & = h_\Sigma(H_j,H_\ell) + h_\mu(H_j,H_\ell)\\
& \leq h_\Sigma(H',H_\ell)+O(h_\mu(H',H_\ell)+A\sqrt{C}\log(k/\delta))\\
& \leq O(h(H',H_\ell)+A\sqrt{C}\log(k/\delta))\\
& = O(A\sqrt{C}\log(k/\delta)).
\end{align*}
Therefore, $V(H_\ell,H_j) = (\delta/k)^{O(A\sqrt{C})}$, which for $C\gg A^2$ contradicts our assumptions. 
This completes the proof.
\end{proof}

We are now prepared to prove Proposition \ref{overlapProp}.
\begin{proof}
Suppose for sake of contradiction that $V(G_i,H_j),V(G_i,H_\ell)>(\delta/k)^{\sqrt{C}}$ for some $j\neq \ell$. 
Then, by Lemma \ref{approxSizeLem}, we have that all of the covariance matrices of $G_i,H_j,H_\ell$ 
are comparable to each other (namely each is no more than a constant multiple of any other). 
We claim that this implies that $h_\mu(H_j,H_\ell)=O(h_\mu(H_j,G_i)+h_\mu(H_\ell,G_i)+\sqrt{C}\log(k/\delta)).$ 
This is because, letting $\Sigma_G$ be the covariance matrix of $G_i$, and letting $w_G = |G_i|_1, w_H=|H_j|_1, w_H' = |H_\ell|_1$, we have the following: 
First, each of $w_G,w_H,w_H' = \exp(O(\sqrt{C}\log(\delta/k)))$, 
because each distribution has large overlap with some other. 
Next, we have that
\begin{align*}
h_\mu(G_i,H_j) & = O(\sqrt{C}\log(\delta/k)) + \inf_x \Theta((x-\mu_{G_i})\Sigma_G(x-\mu_{G_i})+(x-\mu_{H_j})\Sigma_G(x-\mu_{H_h})\\
& = O(\sqrt{C}\log(\delta/k)) + \Theta((\mu_{G_i} - \mu_{H_j})\Sigma_G(\mu_{G_i} - \mu_{H_j})).
\end{align*}
Similarly,
$$
h_\mu(G_i,H_\ell)= O(\sqrt{C}\log(\delta/k)) + \Theta((\mu_{G_i} - \mu_{H_\ell})\Sigma_G(\mu_{G_i} - \mu_{H_\ell})),
$$
and
$$
h_\mu(H_j,H_\ell)= O(\sqrt{C}\log(\delta/k)) + \Theta((\mu_{H_j} - \mu_{H_\ell})\Sigma_G(\mu_{H_j} - \mu_{H_\ell})).
$$
This implies that
$$
h_\mu(H_j,H_\ell) = O(h_\mu(H_j,G)+h_\mu(G,H_\ell)).
$$
Therefore, we have that
\begin{align*}
h(H_j,H_\ell) & = h_\Sigma(H_j,H_\ell)+h_\mu(H_j,H_\ell)\\
& = O(h_\Sigma(H_j,G)+h_\mu(H_j,G)+h_\Sigma(G,H_\ell)+h_\mu(G,H_\ell))\\
& = O(h(H_j,G)+h(G,H_\ell))\\
& = O(\log(1/V(H_j,G))+\log(1/V(G,H_\ell)))\\
& = O(\sqrt{C}\log(k/\delta)).
\end{align*}
However, this implies that $V(H_j,H_\ell) = (\delta/k)^{O(\sqrt{C})},$ a contradiction.

This completes our proof.
\end{proof}

We are now ready to prove Theorem \ref{parameterThrm}
\begin{proof}
For each $G_i$ that has overlap more than $(\delta/k)^{\sqrt{C}}$ with some $H_j$, let $\pi(i)$ be that $j$. 
For other $i$, define $\pi(i)$ arbitrarily subject to $\pi$ being a permutation.

Note that $V(G_i,H_j)<(\delta/k)^{\sqrt{C}}$ for any $j\neq \pi(i)$. 
Also note that $$V(G_i,q) \geq V(G_i,p)-|p-q|_1 =|G_i|_1 - 2(\delta/k)^C.$$ On the other hand,
\begin{align*}
V(G_i,q) & \leq \sum_j V(G_i,H_j)\\
& \leq V(G_i,H_{\pi(i)}) + \sum_{j\neq \pi(i)}V(G_i,H_j)\\
& \leq V(G_i,H_{\pi(i)}) + \delta/3.
\end{align*}
Therefore $V(G_i,H_{\pi(i)}) \geq |G_i|_1 - \delta/2$. 
It is also at most $|G_i|_1 - \delta/2$. On the other hand $|G_i-H_{\pi(i)}|_1 = |G_i|_1 + |H_{\pi(i)}|_1 - 2 V(G_i,H_{\pi(i)}) \leq \delta$. 
This completes the proof.
\end{proof}

\section{Testing the Mean of a High-Dimensional Gaussian} \label{sec:test-app}

\begin{theorem} \label{thm:test-upper}
There exists an algorithm that given $\eps>0$ and $k=O(\sqrt{n}/\eps^2)$ samples 
from an $n$-dimensional Gaussian $G=N(\mu, I)$ distinguishes between the cases
\begin{itemize}
\item $\mu=0$
\item $\|\mu\|_2 > \eps$
\end{itemize}
with probability at least $2/3$.
\end{theorem}
\begin{proof}
The tester is fairly simple. Let $X_i$ be the $i^{th}$ sample, and let
$$
Z:= \frac{1}{\sqrt{k}}\sum_{i=1}^{\new{k}} X_i.
$$
The algorithm returns ``YES'' if $\|Z\|_2^2 <\eps^2 k/2+n$ and ``NO'' otherwise.

To show correctness, note that $Z$ is distributed as 
$N(\mu\sqrt{k},I)$. If $\mu=0$, then $\|Z\|_2^2$ has mean $n$ and variance $O(n)$, 
and so it is less than $n+\eps^2 k /2$ with probability at least $2/3$, 
assuming that $k$ is a sufficiently large multiple of $\sqrt{n}/\eps^2$. 
On the other hand, if $\|\mu\|_2 > \eps$, we note that $\|Z\|_2^2$ has mean $n+k\|\mu\|_2^2$ 
and variance $O(n)+ \new{O(k \|\mu\|^2_2)}$. Thus, if $k\|\mu\|_2^2 \gg \sqrt{n}$, 
the algorithm rejects with probability $2/3$. Again, this happens if $\|\mu\|_2 >\eps$ 
and $k$ is a sufficiently large multiple of $\sqrt{n}/\eps^2$. This completes the proof.
\end{proof}

We also note that this tester can be implemented in the SQ model simply 
by verifying that each coordinate-wise median has absolute value less than $\eps/\sqrt{n}$, 
which can be verified by showing that $\Pr(x_i > 0) = 1/2+O(\eps/\sqrt{n})$.

We also show that the tester above is sample-optimal, up to a constant factor:

\begin{theorem} \label{thm:test-lower}
There is no algorithm that given $k=o(\sqrt{n}/\eps^2)$ samples 
from an $n$-dimensional Gaussian $G=N(\mu,I)$ distinguishes between the cases
\begin{itemize}
\item $\mu=0$
\item $\|\mu\|_2 > \eps$
\end{itemize}
with probability at least $2/3$.
\end{theorem}
\begin{proof}
Suppose for sake of contradiction that such an algorithm does exist. 
Consider the following scenario: 
Let $\mu$ be taken from the distribution $N(0,(2\eps/\sqrt{n})I)$. 
Note that $\|\mu\|_2 > \eps$ with probability at least $9/10$. 
Let $Y_1,Y_2,\ldots,Y_k$ be independent samples taken from $N(\mu,I)$. 
And let $Z_1,\ldots,Z_k$ be independent samples from $N(0,I)$. 
Assuming that our algorithm exists, it can distinguish between a sample from 
$Y_1,\ldots,Y_k$ and a sample from $Z_1,\ldots,Z_k$ with probability better than $1/2$. 
This means that these distributions must have constant variational distance. 
However, note that the vector $(Z_1,\ldots,Z_k)$ is simply a standard $nk$-dimensional Gaussian. 
The vector $(Y_1,\ldots,Y_k)$ on the other hand is an $nk$-dimensional Gaussian with mean $0$ 
and with 
$$
\mathrm{Cov}(Y_{ab},Y_{cd}) = 
\begin{cases} 
1+2\eps^2/n \;, & \textrm{if }ab=cd\\ 
2\eps^2/n \;, & \textrm{if } b=d \textrm{ and }a\neq c\\ 
0 \;, & \textrm{otherwise}\;. 
\end{cases}
$$
By standard results, $G'=N(0,\Sigma)$ has constant variation distance from $N(0,I)$ 
if and only if $\|\Sigma - I\|_F \gg 1$. Taking $\Sigma$ to be the covariance matrix for the $Y$'s, we have that
$$
\|\Sigma - I\|_F^2 = nk(2\eps/\sqrt{n})^2 = 4k\eps^2/n = o(1) \;.
$$
This implies that the distribution on $Y$'s is close, in total variation distance, 
to the distribution on $Z$'s, and gives a contradiction.
\end{proof}

\section{Omitted Proofs} \label{app:om}

\subsection{Proof of Fact~\ref{fact:eigenfunction}}
We will need the following claim:

\begin{claim} 
We have that:
$$He_i(x \cos \theta + y \sin \theta) = \sum_{j=1}^i {i \choose j} \cos^j \theta \sin^{i-j} \theta He_j(x) He_{i-j}(y) \;.$$
\end{claim}
\begin{proof}
The $He_i(x)$ are monic polynomials: 
the lead term is $x^i$ with coefficient $1$. 
Thus, all the degree-$i$ terms of $He_i(x \cos \theta + y \sin \theta)$ are given by
$$(x \cos \theta + y \sin \theta)^i  \sum_{j=1}^i {i \choose j} \cos^j \theta \sin^{i-j} \theta x^j y^{i-j} \;.$$
It follows that the degree-$i$ terms of the LHS and RHS of the lemma agree. 
Therefore, we have 
$$He_i(x \cos \theta + y \sin \theta) =  p(x,y) + \sum_{j=1}^i {i \choose j} \cos^j \theta \sin^{i-j} \theta He_j(x) He_{i-j}(y) \;,$$
for some polynomial $p(x,y)$ of degree at most $i-1$. 
We need to show that $p(x,y)$ is identically zero. 
To show this we consider $\E[He_i(X \cos \theta + Y \sin \theta)^2]$, 
for $(X,Y) \sim N(0,I)$. Since the Gaussian is unaltered by rotations, by a change of coordinates we have that:
\begin{align*} 
\E[He_i(X \cos \theta + Y \sin \theta)^2] 
& = \int_{-\infty}^\infty \int_{-\infty}^\infty He_i(x \cos \theta + y \sin \theta)^2 G(x) G(y) dx dy \\
& = \int_{-\infty}^\infty \int_{-\infty}^\infty He_i(x')^2 G(x') G(y') dx' dy' = i! \;.
\end{align*}
However, pairs of distinct $He_j(x) He_{i-j}(y)$ 
are orthogonal to each other and they are all orthogonal to the lower degree polynomial $p(x,y)$. 
Thus, we have
\begin{align*} 
\E[He_i(X \cos \theta + Y \sin \theta)^2] & = \E[p(X,Y)^2] +  \sum_{j=1}^i {i \choose j}^2 \cos^{2j} \theta \sin^{2(i-j)} \theta \E[He_j(X)^2 He_{i-j}(Y)^2] \\
& = \E[p(X,Y)^2] +  \sum_{j=1}^i {i \choose j}^2 \cos^{2j} \theta \sin^{2(i-j)} \theta  i! (i-j)! \\
& = \E[p(X,Y)^2] + i! \sum_{j=1}^i {i \choose j} \cos^{2j} \theta \sin^{2(i-j)} \theta\\
& = \E[p(X,Y)^2] + i! (\cos^2 \theta + \sin^2 \theta)^i =  \E[p(X,Y)^2] + i! \;.
\end{align*}
We must therefore have that $\E[p(X,Y)^2]=0$. 
Since the Gaussian has positive pdf everywhere, this implies that $p(X,Y)$ is identically zero.
\end{proof}

We now have:
\begin{align*}
U_\theta( He_i G) (x) & = \int_{-\infty}^\infty He_i(x \cos \theta + y \sin \theta) G(x \cos \theta + y \sin \theta) G(x \sin \theta - y \cos \theta) dy \\
& = \int_{-\infty}^\infty He_i(x \cos \theta + y \sin \theta) G(x) G(y) dy \\
& = \sum_{j=1}^i {i \choose j} \cos^j \theta \sin^{i-j} \theta \int_{-\infty}^\infty He_j(x) He_{i-j}(y) G(x \sin \theta - y \cos \theta) dy \\
&= \cos^i \theta He_i(x) G(x) \;,
\end{align*}
since $\int_{-\infty}^\infty He_{i-j}(y) G(y) dy = \delta_{ij}$.
This completes the proof. \qed

\subsection{Proof of Lemma~\ref{lem:set-of-nearly-orthogonal}}
We use the following lemma:
\begin{lemma}[Proposition 1 from \cite{CaiFanJiang}] \label{lem:sphere-cite} 
Given any $0 < \eps < \pi/2$,  let $\theta$ be the angle 
between two random unit vectors uniformly distributed over $\s_n$. 
Then we have that:
$$\pr[|\theta-\pi/2| \geq \eps] \leq O(\sqrt{n} (\cos \eps)^{n-2}) \;.$$
\end{lemma}
\noindent As a corollary, we have:
\begin{corollary} \label{cor:param-nearly-orthogonal} 
Let $\theta$ be the angle between two random unit vectors uniformly distributed over $\s_n$. 
Then we have that:
$$\pr \left[ |\cos \theta| \geq \Omega(n^{-\alpha}) \right] \leq \exp\left(-\Omega(n^{1-2\alpha})\right) \;,$$
for any $0 \leq \alpha \leq 1/2$.
\end{corollary}

\begin{proof} 
We apply Lemma \ref{lem:sphere-cite} with $\eps = n^{-\alpha}$. 
If $n^{-\alpha}=O(1)$, the result is trivial, so we may assume that $\eps \leq 1/100$. 
Then we have that $\cos \eps \leq 1 - \eps^2/2 + \eps^2/24 \leq 1 - \eps^2/3  \leq \exp (\eps^2/4).$
Lemma \ref{lem:sphere-cite} now gives that 
$$\pr\left[ |\theta-\pi/2| \geq n^{-\alpha} \right] \leq O\left( \sqrt{n} \exp(-n^{-2\alpha}/4)^{n-2} \right) \leq \exp(-n^{1-2\alpha}/5) \;.$$
Note that if $|\theta-\pi/2| \new{\leq} n^{-\alpha}$, it follows that 
$|\cos \theta| \leq |\theta-\pi/2| \leq n^{-\alpha}$.
\end{proof}

\new{Using Corollary~\ref{cor:param-nearly-orthogonal} for $\alpha = 1/2-c$,} 
and a union bound \new{over all pairs of distinct vectors in $S$}, 
the probability that there exist $v \neq v'  \in S$ such that 
$|v \cdot v'| \new{\geq} \new{\Omega}(n^{c-1/2})$ is less than
$$|S|^2 2^{-\Omega(n^{2c})} < 1 \;.$$ 
Therefore, the set $S$ will satisfy the statement of Lemma~\ref{lem:set-of-nearly-orthogonal} 
with positive probability, as desired. \qed

\subsection{Proof of Fact~\ref{clm:chi-squared-mixtures}}
We have:
\begin{align*}
1 + \chi^2(wB+(1-w)C, D) 
&= \int \left(wB(x)+(1-w)C(x)\right)^2/D(x) dx \\
&= w^2 \int B(x)^2/D(x) dx + (1-w)^2 \int C(x)^2/D(x) dx + 2w(1-w) \int B(x)C(x)/D(x) dx \\
&= w^2(1+\chi^2(B,D)) + (1-w)^2 (1+\chi^2(C,D)) + 2w(1-w)(1+ \chi_D(B,C)) \\
&= 1 + w^2\chi^2(B,D) + (1-w)^2 \chi^2(C,D) + 2w(1-w)\chi_D(B,C) \;.
\end{align*}
This completes the proof. \qed

\subsection{Proof of Fact~\ref{clm:correlation-different-mean}}
By definition, we can write
\begin{align*}
1+ \chi_{N(0,1)}(N(\mu',1),N(\mu,1)) & = \int_{-\infty}^\infty G(x-\mu')G(x-\mu)/G(x) dx \\
& = (1/\sqrt{2 \pi}) \cdot  \int_{-\infty}^\infty \exp\left(-(x-\mu')^2/2 -(x-\mu)^2/2 + x^2/2\right) dx \\
& = (1/\sqrt{2 \pi}) \cdot  \int_{-\infty}^\infty \exp\left(-x^2/2 +(\mu'+\mu)x - \mu^2/2-\mu'^2/2 \right) dx \\
&= \int_{-\infty}^\infty G(x - \mu-\mu') \exp\left(-\mu^2/2+-\mu'^2/2 + (\mu+\mu')^2/2\right) dx\\
& = \exp(\mu'\mu) \;.
\end{align*}
This completes the proof. \qed

\subsection{Proof of Fact~\ref{clm:correlation-different-variance}}
By definition, we have that
\begin{align*}
1 + \chi^2(N(0,\sigma^2),N(0,1)) & = (1/\sigma) \int_{-\infty}^\infty G(x/\sigma)^2/G(x) dx \\
& = \frac{1}{\sigma^2\sqrt{2\pi}} \int_{-\infty}^\infty \exp(x^2/2 - x^2/\sigma^2) dx \\
& = \frac{\sqrt{2/\sigma^2 - 1}}{\sigma^2} \cdot \int_{-\infty}^\infty \frac{1}{\sqrt{2/\sigma^2 - 1}} G(x/\sqrt{2/\sigma^2 - 1}) dx \\
& = \frac{\sqrt{2/\sigma^2 - 1}}{\sigma^2} = \sqrt{2/\sigma^4 - 1/\sigma^2} \;.
\end{align*}
This completes the proof. \qed

\end{document}